\pdfoutput=1
\documentclass{article}

\usepackage[nonatbib,preprint]{neurips_2020}

\usepackage[utf8]{inputenc} 
\usepackage[T1]{fontenc}    
\usepackage{url}            
\usepackage{booktabs}       
\usepackage{nicefrac}       
\usepackage{amsmath, latexsym, amssymb, bm, ascmac, here, mathtools, comment}
\usepackage{amsthm}
\usepackage{algorithm, algorithmic}
\usepackage{mathrsfs}
\usepackage{sidecap}

\newtheorem{theorem}{Theorem}[section]

\newtheorem{lemma}{Lemma}[section]
\newtheorem{proposition}{Proposition}[section]

\newcommand{\argmax}{\mathop{\rm arg~max}\limits}
\newcommand{\argmin}{\mathop{\rm arg~min}\limits}
\newcommand{\1}{\mbox{1}\hspace{-0.25em}\mbox{l}}

\usepackage{mcr}

\usepackage{graphicx}
\usepackage{color}

\usepackage{hyperref}       

\title{Bayesian Quadrature Optimization for \\ Probability Threshold Robustness Measure}

\author{%
  Shogo Iwazaki \\
  Department of Computer Science, Nagoya Institute of Technology \\
  Gokiso-cho, Showa-ku, Nagoya, 466-8555, Japan \\
  \And
  Yu Inatsu \\
  RIKEN Center for Advanced Intelligence Project \\
  1-4-1 Nihonbashi, Chuo-ku, Tokyo, 103-0027, Japan \\
  \AND
  Ichiro Takeuchi \\
  Department of Computer Science/Research Institute for Information Science,\\
  Nagoya Institute of Technology \\
  Gokiso-cho, Showa-ku, Nagoya, 466-8555, Japan \\
  \texttt{takeuchi.ichiro@nitech.ac.jp} \\
}

\begin{document}

\maketitle
\begin{abstract}
In many product development problems, the performance of the product is governed by two types of parameters called design parameter and environmental parameter.
While the former is fully controllable, the latter varies depending on the environment in which the product is used.
The challenge of such a problem is to find the design parameter that maximizes the probability that the performance of the product will meet the desired requisite level given the variation of the environmental parameter.
In this paper, we formulate this practical problem as active learning (AL) problems and propose efficient algorithms with theoretically guaranteed performance.
Our basic idea is to use Gaussian Process (GP) model as the surrogate model of the product development process, and then to formulate our AL problems as Bayesian Quadrature Optimization problems for probabilistic threshold robustness (PTR) measure. 
We derive credible intervals for the PTR measure and propose AL algorithms for the optimization and level set estimation of the PTR measure. 
We clarify the theoretical properties of the proposed algorithms and demonstrate their efficiency in both synthetic and real-world product development problems.
\end{abstract}

\section{Introduction}
In many product development problems, the performance of the product is governed by two types of parameters called \emph{design parameter} and \emph{environmental parameter}.
While design parameter is fully controllable, environmental parameter varies depending on the environment in which the product is used.
The challenge of such a problem is to identify the design parameter that maximizes the probability that the performance of the product will meet a desired requisite level given the variation of the environmental parameter.
In this problem setup, it is important to clarify the difference between the \emph{development phase} and the \emph{use phase} of the product.
During the \emph{development phase}, we can arbitrarily specify both the design and environmental parameters.
On the other hand, during the use phase, the design parameter is held fixed, while the environmental parameter varies.
The goal of this paper is to formulate this practical problem as active learning (AL) problems and to propose efficient AL algorithms with theoretically guaranteed performances.

Let us represent the performance of a product as a real-valued function $f(\bm x, \bm w)$ and the desired threshold of the performance as a scalar $h$, where $\bm x \in \cX \subseteq \RR^d$ is design parameter and $\bm w \in \Omega \subseteq \RR^k$ is environmental parameter.
We consider the problem of finding the design parameter $\bm x$ such that the probability that $f(\bm x, \bm w) > h$ is as large as possible or greater than a certain value under the variation of the environmental parameter $\bm w$.
Let
\begin{align}
 \label{eq:pup}
 p_{\rm upper}(\bm x) = \int_{\Omega} \1\left[f(\bm x, \bm w) > h\right] p(\bm w) d \bm w,
\end{align}
where
$\1$ is the indicator function and $p(\bm w)$ is the probability density (mass) function of $\bm w$~~\footnote{
The discrete $\bm w$ case can be similarly defined by replacing the integral with summation.
}.
This measure is referred to as the \emph{probabilistic threshold robustness (PTR) measure} in the context of robust optimization~\cite{beyer2007robust}.
Figure~\ref{fig:illustration} illustrates the problem setup considered in this paper.

\begin{SCfigure}
\centering
\includegraphics[width=0.475\linewidth]{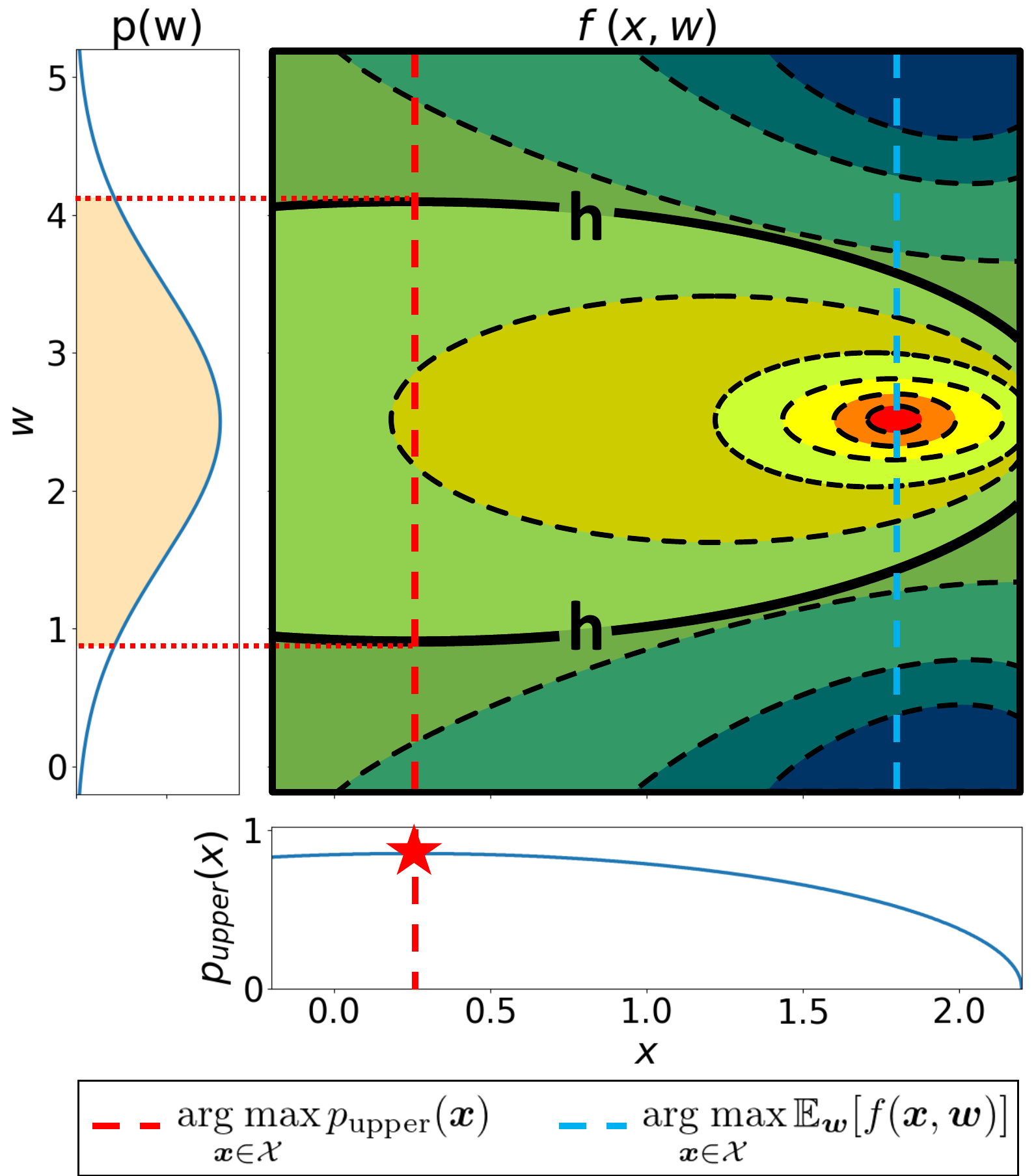}
 \caption{
An illustration of our problem setup in a two-dimensional synthetic example.
The horizontal and the vertical axes represent the design parameter $\bm x$ and the environmental parameter $\bm w$, respectively.
Our goals is to find the design parameter $\bm x^*$ that maximizes $p_{\rm upper}(\bm x)$ \textemdash the probability that the function $f$ exceeds the desired requisite level $h$ under the variation of the environmental parameter $\bm w$ characterized by the probability density $p(\bm w)$.
In this example, the optimal design parameter $\bm x^*$ is indicated by the yellow star and dotted line and the $p_{\rm upper}(\bm x^*)$ is indicated by the filled area of the probability distribution in the left.
In addition, the blue dotted line represents the design parameter that maximizes the expected value of $f(\bm x, \bm w)$.
In general, the design parameter that maximizes the expected value of $f$ and those that maximize $p_{\rm upper}$ are not the same.
}
\label{fig:illustration}
\end{SCfigure}


In order to make the development phase more efficient, it is desirable to be able to find the design parameter $\bm x$ that maximize $p_{\rm upper}(\bm x)$ or to know the range of design parameter $\bm x$ such that $p_{\rm upper}(\bm x)$ is sufficiently enough with as little trial and error as possible.
Therefore, in this paper, we consider AL problems for the optimization and the Level Set Estimation (LSE) of $p_{\rm upper}(\bm x)$.
Our basic idea is to consider the function $f(\bm x, \bm w)$ as a black-box function that is costly to evaluate and to use the Gaussian Process (GP) model as its surrogate model.
We make use of the uncertainties of the black-box function estimated by the surrogate GP model to determine how the design parameter $\bm x$ and the environmental parameter $\bm w$ should be selected at the development stage for the optimization and LSE of $p_{\rm upper}(\bm x)$.

\paragraph{Contributions}
Our contributions in this paper are as follows.
First, we introduce new problem setups that are motivated from practical product development problems that involve optimization and LSE of the PTR measure $p_{\rm upper}(\bm x)$.
Second, we develop AL methods for the optimization and LSE problems which require non-trivial derivation of credible intervals of $p_{\rm upper}(\bm x)$,
Third, we analyze the theoretical properties of $\epsilon$-regret (see \S2) for the optimization setting, and $\epsilon$-accuracy (see \S2) for the LSE setting.
Finally, we demonstrate the efficiency of the proposed methods in both synthetic and real-world problems.

\paragraph{Related works}
AL methods for optimization and LSE problems have been studied in the contexts of Bayesian Optimization (BO)~\cite{shahriari2015taking} and Bayesian LSE~\cite{bryan2006active}, respectively.
%
%
%
In various fields, there are problems in which the effect of uncontrollable and uncertain parameter \textemdash such as the environmental parameter $\bm w$ in \eq{eq:pup} \textemdash must be properly taken into account.
For example, in material simulations, some properties of the target material cannot actually be measured, so the simulation must take into account the uncertainty of these properties.
In medical clinical trials, it is vital to take into account the uncertainty associated with individual differences in patients.
%
%
In modeling functions with uncertainty parameters such as $\bm w$, the most common approach is to consider the expectation \textemdash using our notation, this corresponds to considering the function in the form of $g(\bm x) = \int_\Omega f(\bm x, \bm w) p(\bm w) d \bm w$.
A nice aspect of the function $g(\bm x)$ is that, when $f(\bm x, \bm w)$ is written as a GP model, $g(\bm x)$ is also represented as a GP model.
%
%
AL for maximizing the function in the form of $g(\bm x)$ is called \emph{Bayesian Quadrature Optimization (BQO)}~\cite{toscano2018bayesian}.
%
%
Another line of research, which deals with uncontrollable and uncertain components of GP models is found in the context of \emph{robust learning}.
For example, \cite{bogunovic2018adversarially} studied adversarial robust update of a GP model by considering a scenario where the input is perturbed by an adversary. 
Other closely related works are \cite{nguyen2020distributionally} and \cite{kirschner2020distributionally} where a distributional robust optimization  framework was introduced in the context of BQO. 
Our work is also related to robust BO/LSE methods under input uncertainty~\cite{beland2017bayesian,DBLP:conf/aistats/OliveiraOR19,frohlich2020noisy,DBLP:journals/corr/abs-1909-06064,iwazaki2019bayesian} in which one can only obtain the function values evaluated at noisy inputs.
In addition to these related studies, various forms of robustness of GP modeling have been considered previously~\cite{shah2014student, MartinezCantin18aistats, bogunovic2020corruption}; however, to our knowledge, none of these previous works studied AL problems for the PTR measure in the form of $p_{\rm upper}(\bm x)$, for which it is necessary to solve non-trivial and technically challenging problems.

\section{Preliminaries}\label{sec:Preliminaries}
Let $f: \mathcal{X} \times \Omega \rightarrow \mathbb{R}$ be a black-box
function whose evaluation is costly, where $\mathcal{X}$ is a finite subset \footnote{Extensions to an infinite subset are given in Appendix.} of $\mathbb{R}^d$
and $\Omega$ is a compact subset of  $\mathbb{R}^k$.
At step $t$ in the development phase, we query $f$ at $(\bm{x}_t, \bm{w}_t)$ and observe noisy function
value $y_t = f(\bm{x}_t, \bm{w}_t) + \varepsilon_t$, where
$\varepsilon_t \sim \mathcal{N}(0, \sigma^2)$ is an independent Gaussian noise.
Furthermore, we assume that parameters $\bm{w} \in \Omega$ are distributed by density $p(\bm{w})$ at use phase.
Given a user-specified threshold $h$, we consider the PTR measure defined in \eq{eq:pup}.
In this paper, we assume that $f$ is drawn from GP defined over $\cX \times \Omega$.
Under this setting, we study AL problems for optimization and LSE of the PTR measure.
These problems are non-trivial since $p_{\rm upper}$ cannot be directly evaluated and it is not a GP anymore even if $f$ follows GP.

\paragraph{Optimization Setting}
The first problem we consider is the maximization:
\begin{equation*}
    \bm{x}^\ast = \argmax_{\bm{x} \in \mathcal{X}} ~p_{\text{upper}}(\bm{x}).
\end{equation*}
In this setting, our goal is to find $\bm{x}^\ast$ with few function evaluations as possible.
In order to evaluate an algorithm performance,  we define the following performance metrics based on what we call $\epsilon$-regret\footnote{Note that the name $\epsilon$-regret is used in \cite{bogunovic2018adversarially}, but its definition is different from ours.}.
Given a user-defined accuracy parameter $\epsilon >0$, we define the $\epsilon$-regret $r_t (\epsilon)$ at step $t$ as
\begin{equation*}
    r_t (\epsilon)= (p_{\text{upper}}(\bm{x}^\ast)-\epsilon) - p_{\text{upper}}(\bm{x}_t), \label{eq:epsilon-regret}
\end{equation*}
where $\bm{x}_t$ is the query specified by the algorithm at step $t$.
We then define cumulative $\epsilon$-regret $R_T (\epsilon)$ and Bayes $\epsilon$-regret $BR_T (\epsilon)$ at step $T$ as
\begin{equation*}
 \label{eq:cumulative-epsilon-regret}
  R_T (\epsilon)= \sum_{t=1}^{T} r_t (\epsilon)
  ~~~\text{and}~~~
  BR_T (\epsilon) = \mathbb{E} [R_T (\epsilon)],
\end{equation*}
where the expectation is taken w.r.t. the GP prior, noise $\varepsilon$ and any randomness of the algorithm.

Note that for $\epsilon=0$, $R_T(0)$ and $BR_T(0)$ are cumulative regret and Bayes cumulative regret~\cite{kandasamy2018parallelised}, respectively, which are commonly used in the context of BO.
The reason why we need to consider $r_t(\epsilon)$ instead of $r_t(0)$ is to make a theoretically rigorous argument for the case where $f(\bm x, \bm w)$ is exactly $h$ for some $(\bm x, \bm w) \in {\cal X} \times \Omega$.
In such a case, since only noisy response of $f$ is observed, the uncertainty of $p_{\rm upper}(\bm x)$ cannot be exactly zero no matter how much we evaluate $f(\bm x, \bm w)$.
In \S4, we show that our proposed algorithms in \S3 are sublinear w.r.t. the $\epsilon$-regret (with high probability) and Bayes $\epsilon$-regret for arbitrary small $\epsilon > 0$.

\paragraph{Level Set Estimation (LSE) Setting}
The second problem is the LSE problem \cite{bryan2006active,gotovos2013active}.
An LSE problem is defined as the problem of identifying the input regions where the target function value is above (below) a threshold $\alpha$.
Given a threshold $\alpha \in (0, 1)$, we formulate the LSE of $p_{ {\rm upper}} ({\bm x})$ as the problem of classifying  all ${\bm x} \in \mathcal{X}$ into the \emph{superlevel} set $\mathcal{H}$ and the \emph{sublevel} set $\mathcal{L}$ defined as
\begin{align*}
    \mathcal{H} = \left\{ \bm{x} \in \mathcal{X} \mid p_{\text{upper}}(\bm{x}) \geq \alpha \right\}
 ~~~\text{and}~~~
    \mathcal{L} = \left\{ \bm{x} \in \mathcal{X} \mid p_{\text{upper}}(\bm{x}) < \alpha \right\}.
\end{align*}

In order to evaluate an algorithm performance, we employ $\epsilon$-accuracy which is commonly used in the context of LSE~\cite{gotovos2013active}.
The $\epsilon$-accuracy is defined by using the misclassification loss $e_{\alpha} ({\bm x})$ defined as
\begin{align*}
    e_\alpha ({\bm{x}}) =
    \begin{cases}
        \max \{ 0, p_{\rm upper} ({\bm{x}} ) - \alpha \} & \text{if} \ {\bm{x}} \in \hat{\mathcal{L}}, \\
        \max \{ 0, \alpha - p_{\rm upper} ({\bm{x}} )  \} & \text{if} \ {\bm{x}} \in \hat{\mathcal{H}} \\
    \end{cases} \label{eq:loss_miss}
\end{align*}
where $\hat{\mathcal{H}} $ and $\hat{\mathcal{L}} $ are the estimates of ${\mathcal{H}} $ and ${\mathcal{L}} $ by the algorithm, respectively.
Then, given an accuracy parameter $\epsilon >0$, the pair $(\hat{\mathcal{H}}, \hat{\mathcal{L}})$ is said to be $\epsilon$-accurate solution if every point ${\bm x} \in \mathcal{X} $ satisfies $e_\alpha ({\bm x}) \leq \epsilon$.
In \S4, we show that our proposed algorithm in \S3 returns $\epsilon$-accurate solution with high probability for any $\epsilon >0$.

\subsection{Gaussian Process}
In this paper, we assume $f$ follows GP~\cite{gpml}.
Let $k: (\mathcal{X} \times \Omega) \times (\mathcal{X} \times \Omega) \rightarrow \mathbb{R}$
be a positive definite kernel where $ 0<\sigma^2_{0,min} \leq k((\bm{x}, \bm{w}), (\bm{x}, \bm{w})) \leq 1$ for all $(\bm{x}, \bm{w}) \in \mathcal{X} \times \Omega$,
and we assume $f \sim \mathcal{GP}(0, k)$ where $\mathcal{GP}(\mu, k)$ is the GP with mean function $\mu$ and covariance function $k$.
Given the sequence of queries and responses $\{((\bm{x}_i, \bm{w}_i), y_i)\}_{i=1}^t$,
the posterior distribution of $f(\bm{x}, \bm{w})$ follows a Gaussian with the following mean and variance:
\begin{align*}
    \mu_t(\bm{x}, \bm{w}) &= {\bm k}_t(\bm{x}, \bm{w})^{\top}
    (\bm{K}_t + \sigma^2 \bm{I}_t)^{-1}\bm{y}_t, \\
    \sigma_{t}^{2}(\bm{x}, \bm{w}) &= k((\bm{x}, \bm{w}), (\bm{x}, \bm{w})) -
    \bm{k}_t(\bm{x}, \bm{w})^{\top}(\bm{K}_t + \sigma^2 \bm{I}_t)^{-1}\bm{k}_t(\bm{x}, \bm{w}),
\end{align*}
where $\boldsymbol{k}_{t}(\bm{x}, \bm{w})=(k\left((\bm{x}, \bm{w}), (\bm{x}_1, \bm{w}_1)),\ldots, k\left((\bm{x}, \bm{w}), (\bm{x}_t, \bm{w}_t)\right)\right)^{\top},\ \boldsymbol{y}_{t}=\left(y_{1}, \ldots, y_{t}\right)^{\top}$, and $\bm{K}_t \in \mathbb{R}^{t\times t}$ is the kernel matrix whose $(i, j)$th
element is $k((\bm{x}_i, \bm{w}_i), (\bm{x}_j, \bm{w}_j))$.

\section{Proposed Algorithm}\label{sec:algorithm}
In this section, we propose two AL algorithms for optimization setting and an AL algorithm for LSE setting.
Since $f$ is drawn from GP, $p_{\rm upper}(\bm x)$ is a random variable.
However, it is important to note that $p_{\rm upper}(\bm x)$ does not follow Gaussian distribution anymore, which means that we cannot rely on acquisition functions (AFs) developed in the literature of standard BO and LSE.
Thus, the AFs of our proposed algorithms are constructed using a credible interval of $p_{\rm upper}(\bm{x})$.
At step $t$ in development phase, we are asked to select not only the design parameter $\bm x_t$ but also the environmental parameter $\bm w_t$.
Our basic strategy is to first select $\bm x_t$ based on the credible interval of $p_{\rm upper}(\bm{x})$, and then to select $\bm w_t$ such that the uncertainty of $p_{\rm upper}(\bm x_t)$ is minimized.

\subsection{Credible Interval of PTR Measure}
Here, we derive a credible interval of $p_{\rm upper}(\bm x)$.
\begin{proposition}\label{prop:mean_var_pt1}
 Let the mean and the variance of
 $p_{\rm upper}(\bm x)$
 at step
 $t-1$
 as
 $\mu_{t-1}^{(p)}(\bm x)$
 and
 $\sigma_{t-1}^{(p)2}(\bm x)$.
 Then,
\begin{align*}
    \mu_{t-1}^{(p)}(\bm{x}) &= \int_{\Omega} \Phi
    \left( \frac{\mu_{t-1}(\bm{x}, \bm{w}) - h}{\sigma_{t-1}(\bm{x}, \bm{w})}\right) p(\bm{w}) \text{d}\bm{w}, \\
    \sigma_{t-1}^{(p)2}(\bm x) \le \gamma_{t-1}^2(\bm{x}) &= \int_{\Omega} \Phi
    \left( \frac{\mu_{t-1}(\bm{x}, \bm{w}) - h}{\sigma_{t-1}(\bm{x}, \bm{w})}\right)
    \left\{ 1 - \Phi
    \left( \frac{\mu_{t-1}(\bm{x}, \bm{w}) - h}{\sigma_{t-1}(\bm{x}, \bm{w})}\right) \right\}
     p(\bm{w}) \text{d}\bm{w},
\end{align*}
 where $\Phi$ is the cdf of the standard Gaussian distribution.
\end{proposition}
The proof of the proposition is in Appendix \ref{derivation_mean_upper}.

Based on Proposition~\ref{prop:mean_var_pt1}, the following Lemma implies that the credible interval of $p_{\text{upper}}(\bm{x})$ can be constructed by using $\mu_{t-1}^{(p)}(\bm{x})$ and $\gamma_{t-1}^2(\bm{x})$

\begin{lemma}\label{lem:cred_int}
    Let  $\delta \in (0, 1)$, $m \geq 2$, $t \geq 1$ and
     $\beta_t = \frac{|\mathcal{X}|\pi^2 t^2}{6\delta}$. Then, with probability at least $1-\delta$,
     it holds that
     \begin{equation*}
         |p_{\rm upper}(\bm{x}) - \mu_{t-1}^{(p)}(\bm{x})| <
         \beta_t^{1/m} \gamma_{t-1}^{2/m}(\bm{x}),~\forall \bm{x} \in \mathcal{X},~\forall t \geq 1.
     \end{equation*}
\end{lemma}
The proof of the lemma is in Appendix \ref{lem:fin_union}.
Namely, given $\beta_t > 0, m \geq 2$, credible interval $Q_t(\bm{x})$ can be computed as
\begin{align}
 \label{eq:Q_t}
 Q_t(\bm{x}) &= [\mu_{t-1}^{(p)}(\bm{x}) - \beta_{t}^{1/m}\gamma_{t-1}^{2/m}(\bm{x}),
    ~\mu_{t-1}^{(p)}(\bm{x}) + \beta_{t}^{1/m}\gamma_{t-1}^{2/m}(\bm{x})].
\end{align}
Compared with credible interval of Normal distribution, additional parameter $m$ is introduced to control the $Q_t(\bm{x})$.
In Section \ref{sec:theory}, we discuss in depth for the details of $\beta_t$ and $m$ from theoretical viewpoint.

In the development of the proposed algorithms, for theoretically rigorous arguments, we use the following slightly modified versions of $p_{\rm upper}(\bm x)$ and $h$ which are characterized by a parameter $\eta > 0$:
\begin{align*}
    p_{t-1;\eta}(\bm{x}) &= \int_{\Omega}
    \1[f(\bm{x}, \bm{w}) > h_{t-1, \bm{x}, \bm{w}; \eta}] p(\bm{w}) \text{d}\bm{w}, \\
    h_{t-1, \bm{x}, \bm{w}; \eta} &= \begin{cases}
        h + 2\eta & \text{if}~|\mu_{t-1}(\bm{x}, \bm{w}) - h| < \eta,\\
        h & \text{otherwise}
        \end{cases}.
\end{align*}
In \S4, we show that, given the desired accuracy parameter $\epsilon$ (see \S2), the parameter $\eta$ can be uniquely determined.
%
%
In what follows,
by replacing $h$ in $\mu^{(p)}_{t-1}(\bm x)$ and $\gamma^2_{t-1}(\bm x)$ with $h_{t-1,\bm x,\bm w,\eta}$,
we similarly define $\mu^{(p)}_{t-1;\eta}(\bm x)$ and $\gamma^2_{t-1;\eta}(\bm x)$.
Furthermore,
by replacing $\mu^{(p)}_{t-1} (\bm x)$ and $\gamma^2_{t-1} (\bm x)$ in \eq{eq:Q_t} with
$\mu^{(p)}_{t-1;\eta} (\bm x)$ and $\gamma^2_{t-1;\eta} (\bm x)$,
we similarly define
$Q_{t;\eta} (\bm x)$.
See Appendix \ref{derivation_mean_upper} for details.

\subsection{Optimization}
In this subsection, we propose two AL methods to find maximizer of $p_{\text{upper}}$.

\paragraph{Upper Confidence Bound-based (UCB-based) strategy}
First, we propose a UCB based method with the following AFs at step $t$:
\begin{align}
 \label{eq:x_t}
 \bm{x}_t &= \argmax_{\bm{x} \in \mathcal{X}} ~\mu_{t-1;\eta}^{(p)}(\bm{x}) + \beta_t^{1/m}\gamma_{t-1;\eta}^{2/m}(\bm{x}), \\
 \label{eq:w_t}
 \bm{w}_t &= \argmax_{\bm{w} \in \Omega} ~\Phi\left(\frac{\mu_{t-1}(\bm{x}_t, \bm{w}) - h_{t-1, \bm{x}_t, \bm{w}; \eta}}{\sigma_{t-1}(\bm{x}_t, \bm{w})}\right)
    \left\{1 - \Phi\left(\frac{\mu_{t-1}(\bm{x}_t, \bm{w}) - h_{t-1, \bm{x}_t, \bm{w};\eta}}{\sigma_{t-1}(\bm{x}_t, \bm{w})}\right)\right\},
\end{align}
where $\beta_t > 0$ and $m \geq 2$ are parameters that
control the exploration and exploitation tradeoff.
Hereafter, we call this strategy Bayesian Probability Threshold (BPT)-UCB.
Algorithm \ref{alg:rbqo_ucb} shows the pseudocode of BPT-UCB algorithm.
\begin{algorithm}[t]
    \caption{BPT-UCB}
    \label{alg:rbqo_ucb}
    \begin{algorithmic}
        \REQUIRE Budget $T$, GP prior $\mathcal{GP}(0, k)$,
        ~$\eta \geq 0$, $\{\beta_t\}_{t \leq T}$, $m \geq 2$
        \FOR {$t = 1$ to $T$}
            \STATE Compute $\mu_{t-1;\eta}^{(p)}(\bm{x}),~\gamma_{t-1;\eta}^2(\bm{x})$ for all $\bm{x} \in \mathcal{X}$.
            \STATE Choose ($\bm{x}_t$, $\bm{w}_t$) from (\ref{eq:x_t}) and (\ref{eq:w_t}).
            \STATE Observe $y_t = f(\bm{x}_t, \bm{w}_t) + \varepsilon_t$.
            \STATE Update GP by adding $((\bm{x}_t, \bm{w}_t), y_t)$.
        \ENDFOR
        \ENSURE ${\rm argmax}_{\bm{x} \in \left\{ \bm{x}_1, \ldots, \bm{x}_T \right\}}~\mu_{T;\eta}^{(p)}(\bm{x})$.
    \end{algorithmic}
\end{algorithm}

\paragraph{Thompson Sampling based strategy}
We also propose a Thompson Sampling based strategy,
in which $\bm{x}_t$ is selected
according to the posterior probability such that
$p_{\rm upper}(\bm{x})$
is maximized, while $\bm{w}_t$ is selected in the same way as BPT-UCB.
%
%
Hereafter we call this strategy BPT-TS.
Specifically, the difference from BPT-UCB is that $\hat{f}$ is first sampled from $\mathcal{GP}(\mu_{t-1}, k_{t-1})$, where $k_{t-1}$ is the posterior covariance function at step $t-1$.
Then, the design parameter is chosen as
$\bm{x}_t = {\rm argmax}_{\bm{x} \in \mathcal{X}} \int_{\Omega}\1\left[ \hat{f}(\bm{x}, \bm{w}) > h \right]p(\bm{w})\text{d}\bm{w}$.
%



The two proposed methods BPT-UCB and BPT-TS have both advantages and drawbacks.
An advantage of BPT-TS is that it does not have hyperparameters (whereas BPT-UCB has two hyperparameters $\beta_t$ and $m$).
On the other hand, BPT-UCB is computationally more efficient than BPT-TS.
Specifically, when $\mathcal{X}$ and $\Omega$ are finite sets, BPT-TS requires $O(|\mathcal{X}|^2|\Omega|^2)$ computational cost which is prohibitive when $\mathcal{X}$ and $\Omega$ are large (in contrast to $O(|\mathcal{X}||\Omega|)$ for BPT-UCB).
Moreover, if $\mathcal{X}$ or $\Omega$ is continuous set, BPT-TS needs to resort on approximate posterior sampling strategies (e.g., \cite{rahimi2008random}), which is only applicable for restricted kernel classes.
Therefore, it would be beneficial to use the two proposed methods differently depending on the situation.

  \subsection{Level Set Estimation}
In this subsection, we propose an AL method to for LSE of $p_{\text{upper}}$.
Using the credible interval
$Q_{t;\eta}(\bm{x}) = [l_{t;\eta}(\bm{x}),~u_{t;\eta}(\bm{x})]$, the superlevel set $\mathcal{H}_t$ and the sublevel set $\mathcal{L}_t$ at step $t$ as:
\begin{equation}
 \label{eq:straddle}
  \mathcal{H}_t = \left\{ \bm{x} \in \mathcal{X} \mid l_{t;\eta}(\bm{x}) > \alpha - \epsilon/2 \right\},
  ~ \mathcal{L}_t = \left\{\bm{x} \in \mathcal{X} \mid u_{t;\eta}(\bm{x}) < \alpha + \epsilon/2 \right\}.
\end{equation}
Furthermore, we define unclassified set $\mathcal{U}_t$ as $\mathcal{U}_t = \mathcal{X} \backslash (\mathcal{H}_t \cup \mathcal{L}_t)$.

As the AF for $\bm x_t$, we use the straddle based criteria \cite{bryan2006active, gotovos2013active}:
\begin{align}
    \label{eq:strx}
 \bm{x}_t = \argmax_{\bm{x} \in \mathcal{X}} ~ \text{STR}_t(\bm{x}),
 \text{ where }
 \text{STR}_t(\bm{x}) := \min \left\{u_{t;\eta}(\bm{x}) - \alpha,~\alpha - l_{t;\eta}(\bm{x})\right\}
\end{align}
and
$\bm w_t$
is selected in the same way as
\eq{eq:w_t}.
Hereafter, we call the method as BPT-LSE.
Algorithm~\ref{alg:rbqlse} shows the pseudocode.
\begin{algorithm}[t]
    \caption{BPT-LSE}
    \label{alg:rbqlse}
    \begin{algorithmic}
        \REQUIRE GP prior $\mathcal{GP}(0,\ k)$,
        ~$\eta \geq 0$, $\{\beta_t\}_{t \leq T}$, $m \geq 2$, $\epsilon > 0$, threshold $\alpha$
        \STATE $\mathcal{H}_0 \leftarrow \emptyset$, $\mathcal{L}_0 \leftarrow \emptyset$, $\mathcal{U}_0 \leftarrow \mathcal{X}$, $t \leftarrow 1$
        \WHILE{$\mathcal{U}_{t-1} \neq \emptyset$}
            \STATE Compute $\mu_{t-1;\eta}^{(p)}(\bm{x}),~\gamma_{t-1;\eta}^2(\bm{x})$ and $\text{STR}_t(\bm{x})$ for all $\bm{x} \in \mathcal{X}$.
            \STATE Choose ($\bm{x}_t$, $\bm{w}_t$) from (\ref{eq:strx}) and (\ref{eq:w_t}).
            \STATE Observe $y_t \leftarrow f(\bm{x}_t, \bm{w}_t) + \varepsilon_t$
            \STATE Update GP by adding $((\bm{x}_t, \bm{w}_t), y_t)$ and compute $\mathcal{H}_t, \mathcal{L}_t$ and $\mathcal{U}_t$. 
            \STATE $t \leftarrow t + 1$
        \ENDWHILE
        \STATE $\hat{\mathcal{H}} \leftarrow \mathcal{H}_{t-1}, \hat{\mathcal{L}} \leftarrow \mathcal{L}_{t-1}$
        \ENSURE Estimated Set $\hat{\mathcal{H}}, \hat{\mathcal{L}}$
    \end{algorithmic}
\end{algorithm}

\section{Theoretical Results}\label{sec:theory}
 In this section, we show theoretical guarantees for the proposed algorithm (detail proofs are given in Appendix).
First, we define the mutual information between $f$ and observations.
Let $A=\{ {\bm{a}}_1,\ldots, {\bm{a}}_k \} $ be a finite subset of $\mathcal{X} \times \Omega$,
and let ${\bm{y}}_A$ be a vector whose $i$th element is $y_{\bm{a}_i } = f({\bm{a}}_i) + \varepsilon _{ \bm{a} _i} $.
Moreover, let $I ({\bm{y}}_A;f)$ be the mutual information between $f$ and ${\bm{y}}_A$.
Then, we define the maximum information gain $\kappa _T$ after $T$ rounds as
$
\kappa_T = \max_{ A \subset \mathcal{X} \times \Omega; |A|=T} { I} ({\bm{y}}_A;f).
$
%
The following theorem gives the upper bound of the cumulative $\epsilon$-regret for BPT-UCB:
\begin{theorem}\label{thm:fin_regret}
    Let $\delta \in (0, 1)$, $m \geq 2$, $\epsilon>0$, $\beta_t = |\mathcal{X}|\pi^2t^2/(3\delta)$ and
$2 \eta =    \min \{ \epsilon \sigma_{0,min}/2, \epsilon^2 \delta \sigma_{0,min} /(8|\mathcal{X}|) \}$.
Then, running BPT-UCB with these parameters,
    the cumulative $\epsilon$-regret satisfies the following inequality:
    \begin{equation*}
        Pr   \{ R_T (\epsilon) \leq C_1 \beta_T^{1/m}\kappa_T \eta^{ -(2+1/m) },~\forall T \geq 1 \} \ge 1 - \delta,
    \end{equation*}
    where $C_1 = 8m   ( (2\pi)^{1/2m}\log(1+\sigma^{-2}) )^{-1}$.
\end{theorem}
Moreover, the following theorem gives the upper bound of the Bayes $\epsilon$-regret for BPT-TS:
\begin{theorem}\label{thm:fin_bayes_regret}
    Let $m \geq 2$, $\epsilon >0$ and  $2 \eta =    \min \{ \epsilon \sigma_{0,min}/4, \epsilon^3 \sigma_{0,min} /(32|\mathcal{X}|) \}$.
Then, running BPT-TS with these parameters, the Bayes $\epsilon$-regret satisfies the following inequality:
    \begin{equation*}
       BR_T (\epsilon) \leq \pi^2/6 + C_2T^{2/m} \kappa_T \eta^{-(2+1/m)},
    \end{equation*}
    where $C_2 = 4m|\mathcal{X}|^{1/m}  (  (2\pi)^{1/2m}(\log (1+\sigma^{-2})) )^{-1}$.
\end{theorem}
Finally,  we give the theorem about the convergence and accuracy of BPT-LSE:
\begin{theorem}\label{thm:lse_convergence}
Let $\epsilon>0$, $\alpha \in (0, 1)$, $m \geq 2$ and $C_3 = 4m(2\pi)^{-1/2m} / \log(1 + \sigma^{-2})$.
     Furthermore, let $\delta \in (0, 1)$, $\beta_t = |\mathcal{X}|\pi^2t^2 /(3\delta)$ and $2 \eta =    \min \{ \epsilon \sigma_{0,min} /4, \epsilon^2 \delta \sigma_{0,min} /(32|\mathcal{X}|)\}$.
    Then, BPT-LSE algorithm terminates after at most
     $T$ rounds, where   $T$ is the smallest positive integer satisfying
    \begin{equation}
       C_3 \eta^{-(2+1/m)}  \beta^{1/m}_T \kappa_T T^{-1}     < \epsilon/2. \label{eq:LSEbound}
    \end{equation}
Moreover, with probability at least $1-\delta$,
BPT-LSE returns $\epsilon$-accurate solution, i.e.,
the following inequality holds:
$
Pr  \{ \max _{ {\bm{x}} \in \mathcal{X} } e_\alpha ( {\bm{x}} )   \leq  \epsilon  \} \geq 1-\delta.
$
\end{theorem}

Note that upper bounds of $\kappa_T$ have been studied for some kernels \cite{srinivas2009gaussian}.
For example,  under certain conditions the orders of $\kappa_T$ in Linear and Gaussian    are respectively
$\mathcal{O}( \tilde{d} \log T )$ and $\mathcal{O}( ( \log T)^{\tilde{d}+1} )$, where $\tilde{d}=d+k$.
Moreover, for Mat\'{e}rn kernels with $\nu >1$, its order is $\mathcal{O}(T^{\tilde{d}  (\tilde{d}+1)/( 2 \nu +\tilde{d}  (\tilde{d}+1)    )  } ( \log T)  )$.
Thus, if we use sufficiently large $m$ in Theorem \ref{thm:fin_regret}--\ref{thm:lse_convergence}, $\beta^{1/m}_T \kappa_T$ can be less than $T$.
Hence, it holds that $\lim_{T \to \infty} T^{-1} BR_T (\epsilon)=0$.
Similarly, with high probability, $R_T (\epsilon )$ satisfies $\lim _{T \to \infty} T^{-1} R_T (\epsilon ) =0$.
 Moreover, $\beta^{1/m}_t \kappa_t /t $ tends to zero, i.e., there exists the positive integer $T$ satisfying \eqref{eq:LSEbound}.

\section{Numerical Experiments}
In this section we present numerical experiments both on synthetic and real problems.
Due to the space limitation,
we present the summary here
and the details are deferred to Appendix~\ref{app:detail-exeperiments}.

\paragraph{Artificial Data Experiments}
We compared the performances of the proposed methods
({\tt BPT-UCB}, {\tt BPT-TS}, and {\tt BPT-LSE})
with a variety of existing methods
on two benchmark functions in each of the optimization and the LSE setting.
The evaluation metric at step $t$ in the optimization setting is
$p_{\text{upper}}(\bm{x}^*) - p_{\text{upper}}(\hat{\bm{x}}_t)$,
where
$\hat{\bm{x}}_t$
is the estimated maximizer reported by the algorithm at step $t$\footnote{
We reported this evaluation metric in experiments because it is easy to interpret in practice. This metric is slightly different from $\epsilon$-regret which we discussed in \S4.}.
The evaluation metric at step $t$ in the LSE setting is F1-score
which is computed by treating
the estimated super/sub-level sets $\mathcal{H}$ and $\mathcal{L}$
as positively and negatively labeled instances, respectively.
As the benchmark functions in the optimization setting,
we considered 2D-Rosenbrock function
and
McCormick function.
As the benchmark functions in the LSE setting,
we considered Himmelblau function and Goldstein-Price function.
These benchmark functions are commonly used in previous related studies.
Due to the space limitation, the details of these benchmark functions are deferred to Appendix~\ref{app:detail-exeperiments}.
In the optimization setting,
we considered
{\tt GP-UCB}~\cite{srinivas2009gaussian} (with the environmental parameter $\bm w$ fixed as its mean),
{\tt StableOpt}~\cite{bogunovic2018adversarially},
{\tt BQO-EI}~\cite{nguyen2020distributionally},
and its UCB version
({\tt BQO-UCB}),
{\tt BQO-TS}~\cite{nguyen2020distributionally},
and each of their adaptive versions\footnote{In adaptive version, the estimated maximizer $\hat{\bm{x}}_t$ is chosen in the same way as the proposed method (see Appendix~\ref{app:detail-exeperiments} for the details).}
({\tt Pmax-GP-UCB}, {\tt Pmax-StableOpt}, {\tt Pmax-BQO-EI}, {\tt Pmax-BQO-UCB}, {\tt Pmax-BQO-TS})
as well as
Random Sampling ({\tt RS})
as existing methods for comparison.
In the LSE setting,
we considered
the standard LSE~\cite{bryan2006active} with the environmental parameter $\bm w$ fixed as its mean ({\tt LSE}),
the LSE version of StableOpt~\cite{bogunovic2018adversarially}
({\tt StableLSE}), the LSE version of BQO~\cite{toscano2018bayesian}
({\tt BQLSE})
and each of their adaptive versions
({\tt P-LSE}, {\tt P-StableLSE}, {\tt P-BQLSE})
as well as
Random Sampling ({\tt RS})
as existing methods for comparison.
Due to the space limitation, the details of these existing methods are deferred to Appendix~\ref{app:detail-exeperiments}.
Figures~\ref{fig:bo_benchmark_result} and \ref{fig:lse_benchmark_result}
show
the results in the optimization and the LSE setting, respectively.
In both settings, the proposed methods have better performances than existing methods.
This is reasonable since the proposed methods are developed to optimize the target tasks, while existing methods are developed to optimize different robustness measures.
In the LSE settings, some of the existing methods could rapidly increase the F1-scores in the early stage. However, since the target robustness measures in the existing methods are inconsistent with the problem setup, they are eventually outperformed by the proposed methods.
\begin{figure}[t]
    \begin{center}
        \begin{tabular}{cc}
         \includegraphics[width=0.500\linewidth]{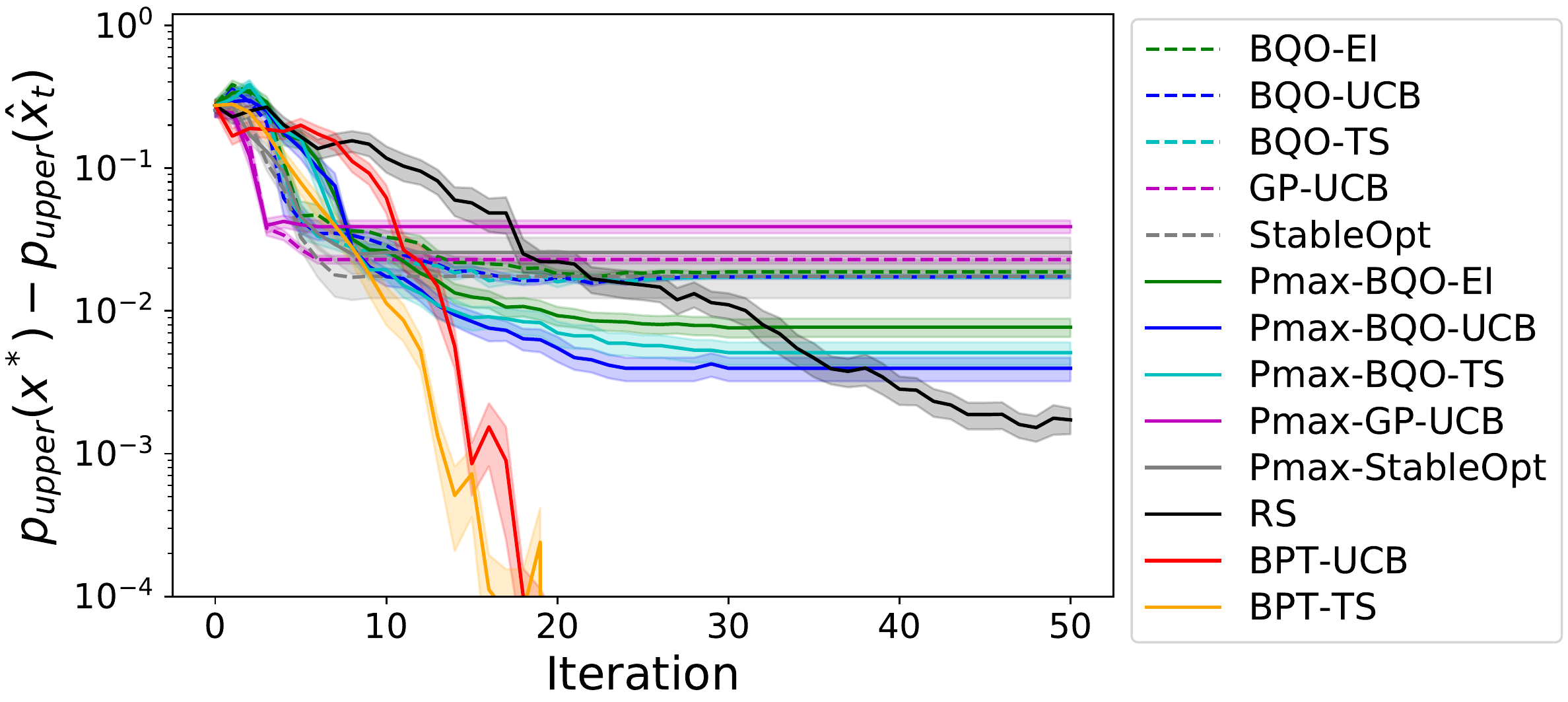} &
         \includegraphics[width=0.500\linewidth]{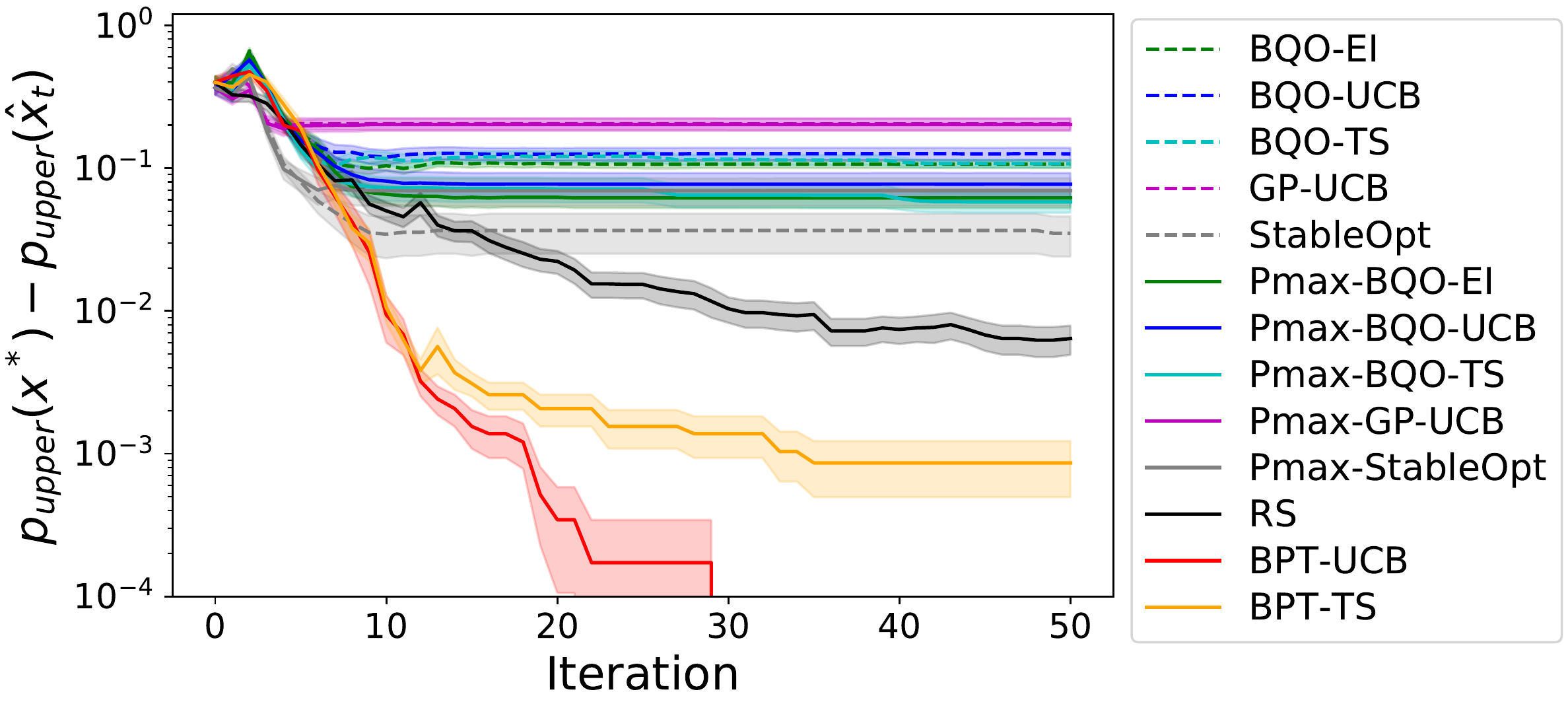} \\
         (a) 2D-Rosenbrock function &
         (b) McCormick function
        \end{tabular}
   \end{center}
 \caption{
 The experimental results in the optimization setting with two benchmark functions.
 These plots indicate the average performances over $50$ trials.
 }
 \label{fig:bo_benchmark_result}
\end{figure}

\begin{figure}[t]
    \begin{center}
        \begin{tabular}{cc}
         \includegraphics[width=0.500\linewidth]{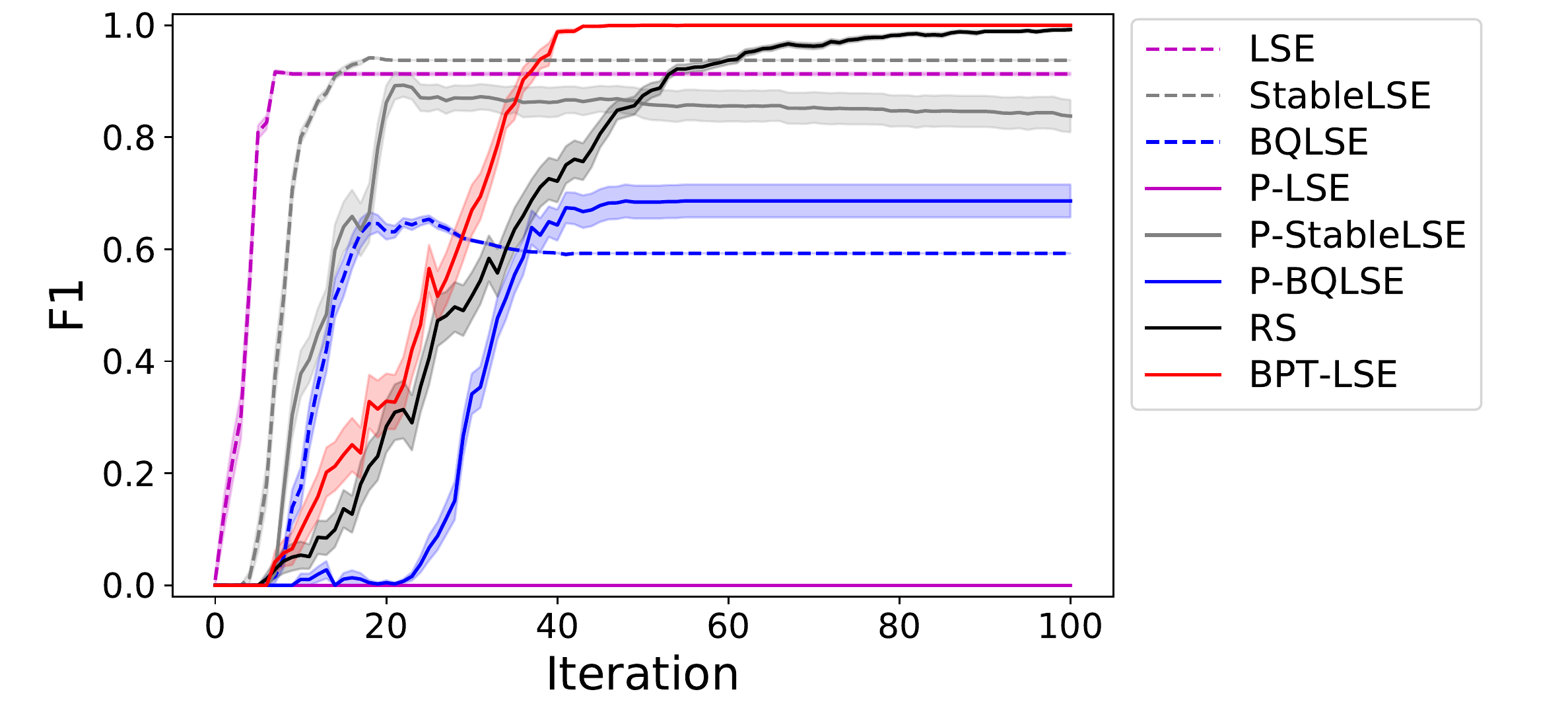} &
         \includegraphics[width=0.500\linewidth]{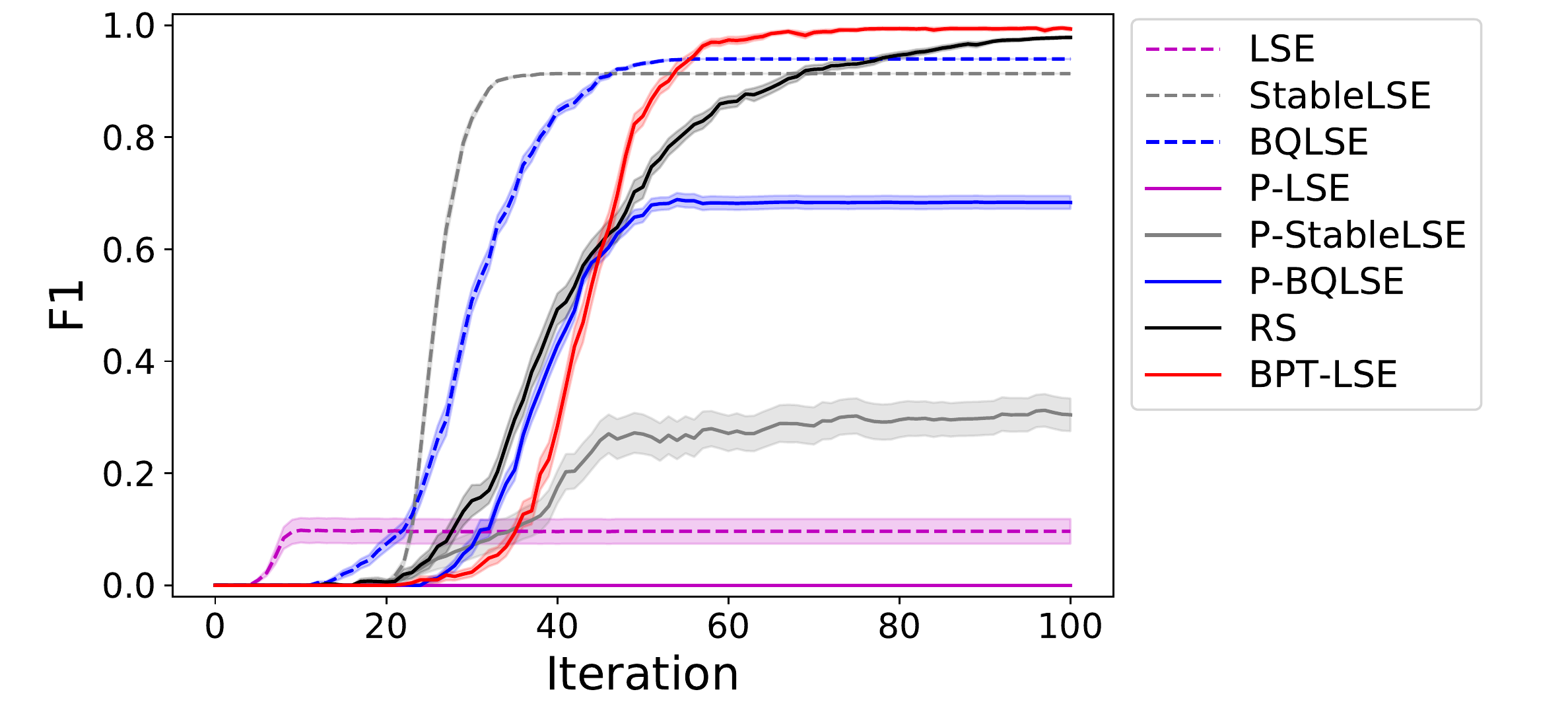} \\
         (a) Himmelblau function &
         (b) Goldstein-Price function
        \end{tabular}
   \end{center}
 \caption{
 The experimental results in the LSE setting with two benchmark functions.
 These plots indicate the average F1-scores over $50$ trials.
 }
 \label{fig:lse_benchmark_result}
\end{figure}

\paragraph{Real Data Experiments}
We applied the proposed methods in the optimization and the LSE setting to  \emph{Newsvendor problem under dynamic consumer substitution}~\cite{mahajan2001stocking} and \emph{Infection control problem}~\cite{kermack1927contribution}, respectively.
Both of them are simulation-based decision making problem in which the goal is to find the optimal decisions with as small number of simulation runs as possible.
The goal of the former problem is to optimize the initial inventory level of each product in order to maximize the revenue that is determined by uncertain customer purchasing behavior.
Here, the design parameter $\bm x$ is the initial inventory levels of the products, while the environmental parameter $\bm w$ is customer purchasing behavior which are assumed to follow mutually independent Gamma distribution.
This problem was also studied in \cite{toscano2018bayesian} for demonstrating the performance of BQO.
The goal of the latter problem is to decide the target infection rate to minimize the associated economic risk.
More specifically, we want to find the range of the target infection rate so that it achieves the economic risk at tolerable level $h$ with sufficiently high probability.
Here, the design parameter $\bm x$ is the target infection rate, while the environmental parameter $\bm w$ is the recovery rate because the latter is uncertain and uncontrollable in reality.
We describe more details in Appendix~\ref{app:detail-exeperiments}.
Figure~\ref{fig:real_results} shows the results of these real data experiments.
We observed that the proposed methods
({\tt BPT-UCB}, {\tt BPT-TS} for the optimization setting and {\tt BPT-LSE})
consistently outperformed the existing methods.
Deeper discussion on the experimental results are also provided in Appendix~\ref{app:detail-exeperiments}.
\begin{figure}[t]
 \begin{center}
  \begin{tabular}{cc}
   \includegraphics[width=0.500\linewidth]{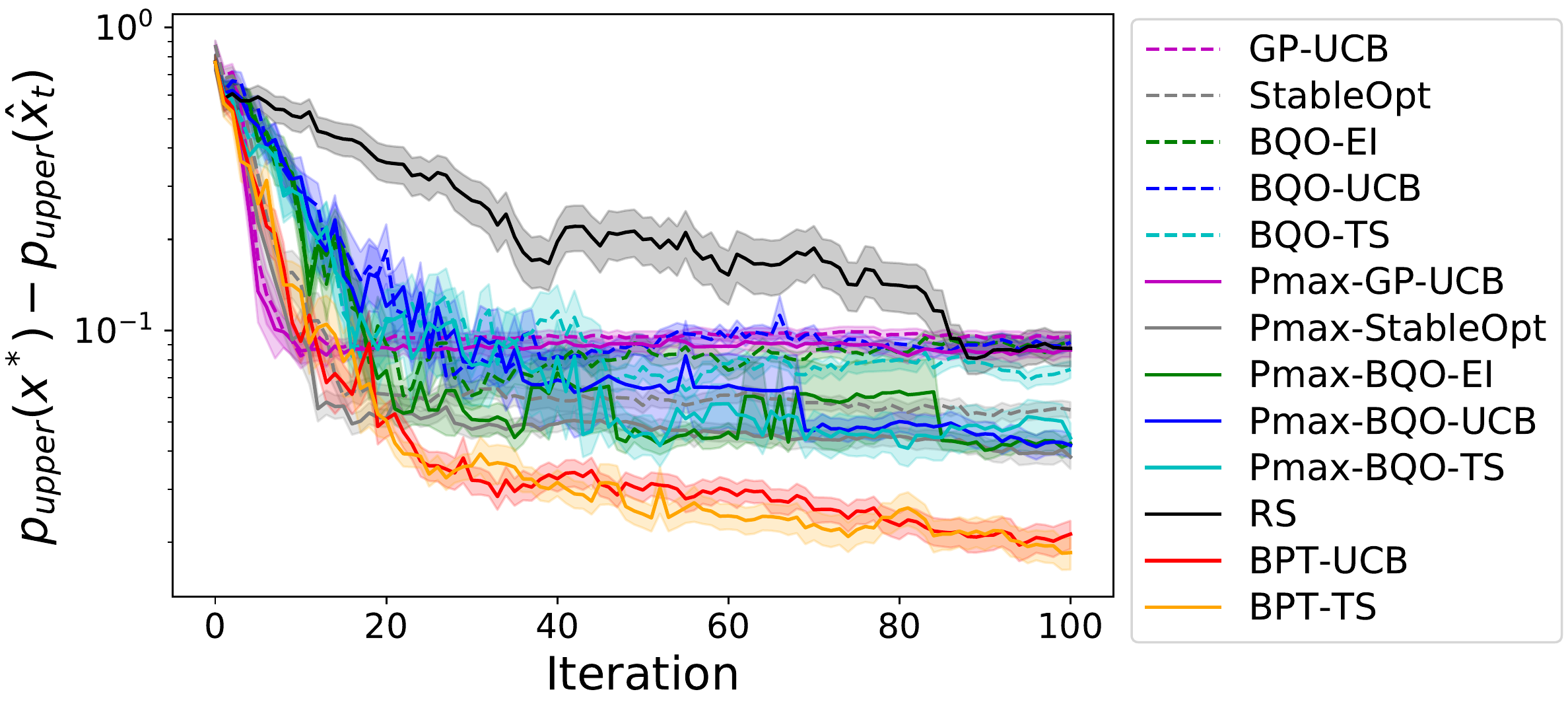} &
   \includegraphics[width=0.500\linewidth]{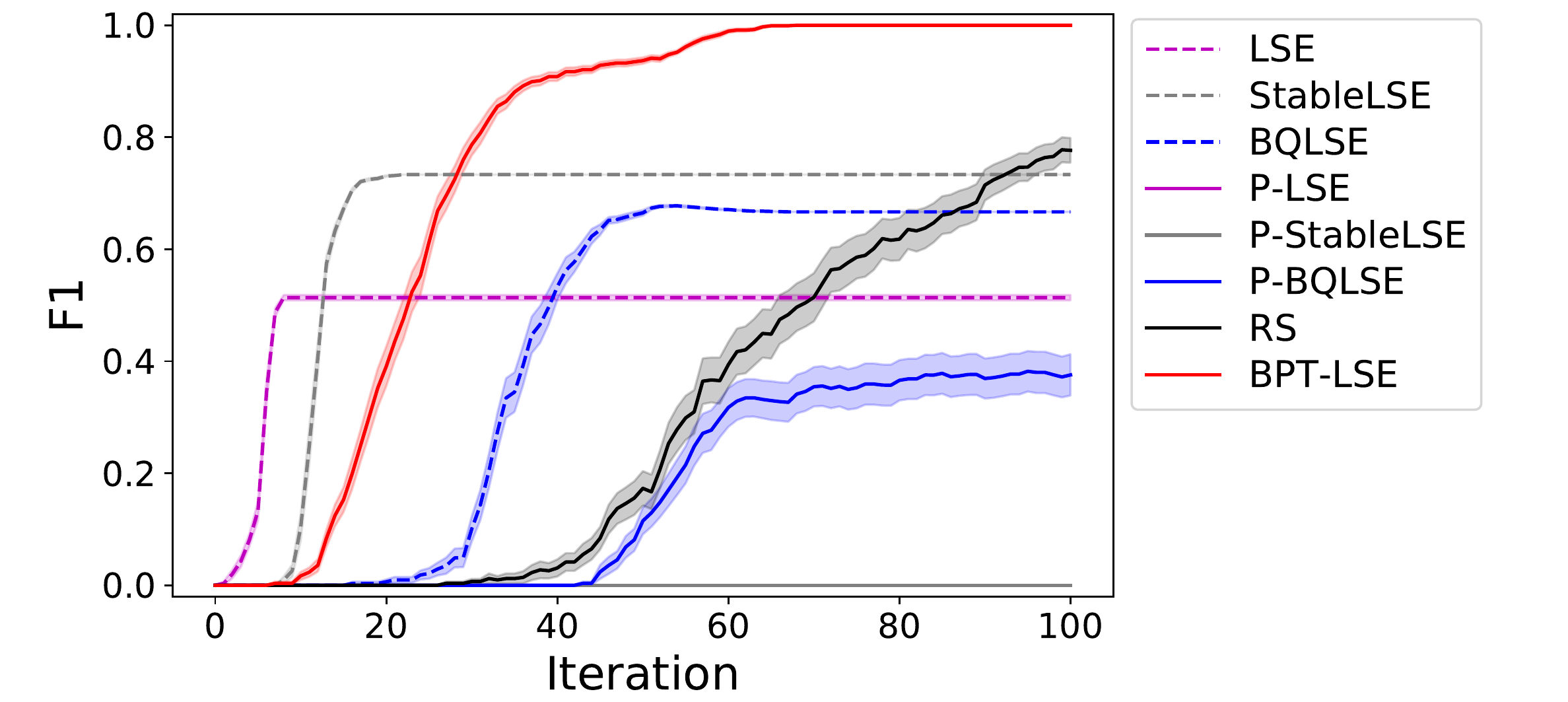} \\
   (a) Optimization setting  \emph{Newsvendor}&
   (b) LSE setting on \emph{Infection Control}
  \end{tabular}
 \end{center}
 \caption{
 The experimental results in real data experiments.
 The left and the right plots show
 the results in the optimization setting on
 \emph{Newsvendor} problem
 and
 the results in the LSE setting on
 \emph{Infection Control} problem.
 These plots indicate the average F1-scores over $50$ trials.
 }
 \label{fig:real_results}
\end{figure}

\section{Conclusion}
We proposed AL methods for optimization and Level Set Estimation (LSE) of Probabilistic Threshold Robustness (PTR) measure under uncertain and uncontrollable environmental parameter.
We showed that the proposed AL methods have theoretically desirable properties and perform better than existing methods in numerical experiments.
One of the key issues for the future is to consider the case where the distribution of environmental parameter is unknown. 

\section*{Broader Impact}
In the fields of manufacturing engineering and materials science, active learning methods are gaining attention as an efficient experimental design method of product development.
A common approach used in this context is Bayesian Optimization (BO) by which engineers or scientists expect to find the optimal design parameter with as small number of experiments as possible.
However, in practice, uncontrolled environmental parameter must often be taken into account.
In such cases, it is necessary to robustly determine design parameter that meet certain requirements, even if they are not necessarily optimal, for variations in environmental parameter.
This study presents a formulation and a solution to this practical problem.
We expect that this study further promotes the use of machine learning in the field of product development.

\section{Acknowledgements}
This work was partially supported by MEXT KAKENHI (20H00601, 16H06538), JST CREST (JPMJCR1502), and RIKEN Center for Advanced Intelligence Project.

\bibliography{myref}
\bibliographystyle{unsrt}

\clearpage

\appendix

\section{The mean and an upper bound of the variance of PTR measure}
\label{derivation_mean_upper}
In this section, we derive  $\mu^{(p)}_{t-1} ({\bm x})$, $\gamma^2_{t-1} ({\bm x})$, $\mu^{(p)}_{t-1;\eta} ({\bm  x})$ and $\gamma^2_{t-1;\eta} ({\bm x})$. 
Since $f$ follows GP, integrands of (\ref{eq:pup}) $\1[f({\bm x}, {\bm w}) > h]:
\mathcal{X} \times \Omega \rightarrow \{0, 1\}$ follow certain stochastic process, which is not GP. 
Thus, (\ref{eq:pup}) becomes the integral of stochastic process. 
Then, by using known results about the integral of stochastic process (see, e.g., \cite{papoulis2002probability}), 
at time $t-1$  its mean $\mu_{t-1}^{(p)}(\bm{x})$ and variance $\sigma^{(p)2}_{t-1} ({\bm x})$ can be expressed as follows:
\begin{align*}
    \mu_{t-1}^{(p)} &= \mathbb{E}[p(\bm{x})] = \int_{\Omega} \mathbb{E}[\1[f(\bm{x}, \bm{w}) > h]]p(\bm{w}) \text{d}\bm{w}, \\
   \sigma^{(p)2}_{t-1} ({\bm x})= \mathbb{V}[p(\bm{x})] &= \int_{\Omega} \int_{\Omega}\mathrm{Cov}[\1[f(\bm{x}, \bm{w}) > h],
    \1[f(\bm{x}, \bm{w} ^\prime) > h]]p(\bm{w})p(\bm{w} ^\prime)\text{d}\bm{w}\text{d}\bm{w} ^\prime,
\end{align*}
where the expectation, covariance and variance are taken with respect to the posterior of $f$. 
Furthermore, from $\mathrm{Cov}(X, Y) \leq (\mathbb{V}[X] + \mathbb{V}[Y])/2$, we define the
variance upper bound $\gamma_{t-1}^2(\bm{x})$ as:
\begin{align*}
    \sigma^{(p)2}_{t-1} ({\bm x}) &\leq \int_{\Omega} \int_{\Omega} \frac{(\mathbb{V}[\1[f(\bm{x}, \bm{w}) > h]]
    + \mathbb{V}[\1[f(\bm{x}, \bm{w}') > h]])}{2}p(\bm{w})p(\bm{w}^\prime)\text{d}\bm{w}\text{d}\bm{w}^\prime \\
    & =\int_{\Omega} \mathbb{V}[\1[f(\bm{x}, \bm{w}) > h]]p(\bm{w}) \text{d}\bm{w} \\
    & \coloneqq \gamma_{t-1}^2(\bm{x}).
\end{align*}
Here, from $f(\bm{x}, \bm{w}) \sim \mathcal{N}(\mu_{t-1}(\bm{x}, \bm{w}), \sigma^2_{t-1}(\bm{x}, \bm{w}))$ at time $t-1$, 
$\1[f(\bm{x}, \bm{w}) > h]$ follows Bernoulli distribution with mean 
$\Phi \left( \frac{\mu_{t-1}(\bm{x}, \bm{w}) - h}{\sigma_{t-1}(\bm{x}, \bm{w})}\right)$, where
$\Phi(\cdot)$ is a cumulative distribution function of standard Normal distribution.
Therefore, $\mu_{t-1}^{(p)}(\bm{x})$ and $\gamma_{t-1}^{(p)}(\bm{x})$ can be expressed as
\begin{align*}
    \mu_{t-1}^{(p)}(\bm{x}) &= \int_{\Omega} \Phi
    \left( \frac{\mu_{t-1}(\bm{x}, \bm{w}) - h}{\sigma_{t-1}(\bm{x}, \bm{w})}\right) p(\bm{w}) \text{d}\bm{w}, \\
    \gamma_{t-1}^2(\bm{x}) &= \int_{\Omega} \Phi
    \left( \frac{\mu_{t-1}(\bm{x}, \bm{w}) - h}{\sigma_{t-1}(\bm{x}, \bm{w})}\right)
    \left\{ 1 - \Phi
    \left( \frac{\mu_{t-1}(\bm{x}, \bm{w}) - h}{\sigma_{t-1}(\bm{x}, \bm{w})}\right) \right\}
     p(\bm{w}) \text{d}\bm{w}.
\end{align*}
Similarly, since $h_{t-1,{\bm x}, {\bm w};\eta } $ is a deterministic function at time $t-1$, 
$\1[f(\bm{x}, \bm{w}) > h_{t-1,{\bm x}, {\bm w};\eta } ]$ follows Bernoulli distribution with mean 
$\Phi \left( \frac{\mu_{t-1}(\bm{x}, \bm{w}) - h_{t-1,{\bm x}, {\bm w};\eta } }{\sigma_{t-1}(\bm{x}, \bm{w})}\right)$. 
Hence, the mean $\mu_{t-1;\eta}^{(p)}(\bm{x})$ and upper bound of variance $\gamma_{t-1;\eta}^2(\bm{x})$
of $p_{t-1;\eta}(\bm{x})$ at time $t-1$ are given by 
\begin{align*}
    \mu_{t-1;\eta}^{(p)}(\bm{x}) &= \int_{\Omega} \Phi
    \left( \frac{\mu_{t-1}(\bm{x}, \bm{w}) - h_{t-1, \bm{x}, \bm{w};\eta}}{\sigma_{t-1}(\bm{x}, \bm{w})}\right) p(\bm{w}) \text{d}\bm{w}, \\
    \gamma_{t-1;\eta}^2(\bm{x}) &= \int_{\Omega} \Phi
    \left( \frac{\mu_{t-1}(\bm{x}, \bm{w}) - h_{t-1, \bm{x}, \bm{w};\eta}}{\sigma_{t-1}(\bm{x}, \bm{w})}\right)
    \left\{ 1 - \Phi
    \left( \frac{\mu_{t-1}(\bm{x}, \bm{w}) - h_{t-1, \bm{x}, \bm{w};\eta}}{\sigma_{t-1}(\bm{x}, \bm{w})}\right) \right\}
     p(\bm{w}) \text{d}\bm{w}.
\end{align*}

\section{Details of the modified version of PTR measure and $\epsilon$-regret}
\label{app:epsilon-regret}
In this section, we explain about   inaccurate behaviors of the predicted distribution for $p_{\rm upper} ({\bm x})$. 
After that, we also explain the motivation of   $p_{t-1;\eta} ({\bm x})$ and $\epsilon$-regret. 
\subsection{Inaccurate behaviors of the predicted distribution}\label{subsec:inaccurate_behavior}
As mentioned in \S \ref{sec:Preliminaries}, when $f({\bm x} ,{\bm w} )$ is exactly $h$, 
the prediction of $p_{\rm upper} ({\bm x})$ is still inaccurate no matter how much we evaluate  $f({\bm x} ,{\bm w} )$ because 
$f({\bm x} ,{\bm w} ) =h$ and observations of $f({\bm x} ,{\bm w} )$ have noise. 
For example, as an extreme case, let $\mathcal{X} =\{ x \} \subset \mathbb{R}$, $\Omega =\{ w \} \subset \mathbb{R}$ and 
$f(x,w) =h$. 
Assume that $f \sim \mathcal{G} \mathcal{P} (0, k)$, where $k((x,w),(x,w))=1$. 
Then, the posterior mean $\mu_t (x,w)$ and variance $\sigma^2_t (x,w)$ can be given by
\begin{align}
\mu_t (x,w) &=  {\bm 1} ^\top _t ( {\bm 1}_t {\bm 1}^\top _t + \sigma^2 {\bm I}_t ) ^{-1} {\bm y}_t, \label{eq:one_dim_mean} \\
\sigma^2_t (x,w) &= 1-  {\bm 1} ^\top _t ( {\bm 1}_t {\bm 1}^\top _t + \sigma^2 {\bm I}_t ) ^{-1} {\bm 1}_t, \label{eq:one_dim_var} 
\end{align}
where ${\bm 1}_t $ is a $t$-dimension vector in which all elements are 1. 
Here, noting that 
$$
( {\bm 1}_t {\bm 1}^\top _t + \sigma^2 {\bm I}_t ) ^{-1} = \sigma^{-2} {\bm I}_t -\frac{\sigma^{-4}  {\bm 1}_t {\bm 1}^\top _t }{1+t \sigma^{-2}},
$$
 \eqref{eq:one_dim_mean} and  \eqref{eq:one_dim_var} can be rewritten as follows:
\begin{align*}
\mu_t (x,w) &= \sigma^{-2} t \bar{y}_t -\frac{ \sigma^{-4} t^2 \bar{y}_t  }{1+t \sigma^{-2}} = \frac{ \sigma^{-2} t \bar{y}_t }{1+ t \sigma^{-2}}   = \frac{1}{ \sigma^2/t +1 } \bar{y}_t     , \\
\sigma^2_t (x,w) &= 1- \left (\sigma^{-2} t - \frac{\sigma^{-4} t^2}{1+ t \sigma^{-2}} \right ) = 1-\frac{\sigma^{-2} t }{1+t \sigma^{-2}}=\frac{1}{1+t \sigma^{-2}},
\end{align*}
where $\bar{y}_t = t^{-1} \sum_{i=1}^t y_i $. 
Therefore, the posterior distribution of $f(x,w)$ can be expressed as 
$$
 \frac{1}{ \sigma^2/t +1 } \bar{y}_t + \frac{1}{\sqrt{1+t \sigma^{-2}}} Z,
$$
where $Z \sim \mathcal{N} (0,1)$. 
Hence, the posterior distribution of $f(x,w) -h$ is given by 
\begin{align*}
 \frac{1}{ \sigma^2/t +1 } \bar{y}_t + \frac{1}{\sqrt{1+t \sigma^{-2}}} Z-h &= 
 \frac{1}{ \sigma^2/t +1 } (\bar{y}_t-h) + \frac{1}{\sqrt{1+t \sigma^{-2}}} Z + \frac{h}{ \sigma^2/t +1 } -h \\
&= \frac{1}{ \sigma^2/t +1 } (\bar{y}_t-h) + \frac{1}{\sqrt{1+t \sigma^{-2}}} Z -  \frac{h\sigma^2/t }{ \sigma^2/t +1 } .
\end{align*}
Thus, we get 
\begin{align*}
f(x,w) >h & \Leftrightarrow f(x,w)-h >0 \\
& \Leftrightarrow   \sqrt{t \sigma^{-2}  }  \{ f(x,w)-h    \}  >0 \\
&\Leftrightarrow  \frac{  \sqrt{t \sigma^{-2}  } }{ \sigma^2/t +1 } (\bar{y}_t-h) + \frac{  \sqrt{t \sigma^{-2}  } }{\sqrt{1+t \sigma^{-2}}} Z -  \frac{h\sigma^2/t }{ \sigma^2/t +1 }   \sqrt{t \sigma^{-2}  } >0.
\end{align*}
Moreover, noting that $\sqrt{t \sigma^{-2} } (\bar{y}_t -h) \sim \mathcal{N} (0,1)$, 
$(\bar{y}_t -h )$ and $Z$ are mutually independent and 
$$
\lim_{t \to \infty }  \frac{1}{\sigma^2/t +1} =1, \ 
\lim_{t \to \infty }  \frac{  \sqrt{t \sigma^{-2}  } }{\sqrt{1+t \sigma^{-2}}}  =1 , \ 
\lim_{t \to \infty }  \frac{h\sigma^2/t }{ \sigma^2/t +1 }   \sqrt{t \sigma^{-2}  } =0,
$$ 
we have 
$$
\frac{  \sqrt{t \sigma^{-2}  } }{ \sigma^2/t +1 } (\bar{y}_t-h) + \frac{  \sqrt{t \sigma^{-2}  } }{\sqrt{1+t \sigma^{-2}}} Z -  \frac{h\sigma^2/t }{ \sigma^2/t +1 }   \sqrt{t \sigma^{-2}  } \xrightarrow{d} \mathcal{N} (0,2),
$$
where $\xrightarrow{d}$ means convergence in distribution. 
This implies that 
\begin{align*}
\lim_{t \to \infty} Pr (p_{\rm upper } (x) =1 ) &= \lim_{t \to \infty} Pr (f(x,w)>h) = \lim_{t \to \infty}  Pr (  \sqrt{t \sigma^{-2}} (f(x,w)-h ) >0) \\
 &= Pr (\mathcal{N} (0,2) >0) = 0.5.
\end{align*}
Hence, the prediction of $p_{\rm upper} (x)$ is still inaccurate no matter how much  observations of $f(x,w)$ including noise are evaluated.

\subsection{Motivations of the modified version of PTR measure and $\epsilon$-regret} 
In order to avoid the issue explained in subsection \ref{subsec:inaccurate_behavior}, we consider the posterior distribution of 
$p_{t-1;\eta} ({\bm x}) $, instead of $p_{\rm upper} ({\bm x})$. 
Our idea is based on the following inequality:
\begin{align}
p_{t-1;\eta} ({\bm x} ) \leq p_{\rm upper} ({\bm x}) \leq p_{t-1;\eta} ({\bm x} ) + \int_{\Omega} \1\left[  h+2 \eta \geq  f(\bm{x}, \bm{w}) > h\right] p(\bm{w}) \text{d}\bm{w}. \label{eq:basic_inequality}
\end{align}
Note that \eqref{eq:basic_inequality} holds for any $t \geq 1$, ${\bm x} \in \mathcal{X}$, threshold $h$, $\eta > 0$, $p({\bm w})$ and $f$. 
In addition, if ${\bm w}_t$ is chosen by using \eqref{eq:w_t}, the upper bound $\gamma^2_{t-1;\eta} ({\bm x}_t)$ of the posterior variance of $p_{t-1;\eta} ({\bm x}_t)$  satisfies 
$$
\gamma^{2/m}_{t-1;\eta} ({\bm x}_t) \leq C_{m,\eta} \sigma^2_{t-1} ({\bm x}_t,{\bm w}_t ), \ \forall m \geq 2, \ \eta >0,
$$
where $C_{m,\eta}$ is a constant (see, \eqref{eq:gamma_upper}).
Therefore, the prediction of $p_{t-1;\eta} ({\bm x}_t ) $ becomes more accurate if   $\gamma^2_{t-1;\eta} ({\bm x}_t)$ becomes small. 
In this sense, the posterior distribution of $p_{t-1;\eta} ({\bm x})$  is more tractable than that of $p_{\rm upper} ({\bm x})$.
Furthermore,  for any $\epsilon >0$, the following holds with high probability if an appropriate $\eta $ is chosen (see, Lemma \ref{lem:prior_pro}):
$$
 \int_{\Omega} \1\left[  h+2 \eta \geq  f(\bm{x}, \bm{w}) > h\right] p(\bm{w}) \text{d}\bm{w} < \epsilon.
$$  
Hence, with high probability, the ordinary regret $r_t (0)$ can be bounded as follows:
$$
r_t (0) = p_{\rm upper } ({\bm x}^\ast ) - p_{\rm upper } ({\bm x}_t) \leq p_{t-1;\eta} ({\bm x}^\ast ) + \epsilon - p_{t-1;\eta} ({\bm x}_t ) .
$$
Thus, from the definition of $\epsilon$-regret, the following holds with high probability:
$$
r_t (\epsilon) = r_t (0) - \epsilon \leq  p_{t-1;\eta} ({\bm x}^\ast )  - p_{t-1;\eta} ({\bm x}_t ).
$$
Note that the right hand side in this inequality has only $p_{t-1;\eta} ({\bm x})$ which is more tractable, not $p_{\rm upper} ({\bm x})$. 
Therefore, by considering $\epsilon$-regret, theoretical guarantees for $R_t (\epsilon) $ and  $BR_t (\epsilon )$ based on $\epsilon$-regret  can be obtained (see, \S \ref{sec:theory}).

\section{Proofs}
In this section, we show the theoretical guarantees for our proposed methods.
First, we define the random variable 
 $\tilde{p} _{2\eta} ({\bm x})$ as
\begin{equation*}
 \tilde{p} _{2 \eta} ({\bm x}) =  \int_{\Omega} \1\left[  h+2 \eta \geq  f(\bm{x}, \bm{w}) > h\right] p(\bm{w}) \text{d}\bm{w}.
\end{equation*}
\subsection{Regret Bound of BPT-UCB}
In this subsection, we show the upper for the cumulative $\epsilon$-regret in BPT-UCB.
The basic techniques used in this section are based on \cite{srinivas2009gaussian}. 
\begin{lemma}\label{lem:prior_pro}
Let $\delta \in (0,1)$, $\epsilon >0$ and $2\eta =\min \{ \frac{\epsilon \sigma_{0,min}}{2}, \frac{\epsilon^2 \delta \sigma_{0,min}}{8|\mathcal{X}|}\}$. Then, with probability at least $1- \delta/2$, the following inequality holds for any ${\bm x} \in \mathcal{X}$:
$$
\tilde{p}_{2 \eta} ({\bm x} ) < \epsilon.
$$
\end{lemma}
\begin{proof}
From Chebyshev's inequality, for any $\tau >0$ and ${\bm x} \in \mathcal{X}$, the following holds:
\begin{align*}
Pr \{ | \tilde{p}_{2 \eta} ({\bm x} ) - \tilde{\mu} ({\bm x} ) | \geq \tau \} \leq \frac{ \mathbb{V} [ \tilde{p}_{2 \eta} ( {\bm x})]  }{\tau^2}
,
\end{align*}
where $\tilde{\mu} ( {\bm x}) = \mathbb{E} [\tilde{p}_{2 \eta} ({\bm x} ) ]$.
Note that the expectation and variance are taken with respect to the prior distribution.
Thus, by replacing $\tau$ with $ (\delta /(2 |\mathcal{X}| ) ) ^{-1/2} ( \mathbb{V} [\tilde{p}_{2 \eta} ( {\bm x})] )^{1/2}$,
with probability at least $1-\delta/2$ the following holds for any ${\bm x} \in \mathcal{X}$:
$$
 | \tilde{p}_{2 \eta} ({\bm x} ) - \tilde{\mu} ({\bm x} ) | < \frac{\sqrt{  \mathbb{V} [\tilde{p}_{2 \eta} ( {\bm x})]      }}{ \sqrt{\delta/(2 |\mathcal{X}|)}} .
$$
This implies that
\begin{align}
\tilde{p}_{2 \eta} ({\bm x} ) <   \tilde{\mu} ({\bm x} ) +\frac{\sqrt{  \mathbb{V} [\tilde{p}_{2 \eta} ( {\bm x})]      }}{ \sqrt{\delta/(2 |\mathcal{X}|)}}.
 \label{eq:prior_bound}
\end{align}
Furthermore, $\tilde{\mu} ({\bm x} )$ can be expressed as
\begin{align*}
\tilde{\mu} ({\bm x} ) = \int _{\Omega} \left \{ \Phi \left (\frac{h+2\eta}{\sigma_0 ( {\bm x},{\bm w})} \right ) - \Phi \left (\frac{h}{\sigma_0 ( {\bm x},{\bm w})} \right ) \right \} p({\bm w}) \text{d} {\bm w}.
\end{align*}
Here, from Taylor's expansion, for any $ a <b$, it holds that
$$
\Phi (b) = \Phi (a) + \phi (c) (b-a) \leq \Phi (a) + \frac{1}{\sqrt{2 \pi} } (b-a) \leq \Phi (a) + (b-a),
$$
where $ c \in (a,b)$.
Therefore, we have
\begin{align}
\tilde{\mu} ({\bm x} ) \leq  \int _{\Omega} \frac{  2 \eta }{\sigma_0 ( {\bm x}, {\bm w} ) } p({\bm w}) \text{d} {\bm w} \leq
 \int _{\Omega} \frac{  2 \eta }{\sigma_{0,min} } p({\bm w}) \text{d} {\bm w} = \frac{  2 \eta }{\sigma_{0,min} }.
\label{eq:tildemubound}
\end{align}
Moreover, $\mathbb{V}  [ \tilde{p}_{2 \eta} ({\bm x} ) ]$ can be bounded as
\begin{align}
&\mathbb{V}  [ \tilde{p}_{2 \eta} ({\bm x} ) ] \nonumber \\
&=  \int _{\Omega} \int _{\Omega}  {\rm Cov}
[  \1\left[  h + 2 \eta \geq  f(\bm{x}, \bm{w}) > h\right] ,  \1\left[  h + 2\eta \geq  f(\bm{x}, \bm{w}^\prime) > h\right] ] p({\bm w}) p({\bm w}^\prime) \text{d} {\bm w}  \text{d} {\bm w^\prime} \nonumber \\
&\leq
\int _{\Omega} \int _{\Omega} \frac{
\mathbb{V}[  \1\left[  h+ 2\eta \geq  f(\bm{x}, \bm{w}) > h\right] ]+ \mathbb{V}  [ \1\left[  h+2 \eta \geq  f(\bm{x}, \bm{w}^\prime) > h\right] ]}{2}  p({\bm w}) p({\bm w}^\prime) \text{d} {\bm w}  \text{d} {\bm w}^\prime \nonumber \\
&= \int _{\Omega} \mathbb{V}[  \1\left[  h+ 2\eta \geq  f(\bm{x}, \bm{w}) > h\right] ] p({\bm w})   \text{d} {\bm w} \nonumber \\
&= \int _{\Omega}  \left \{ \Phi \left (\frac{h+2 \eta}{\sigma_0 ( {\bm x},{\bm w})} \right ) - \Phi \left (\frac{h}{\sigma_0 ( {\bm x},{\bm w})} \right ) \right \}     \left \{1- \Phi \left (\frac{h+2\eta}{\sigma_0 ( {\bm x},{\bm w})} \right ) + \Phi \left (\frac{h}{\sigma_0 ( {\bm x},{\bm w})} \right ) \right \} p({\bm w})   \text{d} {\bm w}  \nonumber \\
&\leq
\int _{\Omega} \left \{ \Phi \left (\frac{h+2 \eta}{\sigma_0 ( {\bm x},{\bm w})} \right ) - \Phi \left (\frac{h}{\sigma_0 ( {\bm x},{\bm w})} \right ) \right \} p({\bm w}) \text{d} {\bm w} \nonumber \\
&= \tilde{\mu} ({\bm x} ) \leq  \frac{  2 \eta }{\sigma_{0,min} }. \label{eq:tildevarbound}
\end{align}
Hence, by substituting \eqref{eq:tildemubound} and \eqref{eq:tildevarbound} into \eqref{eq:prior_bound}, we get
$$
\tilde{p}_{2 \eta} ({\bm x} ) <   \frac{  2 \eta  }{\sigma_{0,min} }+\sqrt{ \frac{ 2|\mathcal{X} | 2 \eta   }{\delta   \sigma_{0,min}}  }.
$$
Thus, from the assumption, we have
$$
\tilde{p}_{2 \eta} ({\bm x} ) < \frac{\epsilon}{2} + \sqrt{   \frac{\epsilon^2}{4}} = \epsilon.
$$
\end{proof}

\begin{lemma}\label{lem:fin_union}
   Let  $\delta \in (0, 1)$, $m \geq 2$, $\eta \geq 0$ and
    $\beta_t = \frac{|\mathcal{X}|\pi^2 t^2}{3\delta}$. Then, with probability at least $1-\delta/2$, it holds that
    \begin{equation*}
        |p_{t-1; \eta}(\bm{x}) - \mu_{t-1;\eta}^{(p)}(\bm{x})| < \beta_t^{1/m} \gamma_{t-1;\eta}^{2/m} ({\bm x}),~\forall \bm{x} \in \mathcal{X},~\forall t \geq 1.
    \end{equation*}
\end{lemma}
\begin{proof}
From Chebyshev's inequality and LemmaA.2 in  \cite{iwazaki2019bayesian},
noting that $|p_{t-1; \eta}(\bm{x}) - \mu_{t-1;\eta}^{(p)}(\bm{x})| \leq 1$ the inequality holds for any $\tau > 0$, $t \geq 1$ and $\bm{x} \in \mathcal{X}$:
    \begin{align}
        Pr\left\{ |p_{t-1; \eta}(\bm{x}) - \mu_{t-1;\eta}^{(p)}(\bm{x})| \geq \tau \right\}
 &= Pr\left\{ |p_{t-1; \eta}(\bm{x}) - \mu_{t-1;\eta}^{(p)}(\bm{x})|^m \geq \tau^m \right\} \nonumber \\
&\leq \frac{\mathbb{E} [    |p_{t-1; \eta}(\bm{x}) - \mu_{t-1;\eta}^{(p)}(\bm{x})|^m   ]  }{\tau^m} \nonumber \\
&\leq \frac{\mathbb{E} [    |p_{t-1; \eta}(\bm{x}) - \mu_{t-1;\eta}^{(p)}(\bm{x})|^2   ]  }{\tau^m} \nonumber \\
& = \frac{\mathbb{V}   [p_{t-1; \eta}(\bm{x}) ]}{\tau^m} \leq \frac{\gamma_{t-1;\eta}^2 ({\bm x})}{\tau^m}. \label{eq:Cinequality}
    \end{align}
By replacing $\tau $ with $ (\delta/2)^{-1/m}\gamma_{t-1;\eta}^{2/m} ({\bm x})$, for any   $t \geq 1$ and $\bm{x} \in \mathcal{X}$ the
following inequality holds with probability at least $1-\delta/2$:
    \begin{equation*}
        |p_{t-1; \eta}(\bm{x}) - \mu_{t-1;\eta}^{(p)}(\bm{x})| < (\delta/2)^{-1/m} \gamma_{t-1;\eta}^{2/m} ({\bm x}).
    \end{equation*}
Therefore, by replacing $\delta$ with $6 \delta /(|\mathcal{X}| \pi^2 t^2) $, with probability at least $1- 3 \delta /(|\mathcal{X}| \pi^2 t^2)$ the following holds:
    \begin{equation*}
        |p_{t-1; \eta}(\bm{x}) - \mu_{t-1;\eta}^{(p)}(\bm{x})| <
(|\mathcal{X}| \pi^2 t^2 /(3 \delta) )^{1/m} \gamma_{t-1;\eta}^{2/m} ({\bm x})
=\beta^{1/m}_t \gamma_{t-1;\eta}^{2/m} ({\bm x}).
    \end{equation*}
Thus, noting that $\sum_{t=1}^\infty t^{-2} = \pi^2/6$, with probability at least $1-\delta/2$ the following union bound holds:
    \begin{equation*}
    |p_{t-1; \eta}(\bm{x}) - \mu_{t-1;\eta}^{(p)}(\bm{x})| < \beta_t^{1/m} \gamma_{t-1;\eta}^{2/m} ({\bm x}),~\forall \bm{x} \in \mathcal{X},~\forall t \geq 1.
    \end{equation*}
\end{proof}
Note that when $\eta =0$, by replacing $\delta$ with $2 \delta$ we have Lemma \ref{lem:cred_int}.

\begin{lemma}\label{lem:r_bound}
Let $\epsilon >0$. Assume that there exists $t \geq 1$ such that $|p_{t-1; \eta}(\bm{x}) - \mu_{t-1;\eta}^{(p)}(\bm{x})| < \beta_t^{1/m} \gamma_{t-1;\eta}^{2/m}(\bm{x})$ for any $\bm{x} \in \mathcal{X}$.
Also assume that $\tilde{p}_{2 \eta} ({\bm x} ) < \epsilon $ for any ${\bm x} \in \mathcal{X}$.
   Then, the $\epsilon$-regret $r_t (\epsilon)$ satisfies the following inequality:
    \begin{equation*}
        r_t (\epsilon) \leq \frac{4C_4 \beta_t^{1/m}\sigma_{t-1}^2(\bm{x}_t, \bm{w}_t)}{\eta^{2+1/m}},
    \end{equation*}
   where $C_4 = \frac{m}{(2\pi)^{1/2m}}$.
\end{lemma}
\begin{proof}
    From the definition of $r_t (\epsilon)$, noting that $p_{t-1;\eta} ({\bm x}) \leq p_{\rm upper} ({\bm x}) \leq p_{t-1;\eta} ({\bm x}) +\tilde{p}_{2 \eta} ({\bm x})$, the following holds:
    \begin{align}
        r_t (\epsilon) &= p_{\text{upper}}(\bm{x}^*) -\epsilon - p_{\text{upper}}(\bm{x}_t) \nonumber \\
&\leq p_{t-1;\eta} ({\bm x}^\ast) + \tilde{p}_{2 \eta} ({\bm x} ^\ast ) -\epsilon -p_{t-1;\eta} ({\bm x}_t) \nonumber \\
            &\leq \mu_{t-1;\eta}^{(p)}(\bm{x}^\ast) + \beta_t^{1/m}\gamma_{t-1;\eta}^{2/m}(\bm{x}^\ast) - \left(\mu^{{(p)}}_{t-1;\eta}(\bm{x}_t) - \beta_t^{1/m}\gamma_{t-1;\eta}^{2/m}(\bm{x}_t)\right) \nonumber \\
            &\leq \mu_{t-1;\eta}^{(p)}(\bm{x}_t) + \beta_t^{1/m}\gamma_{t-1;\eta}^{2/m}(\bm{x}_t) - \left(\mu_{t-1;\eta}^{(p)}(\bm{x}_t) - \beta_t^{1/m}\gamma_{t-1;\eta}^{2/m}(\bm{x}_t)\right) \nonumber \\
            \label{eq:r_upper}
            &= 2\beta_{t}^{1/m}\gamma_{t-1;\eta}^{2/m}(\bm{x}_t).
    \end{align}
    Moreover, from the definitions of $\gamma_{t-1;\eta}$ and  $\bm{w}_t$, $ \gamma_{t-1;\eta}^2(\bm{x}_t)$ can be bounded as follows:
    \begin{align*}
&        \gamma_{t-1;\eta}^2(\bm{x}_t) \\
&= \int_{\Omega} \Phi\left(\frac{\mu_{t-1}(\bm{x}_t, \bm{w}) - h_{t-1, \bm{x}_t, \bm{w}; \eta}}{\sigma_{t-1}(\bm{x}_t, \bm{w})}\right)
        \left(1 - \Phi\left(\frac{\mu_{t-1}(\bm{x}_t, \bm{w}) - h_{t-1, \bm{x}_t, \bm{w};\eta}}{\sigma_{t-1}(\bm{x}_t, \bm{w})}\right)\right) p(\bm{w}) \text{d}\bm{w}\\
        &\leq \Phi\left(\frac{\mu_{t-1}(\bm{x}_t, \bm{w}_t) - h_{t-1, \bm{x}_t, \bm{w}_t; \eta}}{\sigma_{t-1}(\bm{x}_t, \bm{w}_t)}\right)
        \left(1 - \Phi\left(\frac{\mu_{t-1}(\bm{x}_t, \bm{w}_t) - h_{t-1, \bm{x}_t, \bm{w}_t;\eta}}{\sigma_{t-1}(\bm{x}_t, \bm{w}_t)}\right)\right) \\
        &\leq \Phi\left(\frac{-\eta}{\sigma_{t-1}(\bm{x}_t, \bm{w}_t)}\right).
    \end{align*}
Furthermore, for the cumulative distribution function of standard Normal distribution, the following holds:
    \begin{align*}
        \Phi\left(\frac{-\eta}{\sigma_{t-1}(\bm{x}_t, \bm{w}_t)}\right) &= \int_{-\infty}^{-\frac{\eta}{\sigma_{t-1}(\bm{x}_t, \bm{w}_t)}} \phi(a) \text{d}a \\
                        &= \int_{\frac{\eta}{\sigma_{t-1}(\bm{x}_t, \bm{w}_t)}}^{\infty}\phi(a) \text{d}a \\
                        &\leq \int_{\frac{\eta}{\sigma_{t-1}(\bm{x}_t, \bm{w}_t)}}^{\infty} \frac{a\phi(a)}{\eta} \text{d}a \\
                        &= \frac{1}{\eta\sqrt{2\pi}}\exp{\left (-\frac{\eta^2}{2\sigma_{t-1}^2(\bm{x}_t, \bm{w}_t)}\right )},
    \end{align*}
    where $\phi(\cdot)$ is the probability density function of standard Normal distribution.
In addition, noting that  $k( ({\bm x}, {\bm w}   ),  ({\bm x}, {\bm w}   )  ) \leq 1$ for any $  ({\bm x}, {\bm w}   )$,
the inequality for the third row can be obtained by using $\sigma_{t-1}(\bm{x}_t,{\bm w}_t) \leq 1$, i.e., $\eta/\sigma_{t-1}(\bm{x}_t,{\bm w}_t)  \geq \eta$.
Moreover, since $e^{-x} \leq 1/x$ for any positive number $x$, the following inequality holds:
    \begin{align}
        \gamma_{t-1;\eta}^{2/m}(\bm{x}_t) &\leq (\eta\sqrt{2\pi})^{-1/m}\exp{\left (-\frac{\eta^2}{2m\sigma_{t-1}^2(\bm{x}_t, \bm{w}_t)} \right )} \nonumber \\
        \label{eq:gamma_upper}
        &\leq (\eta\sqrt{2\pi})^{-1/m} \frac{2m\sigma_{t-1}^2(\bm{x}_t, \bm{w}_t)}{\eta^2}.
    \end{align}
Therefore, by using (\ref{eq:r_upper}) and  (\ref{eq:gamma_upper}), we get the desired inequality.
\end{proof}

\begin{lemma}\label{lem:var_bound}
Fix $T \geq 1$. Then, the following inequality holds:
    \begin{equation}
        \label{eq:var_bound}
        \sum_{t=1}^T \sigma_{t-1}^2(\bm{x}_t, \bm{w}_t) \leq \frac{2}{ \log (1 + \sigma^{-2})  } \kappa_T.
    \end{equation}
\end{lemma}
\begin{proof}
From Lemma 5.3 in \cite{srinivas2009gaussian}, $I(\bm{y}_T; f)$ can be expressed as
    \begin{equation}
        I(\bm{y}_T; f) = \frac{1}{2} \sum_{t=1}^T \log(1 + \sigma^{-2}\sigma_{t-1}(\bm{x}_t, \bm{w}_t)). \label{eq:IG}
    \end{equation}
Similarly, from Lemma 5.4 in \cite{srinivas2009gaussian}, it holds that
    \begin{equation}
        \sigma_{t-1}^2(\bm{x}_t, \bm{w}_t) \leq \frac{\log(1 + \sigma^{-2}\sigma_{t-1}(\bm{x}_t, \bm{w}_t))}{ \log (1 + \sigma^{-2})}. \label{eq:PVbound}
    \end{equation}
    Therefore, using \eqref{eq:IG} and \eqref{eq:PVbound} we have
    \begin{align*}
        \sum_{t=1}^T\sigma_{t-1}^2(\bm{x}_t, \bm{w}_t)
        &\leq \frac{2}{ \log ( 1 + \sigma^{-2})} I(\bm{y}_T; f) \\
        &\leq \frac{2}{ \log (1 + \sigma^{-2})} \kappa_T.
    \end{align*}
\end{proof}

Finally, we prove Theorem \ref{thm:fin_regret}.
\begin{proof}
Let $\delta \in (0, 1),~m \geq 2$, and let $\beta_t$ be  defined as in Lemma \ref{lem:fin_union}.
Moreover, let $\epsilon >0$ and let $2 \eta$ be defined as in Lemma \ref{lem:prior_pro}.
Then, from Lemma \ref{lem:prior_pro}, \ref{lem:fin_union} and \ref{lem:r_bound}, with probability at least $1-\delta$, the following inequality holds for any $t \geq 1$:
    \begin{align*}
        r_t (\epsilon )&\leq \frac{4C_4 \beta_t^{1/m} \sigma_{t-1}^2(\bm{x}_t, \bm{w}_t)}{\eta^{2+1/m}} \\
            &\leq \frac{4C_4 \beta_T^{1/m} \sigma_{t-1}^2(\bm{x}_t, \bm{w})}{\eta^{2+1/m}},  
    \end{align*}
where the second inequality is obtained by using monotonicity of $\beta_t$.
   Therefore, from Lemma \ref{lem:var_bound}, the following holds with probability at least $1-\delta$:
    \begin{align*}
      R_T (\epsilon) =  \sum_{t=1}^T r_t (\epsilon) &\leq \frac{4C_4 \beta_T^{1/m} \sum_{t=1}^T \sigma_{t-1}^2(\bm{x}_t, \bm{w}_t)}{\eta^{2+1/m}} \\
            &\leq \frac{8C_4 \beta_T^{1/m} \kappa_T}{\log (1 + \sigma^{-2})\eta^{2+1/m}},~\forall T \geq 1.
    \end{align*}
   Therefore, noting that $\frac{8C_4}{ \log (1 + \sigma^{-2})} = \frac{8m}{(2\pi)^{\frac{1}{2m}}\log(1+\sigma^{-2})}=C_1$ we get Theorem \ref{thm:fin_regret}.
\end{proof}

\subsection{Regret Bound of BPT-TS}
In this subsection, we prove Theorem \ref{thm:fin_bayes_regret}.
The basic ideas for the proof are based on \cite{russo2014learning}.
First, we show the  following lemmas.
\begin{lemma}\label{lem:split_regret}
   Let  $T \geq 1$, $\epsilon >0$ and $\eta >0$. Then, for any sequence of upper confidence bounds $\left\{U_t: \mathcal{X} \rightarrow \mathbb{R} \right\}_{t \geq 1}$, the following holds:
    \begin{equation*}
       BR_T (\epsilon) \leq \sum_{t=1}^T\mathbb{E}\left[ U_t(\bm{x}_t) - p_{t-1;\eta}(\bm{x}_t) \right]
                + \sum_{t=1}^T \mathbb{E}\left[ p_{t-1; \eta}(\bm{x}^\ast) - U_t(\bm{x}^\ast) \right]
+ \sum_{t=1}^T \mathbb{E} [   \tilde{p}_{2 \eta} ({\bm x}^\ast) - \epsilon].
    \end{equation*}
\end{lemma}
\begin{proof}
Noting that $p_{t-1;\eta} ({\bm x} ) \leq p_{\rm upper} ({\bm x}) \leq p_{t-1;\eta} ({\bm x} ) + \tilde{p}_{2 \eta} ({\bm x})$,  the following inequality holds:
$$
p_{\rm upper} ({\bm x}^\ast ) -\epsilon - p_{\rm upper} ({\bm x}_t ) \leq  p_{t-1;\eta} ({\bm x}^\ast ) + \tilde{p}_{2 \eta} ({\bm x}^\ast) -\epsilon - p_{t-1;\eta} ({\bm x}_t ).
$$
Therefore, we get
\begin{align}
&\mathbb{E} [p_{\rm upper} ({\bm x}^\ast ) -\epsilon - p_{\rm upper} ({\bm x}_t ) ] \nonumber \\
&\leq \mathbb{E}[p_{t-1;\eta} ({\bm x}^\ast) - p_{t-1;\eta} ({\bm x}_t )] + \mathbb{E}[ \tilde{p}_{2 \eta} ({\bm x}^\ast ) -\epsilon] \nonumber \\
&=
\mathbb{E}[p_{t-1;\eta} ({\bm x}^\ast) - U_t ({\bm x}^\ast ) ] + \mathbb{E}[U_t ({\bm x}^\ast )-U_t ({\bm x}_t )]  +\mathbb{E}[U_t ({\bm x}_t)- p_{t-1;\eta} ({\bm x}_t )] + \mathbb{E}[ \tilde{p}_{2 \eta} ({\bm x}^\ast ) -\epsilon]. \label{eq:Ber1}
\end{align}
Here, let $H_t =\{ {\bm x}_1, {\bm w}_1, y_1,\ldots, {\bm x}_{t-1}, {\bm w}_{t-1}, y_{t-1} \}$. Then, conditioned on $H_t$,
${\bm x}^\ast$ and ${\bm x}_t $ have the same distribution, and $U_t $ is a deterministic function.  Therefore, it holds that
$\mathbb{E} [ U_t ({\bm x}^\ast) \mid H_t] = \mathbb{E} [U_t ({\bm x}_t ) \mid H_t]$.
This implies that
\begin{align}
\mathbb{E}[U_t ({\bm x}^\ast )-U_t ({\bm x}_t )]  &=   \mathbb{E} [ \mathbb{E}[U_t ({\bm x}^\ast) -U_t ({\bm x}_t )  \mid H_t]] \nonumber \\
&=\mathbb{E}[  \mathbb{E}[ U_t ({\bm x}^\ast ) \mid H_t] -  \mathbb{E}[ U_t ({\bm x}_t ) \mid H_t]  ] =0. \label{eq:Ber2}
\end{align}
Thus, from \eqref{eq:Ber1} and \eqref{eq:Ber2}, we have
$$
\mathbb{E} [p_{\rm upper} ({\bm x}^\ast ) -\epsilon - p_{\rm upper} ({\bm x}_t ) ] \leq
\mathbb{E}[p_{t-1;\eta} ({\bm x}^\ast) - U_t ({\bm x}^\ast ) ]   +\mathbb{E}[U_t ({\bm x}_t)- p_{t-1;\eta} ({\bm x}_t )] + \mathbb{E}[ \tilde{p}_{2 \eta} ({\bm x}^\ast ) -\epsilon].
$$
Summing over $t$,  we get the desired inequality.
\end{proof}

\begin{lemma}\label{lem:expectation0}
Let $\epsilon >0$ and $2\eta =\min \{ \frac{\epsilon \sigma_{0,min}}{4}, \frac{\epsilon^3 \sigma_{0,min}}{32|\mathcal{X}|}\}$.
 Then, the following inequality holds for any $t \geq 1$:
$$
\mathbb{E}[ \tilde{p}_{2 \eta} ( {\bm x}^\ast ) - \epsilon ] \leq 0.
$$
\end{lemma}
\begin{proof}
Let $2 \eta $ be defined as in Lemma \ref{lem:expectation0}.
Note that this value is given by replacing $\epsilon $ and $\delta$ in Lemma \ref{lem:prior_pro} with $\epsilon/2$ and $\epsilon$, respectively.
Then, from Lemma \ref{lem:prior_pro}, with probability at least $1- \epsilon/2$, the following holds for any ${\bm x} \in \mathcal{X}$:
$$
\tilde{p}_{2 \eta} ({\bm x}) < \epsilon /2.
$$
Here,  we define an event $H$ as
$$
H= \{ f \mid \tilde{p}_{2 \eta} ({\bm x}) < \epsilon /2,\ \forall {\bm x} \in \mathcal{X} \}.
$$
Since the result of Lemma \ref{lem:prior_pro} is derived for the prior distribution,
the expected value $\mathbb{E}[\1 [ f \in H] ] $ satisfies $\mathbb{E} [\1 [ f \in H] ] \geq 1-\epsilon/2$,
where the expectation is taken with respect to the prior distribution $f$.
Moreover, for any $t \geq 1$, let ${\bm\theta}_t$ be a random  vector which contains the observation noise, optimal value ${\bm x}^\ast$ and any randomness of the algorithm.
Note that ${\bm\theta}_t$ does not contain $f$. Therefore, for any $t \geq 1$, noting that $0 \leq \tilde{p}_{2 \eta} ({\bm x}^\ast) \leq 1$ and the definition of $H$, the following holds:
\begin{align*}
\mathbb{E}[ \tilde{p}_{2 \eta} ( {\bm x}^\ast ) ] =
\mathbb{E}_{f, {\bm\theta}_t } [ \tilde{p}_{2 \eta} ( {\bm x}^\ast ) ] &= \mathbb{E}_{f, {\bm\theta}_t } [ \1[f \in H] \tilde{p}_{2 \eta} ( {\bm x}^\ast ) ]
+
\mathbb{E}_{f, {\bm\theta}_t } [ \1[f \notin H] \tilde{p}_{2 \eta} ( {\bm x}^\ast ) ] \\
&\leq  \mathbb{E}_{f, {\bm\theta}_t } [ \1[f \in H]  (\epsilon /2) ] +
\mathbb{E}_{f, {\bm\theta}_t } [ \1[f \notin H]  ] \\
&\leq  \mathbb{E}_{f, {\bm\theta}_t } [\epsilon /2 ] + \mathbb{E}_{f, {\bm\theta}_t } [ \1[f \notin H]  ] \\
&\leq \epsilon /2 +  \mathbb{E}_{f } [ \1[f \notin H]  ] \leq  \epsilon/2 + \epsilon / 2 = \epsilon.
\end{align*}
Hence,  we obtain $\mathbb{E}  [ \tilde{p}_{2 \eta} ( {\bm x}^\ast ) ] -\epsilon \leq 0$.
\end{proof}

\begin{lemma}\label{lem:second_term_bound}
    Let $m \geq 2$, $\eta>0$, $\beta_t = t^2|\mathcal{X}|$ and
$U_t(\bm{x}) = \mu_{t-1;\eta}^{(p)}(\bm{x}) + \beta_t^{1/m}\gamma^{2/m}_{t-1;\eta}(\bm{x})$.
Then, the following holds:
    \begin{equation*}
        \sum_{t=1}^T\mathbb{E}\left[p_{t-1;\eta}(\bm{x}^\ast) - U_t(\bm{x}^\ast)\right] \leq \pi^2/6, ~\forall T \geq 1.
    \end{equation*}
\end{lemma}
\begin{proof}
    Noting that $0 \leq p_{t-1; \eta} ({\bm x}) \leq 1$  and $U_t(\bm{x})  \geq 0$,      for any $\bm{x} \in \mathcal{X}$ the following inequality holds:
    \begin{equation*}
        \mathbb{E}\left[\1\left[p_{t-1; \eta}(\bm{x}) - U_t(\bm{x}) > 0 \right] \left\{ p_{t-1;\eta}(\bm{x}) - U_t(\bm{x})\right\} \mid H_t\right]
        \leq \mathbb{E}\left[\1\left[p_{t-1; \eta}(\bm{x}) - U_t(\bm{x}) > 0\right] \mid H_t\right ].
    \end{equation*}
    Moreover, conditioned on $H_t$, the conditional expected value of $p_{t-1;\eta} ({\bm x} )$ is equal to $\mu^{(p)}_{t-1;\eta} ({\bm x})$.
Furthermore, the conditional variance can be bounded by $\gamma^2_{t-1;\eta} ({\bm x})$.
Therefore, by using the same technique as in \eqref{eq:Cinequality}, the following holds for any $t \geq 1$ and ${\bm x} \in \mathcal{X}$:
$$
\left\{ |p_{t-1; \eta}(\bm{x}) - \mu_{t-1;\eta}^{(p)}(\bm{x})| > \beta_t^{1/m}\gamma_{t-1;\eta}^{2/m}(\bm{x}) \mid H_t\right\} <  \beta_t^{-1} .
$$
In addition, it holds that
    \begin{align*}
        &\hspace{15pt}Pr\left\{ |p_{t-1; \eta}(\bm{x}) - \mu_{t-1,\eta}^{(p)}(\bm{x})| > \beta_t^{1/m}\gamma_{t-1,\eta}^{2/m}(\bm{x}) \mid H_t\right\} <  \beta_t^{-1} \\
        &\Leftrightarrow Pr\left\{ p_{t-1; \eta}(\bm{x}) - \mu_{t-1;\eta}^{(p)}(\bm{x}) - \beta_t^{1/m}\gamma_{t-1;\eta}^{2/m}(\bm{x}) > 0 \mid H_t\right\} \\
& \quad \quad  +
         Pr\left\{ \mu_{t-1}^{(p)}(\bm{x}) - \beta_t^{1/m}\gamma_{t-1;\eta}^{2/m}(\bm{x}) - p_{t-1; \eta}(\bm{x}) > 0 \mid H_t\right\} < \beta_t^{-1} \\
        &\Rightarrow Pr\left\{ p_{t-1; \eta}(\bm{x}) - \mu_{t-1;\eta}^{(p)}(\bm{x}) - \beta_t^{1/m}\gamma_{t-1;\eta}^{2/m}(\bm{x}) > 0 \mid H_t\right\} < \beta_t^{-1}.
    \end{align*}
    Thus, we have
    \begin{align}
        \mathbb{E}\left[\1\left[p_{t-1; \eta}(\bm{x}) - U_t(\bm{x})>0 \right] \mid H_t \right]
        &= Pr\left\{ p_{t-1; \eta}(\bm{x}) - \mu_{t-1;\eta}^{(p)}(\bm{x}) - \beta_t^{1/m}\gamma_{t-1;\eta}^{2/m}(\bm{x}) > 0 \mid H_t\right\} \nonumber \\
        &< \beta^{-1} _{  t    }
        =\frac{1}{t^2|\mathcal{X}|}. \label{eq:betatinv}
    \end{align}
Here, the following inequality holds for any $t \geq 1$ and ${\bm x} \in \mathcal{X}$:
\begin{align*}
 p_{t-1;\eta} ({\bm x} ) -U_t ({\bm x}) &\leq \1 [ p_{t-1;\eta} ({\bm x} ) -U_t ({\bm x}) >0] \{ p_{t-1;\eta} ({\bm x} ) -U_t ({\bm x})  \}
\\
&\leq \sum_{ {\bm x} \in \mathcal{X}} \1 [ p_{t-1;\eta} ({\bm x} ) -U_t ({\bm x}) >0] \{ p_{t-1;\eta} ({\bm x} ) -U_t ({\bm x})  \}.
\end{align*}
Note  that  $p_{t-1;\eta} ({\bm x}) - U_t ({\bm x})  \leq 1$      because  $0 \leq p_{t-1;\eta} ({\bm x}) \leq 1$ and $U_t ({\bm x}) \geq 0$.
Thus, we get
$$
 p_{t-1;\eta} ({\bm x}^\ast  ) -U_t ({\bm x}^\ast) \leq \sum_{ {\bm x} \in \mathcal{X}} \1 [ p_{t-1;\eta} ({\bm x} ) -U_t ({\bm x}) >0].
$$
By using this inequality and \eqref{eq:betatinv}, the following holds for any $T \geq 1$:
    \begin{align}
    \sum_{t=1}^T \mathbb{E}\left[p_{t-1;\eta}(\bm{x}^\ast) - U_t(\bm{x}^\ast)\right] &\leq
    \sum_{t=1}^{T}\sum_{\bm{x} \in \mathcal{X}} \mathbb{E}\left[\1\left[ p_{t-1;\eta}(\bm{x}) - U_t(\bm{x}) > 0 \right] \right] \nonumber \\
&=  \sum_{t=1}^{T}\sum_{\bm{x} \in \mathcal{X}} \mathbb{E} \left [ \mathbb{E}\left[\1\left[ p_{t-1;\eta}(\bm{x}) - U_t(\bm{x}) > 0 \right] \mid H_t \right]  \right ] \nonumber \\
&\leq  \sum_{t=1}^{T}\sum_{\bm{x} \in \mathcal{X}} \frac{1}{t^2 |\mathcal{X}| } \nonumber \\
    &= \sum_{t=1}^{T}\frac{1}{t^2} \leq \sum_{t=1}^\infty \frac{1}{t^2} = \frac{\pi^2}{6}. \nonumber
    \end{align}
\end{proof}

\begin{lemma}\label{lem:first_term_bound}
   Let  $m \geq 2$, $\eta >0$, and let $U_t$ and $ \beta_t$ be defined as in Lemma \ref{lem:second_term_bound}.
Then, the following holds:
    \begin{equation*}
    \sum_{t=1}^T \mathbb{E}\left[ U_t(\bm{x}_t) - p_{t-1; \eta}(\bm{x}_t) \right] \leq C_2 T^{2/m}\eta^{-(2+1/m)} \kappa_T,
    \end{equation*}
    where $C_2 = \frac{4m|\mathcal{X}|^{1/m}}{(2\pi)^{1/2m}(\log (1+\sigma^{-2}))}$.
\end{lemma}
\begin{proof}
    From the definition of $U_t$, the following holds:
    \begin{align}
        \sum_{t=1}^T\mathbb{E}\left[ U_t(\bm{x}_t) - p_{t-1; \eta} ({\bm x}_t ) \right] &=
\sum_{t=1}^T \mathbb{E} \left [ \mu^{(p)}_{t-1;\eta} ({\bm x}_t ) + \beta_t^{1/m}\gamma^{2/m}_{t-1;\eta} (\bm{x}_t) - p_{t-1;\eta} ({\bm x}_t) \right ] \nonumber  \\
&=\sum_{t=1}^T \mathbb{E} \left [ \mu^{(p)}_{t-1;\eta} ({\bm x}_t ) - p_{t-1;\eta} ({\bm x}_t) \right ]
+\sum_{t=1}^T \mathbb{E} \left [  \beta_t^{1/m}\gamma^{2/m} _{t-1;\eta} (\bm{x}_t)  \right ]. \label{eq:TS12}
    \end{align}
Here, by using the tower property of conditional expectation,
 we get
\begin{equation}
\sum_{t=1}^T \mathbb{E} \left [ \mu^{(p)}_{t-1;\eta} ({\bm x}_t ) - p_{t-1;\eta} ({\bm x}_t) \right ]  = 0. \label{eq:lemmaA8first}
\end{equation}
Next, from  (\ref{eq:gamma_upper}) we have
    \begin{align*}
       \sum_{t=1}^T \mathbb{E} \left [  \beta_t^{1/m}\gamma^{2/m} _{t-1;\eta} (\bm{x}_t)  \right ] \leq \sum_{t=1}^T\mathbb{E}\left[\beta_t^{1/m}(\eta\sqrt{2\pi})^{-1/m} \frac{2m\sigma_{t-1}^2(\bm{x}_t, \bm{w}_t)}{\eta^2} \right].
    \end{align*}
    In addition, from monotonicity of $\beta_t$, the following inequality holds:
    \begin{align*}
   \sum_{t=1}^T \mathbb{E} \left [  \beta_t^{1/m}\gamma^{2/m} _{t-1;\eta} (\bm{x}_t)  \right ] \leq 2m(2\pi)^{-1/2m} \beta_T^{1/m}(\eta)^{-(2+1/m)}\mathbb{E}\left[\sum_{t=1}^T \sigma_{t-1}^2(\bm{x}_t, \bm{w}_t)\right].
    \end{align*}
    Hence, from Lemma \ref{lem:var_bound} we obtain
    \begin{align}
         \sum_{t=1}^T \mathbb{E} \left [  \beta_t^{1/m}\gamma^{2/m} _{t-1;\eta} (\bm{x}_t)  \right ]
        &\leq 4m\frac{(2\pi)^{-1/2m}}{\log (1+\sigma^{-2})} \beta_T^{1/m}(\eta)^{-(2+1/m)} \kappa_T \nonumber \\
        &= 4m\frac{(2\pi)^{-1/2m}}{\log (1+\sigma^{-2})} |\mathcal{X}|^{1/m}T^{2/m}(\eta)^{-(2+1/m)} \kappa_T \nonumber \\
        &\leq C_2 T^{2/m}(\eta)^{-(2+1/m)} \kappa_T, \label{eq:lemmaA8second}
    \end{align}
   where $C_2 = \frac{4m|\mathcal{X}|^{1/m}}{(2\pi)^{1/2m}(\log (1+\sigma^{-2}))}$.
Therefore, by substituting \eqref{eq:lemmaA8first} and \eqref{eq:lemmaA8second} into \eqref{eq:TS12}, we get the desired
inequality.
\end{proof}
Finally, from Lemma \ref{lem:split_regret}, \ref{lem:expectation0}, \ref{lem:second_term_bound} and \ref{lem:first_term_bound},
we have Theorem \ref{thm:fin_bayes_regret}.

\subsection{Convergence of BPT-LSE}
In this subsection, we prove Theorem \ref{thm:lse_convergence}.
\begin{lemma}\label{lem:str_gam}
    For any $t\geq 1$ and  $m > 0$, it holds that
    \begin{align*}
        \text{STR}_t(\bm{x}_t) \leq \beta_t^{1/m}\gamma_{t-1;\eta}^{2/m} ({\bm x}_t).
    \end{align*}
\end{lemma}
\begin{proof}
    From the definition of $\text{STR}_t$, the following holds:
    \begin{align*}
        \text{STR}_t(\bm{x}_t)
        &= \min\{u_{t;\eta}(\bm{x}_t) - \alpha,~\alpha - l_{t;\eta}(\bm{x}_t)\} \\
        &\leq \frac{u_{t;\eta}(\bm{x}_t) - l_{t;\eta}(\bm{x}_t)}{2} \\
        &= \beta_t^{1/m}\gamma_{t-1;\eta}^{2/m} ({\bm x}_t).
    \end{align*}
\end{proof}

\begin{lemma}\label{lem:str_bound}
    Let $m > 0$, and let $\beta_t$ be defined as in Lemma \ref{lem:fin_union}.
Then, for any $t \geq 1$, there exists a natural number $t^\prime$ such that  $t^\prime \leq t$ and
    \begin{equation}
        \text{STR}_{t^\prime}(\bm{x}_{t^\prime}) \leq \frac{C_3\eta^{-(2+1/m)}\beta_t^{1/m}\kappa_t}{t}, \label{eq:STRbyt}
    \end{equation}
    where $C_3 = \frac{4m}{(2\pi)^{1/2m}\log(1 + \sigma^{-2})}$.
\end{lemma}
\begin{proof}
  Fix $t \geq 1$. From   (\ref{eq:gamma_upper}) and Lemma \ref{lem:str_gam}, it holds that
    \begin{align*}
        \text{STR}_t(\bm{x}_t)
        &\leq \beta_t^{1/m} \gamma_{t-1;\eta}^{2/m} ({\bm x}_t) \\
        &\leq \beta_t^{1/m} 2m(2\pi)^{-1/2m}\eta^{-(2+1/m)} \sigma_{t-1}^2(\bm{x}_t, \bm{w}_t).
    \end{align*}
    Thus, noting that monotonicity of $\beta_t$  we get
    \begin{align*}
        \sum_{i=1}^t \text{STR}_i(\bm{x}_i)
        &\leq \sum_{i=1}^t \beta_i^{1/m} 2m(2\pi)^{-1/2m}\eta^{-(2+1/m)} \sigma_{i-1}^2(\bm{x}_i, \bm{w}_i) \\
        &\leq \beta_t^{1/m} 2m(2\pi)^{-1/2m}\eta^{-(2+1/m)} \sum_{i=1}^{t} \sigma_{i-1}^2(\bm{x}_i, \bm{w}_i).
    \end{align*}
    Hence, from Lemma\ref{lem:var_bound} the following inequality holds:
    \begin{equation}
        \sum_{i=1}^t \text{STR}_i(\bm{x}_i)
        \leq \frac{4m(2\pi)^{-1/2m}}{\log(1 + \sigma^{-2})}\eta^{-(2+1/m)}\beta_t^{1/m} \kappa_t
        = C_3\eta^{-(2+1/m)}\beta_t^{1/m}\kappa_t, \nonumber
    \end{equation}
    where $C_3 = \frac{4m}{(2\pi)^{1/2m}\log(1 + \sigma^{-2})}$.
Here, let  $t^\prime$ be a natural number satisfying $t^\prime \leq t$ and $\text{STR}_{t^\prime}(\bm{x}_{t^\prime}) \leq \text{STR}_{i}(\bm{x}_i)~(\forall{i} \leq t)$, i.e.,
$t^\prime$ is given by
$$
t^\prime = \argmin _{ i  : \  i \leq t } \text{STR}_i ({\bm x}_i) .
$$
Then, it holds that
    \begin{align}
        t \text{STR}_{t^\prime}(\bm{x}_{t^\prime}) \leq \sum_{i=1}^t \text{STR}_i(\bm{x}_i) \leq C_3\eta^{-(2+1/m)}\beta_t^{1/m}\kappa_t. \label{eq:tSTR}
    \end{align}
    Therefore, by dividing \eqref{eq:tSTR} by $t$, we obtain \eqref{eq:STRbyt}.
\end{proof}

\begin{lemma}\label{lem:terminate_cond}
While running BPT-LSE, if $\text{STR}_t(\bm{x}_t) < \epsilon /2 $ for some $t \geq 1$, then  $\mathcal{U}_{t} = \emptyset$.
\end{lemma}
\begin{proof}
Assume that     $\mathcal{U}_t \neq \emptyset$. 
Then, there exists $\bm{x} \in \mathcal{X}$ such that
    $\text{STR}_t(\bm{x}) \geq \epsilon/2$.
Therefore, from the definition of $\bm{x}_t$, it holds that
$$\epsilon /2 \leq \text{STR}_t(\bm{x}) \leq \text{STR}_t(\bm{x}_t) .$$
 However, it contradicts the   assumption
   $\text{STR}_t(\bm{x}_t) < \epsilon /2$.
\end{proof}
Finally, we prove Theorem \ref{thm:lse_convergence} by using Lemma \ref{lem:prior_pro}, \ref{lem:fin_union}, \ref{lem:str_bound} and  \ref{lem:terminate_cond}.
Let $m \geq 2$. Then, from Lemma \ref{lem:fin_union}, with probability at least $1-\delta/2$, it holds that for any $t \geq 1$ and ${\bm x} \in \mathcal{X}$:
$$
p_{t-1;\eta} ( {\bm x} )   \in Q_{t;\eta} ({\bm x}) = [l_{t;\eta} ({\bm x}) , u_{t;\eta} ({\bm x} )     ].
$$
Moreover, by replacing $\epsilon$ in Lemma \ref{lem:prior_pro} with $\epsilon/2$,  with probability at least $1-\delta/2$, the following holds for any ${\bm x} \in \mathcal{X}$:
$$
\tilde{p}_{2 \eta} ( {\bm x} ) < \epsilon /2 .$$
Furthermore, noting that $p_{t-1;\eta} ({\bm x}) \leq p_{\rm upper} ({\bm x}) \leq p_{t-1;\eta} ({\bm x}) + \tilde{p}_{2 \eta} ({\bm x}) $,
with probability at least $1-\delta$, it holds that
 $$p_{\rm upper} ({\bm x}) \in [l_{t;\eta} ({\bm x}) , u_{t;\eta} ({\bm x} ) +\epsilon/2],\   \forall t \geq 1, \ \forall {\bm x} \in \mathcal{X}.  $$
Therefore, from the classification condition, if ${\bm x} \in  \mathcal{H}_t$, then $p_{\rm upper} ({\bm x} ) \geq \alpha -\epsilon/2.$
Similarly, if ${\bm x} \in  \mathcal{L}_t$, then $p_{\rm upper} ({\bm x} )  \leq \alpha +\epsilon$.
Hence,
noting that  the definition of $e_\alpha ({\bm x}) $,
with probability at least $1-\delta$ the following holds:
$$
\max _{ {\bm x} \in \mathcal{X}} e_ \alpha ({\bm x} ) \leq  \epsilon.
$$
Next,  from Lemma \ref{lem:str_bound} and  \eqref{eq:LSEbound}, there exists $t^\prime$ such that $t^\prime \leq T$ and
$\text{STR}_{t^\prime} ({\bm x}_{t^\prime} ) < \epsilon /2$.
Hence, from Lemma \ref{lem:terminate_cond},  BPT-LSE terminates after at most $T$ rounds.

\section{Extension to Query-based Setting}
\label{app:query-based}
In this section, we consider extensions to an infinite set for $\mathcal{X}$.
Basic ideas used in this section are based on \cite{srinivas2009gaussian}, i.e., we assume stochastic Lipschitz continuity for $p_{\rm upper} ({\bm x})$.
Let $\mathcal{X} \subset \mathbb{R}^d$ be an infinite set. For simplicity,  assume that $\mathcal{X} \subset [0,r]^d$, for some $r>0$.
In addition, for  finite subset $\tilde{\mathcal{X}} \subset \mathcal{X}$ and ${\bm x} \in \mathcal{X}$,
let $  [ {\bm x} ] $ be the closest point in $\tilde{\mathcal{X}}$ to ${\bm x}$.
Hereafter, we assume that $\tilde{\mathcal{X}}$ has $\zeta ^d$ elements.
Also assume that the following inequality holds for any ${\bm x} \in \mathcal{X}$:
\begin{align}
\| {\bm x} - [{\bm x} ] \|_1 \leq rd /\zeta. \label{eq:discretization}
\end{align}
Moreover, for $p_{\rm upper} ({\bm x})$, we assume the following condition:
\begin{description}
\item [(C1)] There exists positive constants $a$ and $b$ such that
$$
Pr \left \{ \sup_{ {\bm x} \in \mathcal{X} } | \partial p_{\rm upper} ({\bm x})/ \partial x_j | >L \right  \}  \leq a e^{-(L/b)^2}, \ j=1,\ldots,d.
$$
\end{description}

\subsection{Regret bound of  BPT-UCB when $\mathcal{X}$ is infinite set}
In this subsection, we give the theorem about the cumulative $\epsilon$-regret for BPT-UCB.
The following theorem holds:
\begin{theorem}\label{thm:BPT-UCB-infinite}
Let $\delta \in (0,1)$, $\epsilon >0$, $m \geq 2$, $\zeta = \lceil d (\epsilon /4)^{-1} b r \sqrt{  \log (4 d a /\delta) }  \rceil  $ and $\beta _t = 2 (1+ \zeta^d ) \pi^2 t^2 /(3 \delta ) $ and    $2 \eta = \min \{ \frac{3 \epsilon \sigma_{0,min}}{8}, \frac{9 \epsilon^2 \delta \sigma_{0,min}}{128 \zeta ^d }\}$.
Then, running BPT-UCB with these parameters, the cumulative $\epsilon$-regret satisfies the following inequality:
\begin{equation*}
        Pr\left\{ R_T (\epsilon) \leq C_1 \beta_T^{1/m}\kappa_T \eta^{ -(2+1/m) },~\forall T \geq 1 \right\} \geq 1 - \delta,
    \end{equation*}
    where $C_1 = \frac{8m}{(2\pi)^{\frac{1}{2m}}\log(1+\sigma^{-2})}$.
\end{theorem}
In order to prove Theorem \ref{thm:BPT-UCB-infinite},
first, we show the following lemma:
\begin{lemma}\label{lem:stochastic}
Let $\delta \in (0,1)$, $\epsilon >0$, $m \geq 2$,  $\zeta = \lceil d (\epsilon /4)^{-1} b r \sqrt{  \log (2 d a /\delta) }  \rceil  $ and $\beta _t =  (1+ \zeta^d ) \pi^2 t^2 /(3 \delta ) $.
Then, for any $ \eta \geq 0$, with probability at least $1-\delta$, the following inequality holds for any ${\bm x} \in \mathcal{X}$ and $ t \geq 1$:
\begin{align*}
p_{\rm upper} ({\bm x} ) &\leq \epsilon/4 + \tilde{p}_{2 \eta}  ([ {\bm x} ]) + \mu^{(p)}_{t-1;\eta} ([{\bm x}]) +
\beta^{1/m}_t  \gamma^{2/m}_{t-1;\eta} ([{\bm x}]), \\
 p_{\rm upper} ({\bm x} ) &\geq -\epsilon/4  + \mu^{(p)}_{t-1;\eta} ([{\bm x}]) -
\beta^{1/m}_t  \gamma^{2/m}_{t-1;\eta} ([{\bm x}]), \\
|p_{t-1;\eta} ({\bm x}_t ) - \mu^{(p)}_{t-1;\eta} ({\bm x}_t ) | & \leq \beta^{1/m}_t \gamma^{2/m}_{t-1;\eta} ({\bm x}_t ).
\end{align*}
\end{lemma}
\begin{proof}
From the condition {\sf (C1)}, we have
$$
Pr \{\forall j, \ \forall {\bm x} \in \mathcal{X}, \ | \partial p_{\rm upper} ({\bm x}) / \partial x_j | <L \}
\geq 1- da e^{-L^2/(b^2) } .
$$
This implies that with probability at least $1-da e^{-L^2/(b^2) }$ the following holds:
$$
\forall {\bm x}, {\bm x}^\prime \in \mathcal{X}, \ | p_{\rm upper} ({\bm x}) -p_{\rm upper} ({\bm x}^\prime) |
\leq  L \| {\bm x} - {\bm x}^\prime \|_1 .
$$
Here, by replacing ${\bm x}^\prime $ with $ [{\bm x} ] $, we get
$$
\forall {\bm x} \in \mathcal{X}, \ | p_{\rm upper} ({\bm x}) -p_{\rm upper}   ( [ {\bm x} ] ) |
\leq  L \| {\bm x} - [{\bm x}] \|_1 .
$$
Therefore, by choosing $L= b \sqrt{  \log (2 d a /\delta )}$, with probability at least $1-\delta/2$, the following holds:
$$
\forall {\bm x} \in \mathcal{X}, \ | p_{\rm upper} ({\bm x}) -p_{\rm upper}   ( [ {\bm x} ] ) |
\leq  b \sqrt{  \log (2 d a /\delta )} \| {\bm x} - [{\bm x}] \|_1 .
$$
Moreover, from \eqref{eq:discretization}, we have
$$
\forall {\bm x} \in \mathcal{X}, \ | p_{\rm upper} ({\bm x}) -p_{\rm upper}   ( [ {\bm x} ] ) |
\leq  b \sqrt{  \log (2 d a /\delta )} r d /\zeta .
$$
Hence, from the definition of $\zeta$, the following inequality holds:
\begin{equation}
\forall {\bm x} \in \mathcal{X}, \ | p_{\rm upper} ({\bm x}) -p_{\rm upper}   ( [ {\bm x} ] ) |
\leq  \epsilon /4. \label{eq:epsilon4}
\end{equation}
On the other hand, by using the same argument as the proof of Lemma \ref{lem:fin_union}, for any $t \geq 1$ and ${\bm x}^\prime \in \tilde{\mathcal{X}}$,
the following holds with probability at least $1-3 \delta /( (1+\zeta^d)\pi^2 t^2  ) $:
\begin{align*}
|p_{t-1;\eta} ({\bm x}^\prime ) - \mu^{(p)}_{t-1;\eta} ({\bm x}^\prime ) |  & \leq \beta^{1/m}_t \gamma^{2/m}_{t-1;\eta} ({\bm x}^\prime ).
\end{align*}
Similarly, with probability at least $1-3 \delta /( (1+\zeta^d)\pi^2 t^2  ) $, the following inequality holds:
$$
|p_{t-1;\eta} ({\bm x}_t ) - \mu^{(p)}_{t-1;\eta} ({\bm x}_t ) |  \leq \beta^{1/m}_t \gamma^{2/m}_{t-1;\eta} ({\bm x}_t ).
$$
Thus, since $|\tilde{\mathcal{X}}| = \zeta^d$, with probability at least $1- \delta/2$, the following holds for any $t \geq 1$ and ${\bm x}^\prime \in \tilde{\mathcal{X}}$:
\begin{align}
|p_{t-1;\eta} ({\bm x}^\prime ) - \mu^{(p)}_{t-1;\eta} ({\bm x}^\prime ) |  & \leq \beta^{1/m}_t \gamma^{2/m}_{t-1;\eta} ({\bm x}^\prime ), \label{eq:66}\\
|p_{t-1;\eta} ({\bm x}_t ) - \mu^{(p)}_{t-1;\eta} ({\bm x}_t ) |  &\leq \beta^{1/m}_t \gamma^{2/m}_{t-1;\eta} ({\bm x}_t ). \label{eq:67}
\end{align}
Therefore, noting that $p_{t-1;\eta} ({\bm x} ) \leq p_{\rm upper} ({\bm x}) \leq p_{t-1;\eta} ({\bm x} ) +\tilde{p}_{2 \eta} ({\bm x} )$, using
\eqref{eq:epsilon4}, \eqref{eq:66}  and \eqref{eq:67}, with probability at least $1-\delta$, desired inequalities hold for any ${\bm x} \in \mathcal{X}$ and $t \geq1$.
\end{proof}
By using Lemma \ref{lem:stochastic}, we prove Theorem \ref{thm:BPT-UCB-infinite}.
\begin{proof}
Let $\delta \in (0,1)$, $\epsilon >0$, $m \geq 2$, $\zeta = \lceil d (\epsilon /4)^{-1} b r \sqrt{  \log (4 d a /\delta) }  \rceil  $,  $\beta _t = 2 (1+ \zeta^d ) \pi^2 t^2 /(3 \delta ) $ and $2 \eta = \min \{ \frac{3 \epsilon \sigma_{0,min}}{8}, \frac{9 \epsilon^2 \delta \sigma_{0,min}}{128 \zeta ^d }\}$.
Then, from Lemma \ref{lem:stochastic}, with probability at least $1-\delta/2$, the following holds for any ${\bm x} \in \mathcal{X}$ and $t \geq 1$:
\begin{align*}
p_{\rm upper} ({\bm x} ) &\leq \epsilon/4 + \tilde{p}_{2 \eta}  ([ {\bm x} ]) + \mu^{(p)}_{t-1;\eta} ([{\bm x}]) +
\beta^{1/m}_t  \gamma^{2/m}_{t-1;\eta} ([{\bm x}]), \\
p_{t-1;\eta} ({\bm x}_t ) & \geq  \mu^{(p)}_{t-1;\eta} ({\bm x}_t )  -\beta^{1/m}_t \gamma^{2/m}_{t-1;\eta} ({\bm x}_t ).
\end{align*}
Here, noting that $p_{t-1;\eta} ({\bm x}_t ) \leq p_{\rm upper} ({\bm x}_t )$, from the definition of $r_t (\epsilon)$ we get
\begin{align*}
&r_t (\epsilon ) \\
&=   (p_{\rm upper} ({\bm x}^\ast ) - \epsilon)- p_{\rm upper } ({\bm x}_t ) \\
&\leq  -3 \epsilon/4 + \tilde{p}_{2 \eta}  ([ {\bm x}^\ast ]) + \mu^{(p)}_{t-1;\eta} ([{\bm x}^\ast]) +
\beta^{1/m}_t  \gamma^{2/m}_{t-1;\eta} ([{\bm x}^\ast])-p_{t-1;\eta} ({\bm x}_t ) \\
&\leq -3 \epsilon/4 + \tilde{p}_{2 \eta}  ([ {\bm x}^\ast ]) + \mu^{(p)}_{t-1;\eta} ([{\bm x}^\ast]) +
\beta^{1/m}_t  \gamma^{2/m}_{t-1;\eta} ([{\bm x}^\ast])-(\mu^{(p)}_{t-1;\eta} ({\bm x}_t )  -\beta^{1/m}_t \gamma^{2/m}_{t-1;\eta} ({\bm x}_t )).
\end{align*}
Moreover, from the definition of ${\bm x}_t$, it holds that
$$
\mu^{(p)}_{t-1;\eta} ([{\bm x}^\ast]) +
\beta^{1/m}_t  \gamma^{2/m}_{t-1;\eta} ([{\bm x}^\ast]) \leq \mu^{(p)}_{t-1;\eta} ({\bm x}_t )  +\beta^{1/m}_t \gamma^{2/m}_{t-1;\eta} ({\bm x}_t ).
$$
Thus, we have
$$
r_t (\epsilon ) \leq -3 \epsilon/4 + \tilde{p}_{2 \eta}  ([ {\bm x}^\ast ]) +2 \beta^{1/m}_t \gamma^{2/m}_{t-1;\eta} ({\bm x}_t ).
$$
In addition, from Lemma  \ref{lem:prior_pro}, with probability at least $1-\delta /2 $, it holds that
$$
\tilde{p}_{2 \eta}  ({\bm x} ^\prime ) < 3 \epsilon /4, \ \forall {\bm x}^\prime \in \tilde{\mathcal{X}}
$$
because the value of $2 \eta$ is given by replacing $\epsilon$ and $|\mathcal{X}|$ in Lemma \ref{lem:prior_pro} with $3\epsilon /4$ and
$|  \tilde{\mathcal{X}} | = \zeta^d$, respectively.
Therefore, with probability at least $1- \delta$, the following holds for any $t \geq 1$:
\begin{align}
r_t (\epsilon ) \leq 2 \beta^{1/m}_t \gamma^{2/m}_{t-1;\eta} ({\bm x}_t ). \label{eq:epsilon_regret_infinite}
\end{align}
Hence, from \eqref{eq:epsilon_regret_infinite}, \eqref{eq:gamma_upper} and \eqref{eq:var_bound},
with probability at least $1-\delta$, the following holds for any $T \geq 1$:
\begin{align*}
R_T (\epsilon) = \sum_{t=1}^T r_t (\epsilon)
& \leq 2 \beta ^{1/m}_T \sum_{t=1}^T \gamma^{2/m}_{t-1;\eta} ({\bm x}_t ) \\
&\leq  2 \beta ^{1/m}_T    \frac{2m}{   (2 \pi)^{1/(2m)}    \eta^{2+1/m}}        \sum_{t=1}^T  \sigma^2_{t-1} ({\bm x}_t,{\bm w}_t) \\
&\leq  2 \beta ^{1/m}_T    \frac{2m}{   (2 \pi)^{1/(2m)}    \eta^{2+1/m}}   \frac{2}{\log (1+\sigma^{-2} ) } \kappa_T
= C_1 \beta^{1/m} _T \kappa_T \eta^{-(2+1/m)}.
\end{align*}
\end{proof}

\subsection{Regret bound of  BPT-TS when $\mathcal{X}$ is infinite set}
In this subsection, we give the theorem about the Bayes $\epsilon$-regret for BPT-TS.
First,  we change the selection strategy for ${\bm x}_t$.
The design parameter ${\bm x}_t$ is chosen as
\begin{align*}
{\bm x}_t = [ {\bm x}^\prime_t  ] \ \text{and} \
 {\bm x}^\prime_t = \argmax _{ {\bm x} \in \mathcal{X} } \int _{\Omega} \1 [\hat{f} ({\bm x}, {\bm w} ) > h  ] p({\bm w} ) \text{d} {\bm w}. %
\end{align*}
Then, the following theorem holds:
\begin{theorem}\label{thm:BPT-TS-infinite}
Let $m \geq 2$, $\epsilon >0$, $\zeta = \lceil d (\epsilon/4) ^{-1} b r  \sqrt{\log ( 4 d a  / \epsilon ) } \rceil $ and
$2 \eta = \min \{   \frac{ \epsilon \sigma_{0,min} }{8}, \frac{\epsilon^3 \sigma_{0,min}}{256 \zeta^d} \}$.
Then, running BPT-TS with these parameters, the Bayes $\epsilon$-regret satisfies the following inequality:
$$
BR_T (\epsilon) \leq \pi^2/6 +C_4 T^{2/m} \kappa_T \eta^{-(2+1/m)},
$$
where $C_4 = \frac{4 m \zeta^{d/m}}{   (2 \pi)^{1/(2m)} \log (1+\sigma^{-2})   }$.
\end{theorem}
In order to prove Theorem \ref{thm:BPT-TS-infinite}, we show the following lemmas:
\begin{lemma}\label{lem:4bunkatu}
Let $T \geq 1$, $\epsilon >0$ and $\eta >0$. Then, for any sequence of upper confidence bounds $\{U_t : \mathcal{X} \to \mathbb{R} \}_{ t \geq 1} $, the following holds:
\begin{align*}
BR_T (\epsilon ) &\leq \sum _{t=1}^T \mathbb{E} [ p_{\rm upper} ({\bm x}^\ast ) -  p_{\rm upper} ([{\bm x}^\ast] )-\epsilon/2]
+\sum _{t=1}^T \mathbb{E} [   p_{t-1;\eta} ( [{\bm x}^\ast]  ) - U_t ( [{\bm x}^\ast] )       ] \\
&\  + \sum _{t=1}^T \mathbb{E} [   \tilde{p}_{2 \eta} ( [{\bm x}^\ast] )  -\epsilon/2  ]    
   + \sum _{t=1}^T \mathbb{E} [    U_t ( {\bm x}_t )  -   p_{t-1;\eta} (  {\bm x}_t  )   ] .
\end{align*}
\end{lemma}
\begin{proof}
From the definition of $r_t (\epsilon)$, we have
\begin{align*}
r_t (\epsilon) &= ( p_{\rm upper} ({\bm x}^\ast ) -\epsilon) -   p_{\rm upper } (  {\bm x}_t  )  \\
&= (p_{\rm upper} ({\bm x}^\ast ) - p_{\rm upper} ([{\bm x}^\ast] ) -\epsilon/2 ) + (p_{\rm upper} ([{\bm x}^\ast] ) -U_t ([{\bm x}^\ast]) -\epsilon/2) ) \\
&+ (  U_t ([{\bm x}^\ast] ) - U_t ({\bm x}_t ) ) + (U_t ({\bm x}_t ) - p_{\rm upper} ({\bm x}_t )  ) .
\end{align*}
Thus, from $ p_{t-1;\eta} ( {\bm x} ) \leq p_{\rm upper} ({\bm x} ) \leq p_{t-1;\eta} ( {\bm x} ) +\tilde{p}_{2 \eta} ({\bm x})$ and
$$
\mathbb{E} [ U_t ([{\bm x}^\ast] ) - U_t ({\bm x}_t )] =
\mathbb{E}[ \mathbb{E} [ U_t ([{\bm x}^\ast] ) - U_t ({\bm x}_t ) \mid H_t] ] , \
\mathbb{E} [ U_t ([{\bm x}^\ast] ) - U_t ({\bm x}_t) \mid H_t] =0,
$$
taking expectation and summing over $t$ we get the desired inequality.
\end{proof}

\begin{lemma}\label{lem:upper_dis}
Let $\epsilon >0$ and $\zeta = \lceil d (\epsilon/4) ^{-1} b r  \sqrt{\log ( 4 d a  / \epsilon ) } \rceil $. Then,
the following holds:
$$
\sum _{t=1}^T \mathbb{E} [ p_{\rm upper} ({\bm x}^\ast ) -  p_{\rm upper} ([{\bm x}^\ast] ) -\epsilon/2 ] \leq 0.
$$
\end{lemma}
\begin{proof}
By using the same argument as proof of Lemma \ref{lem:stochastic}, with probability at least $1-\epsilon/4$,
the following holds:
$$
\forall {\bm x} \in \mathcal{X}, \ | p_{\rm upper} ({\bm x}) -   p_{\rm upper} ([ {\bm x} ]) | \leq b \sqrt{\log (4 d a / \epsilon)} r d / \zeta.
$$
 Hence, from the definition of $\zeta$, the following holds with probability at least $1-\epsilon /4$:
$$
\forall {\bm x} \in \mathcal{X}, \ | p_{\rm upper} ({\bm x}) -   p_{\rm upper} ([ {\bm x} ]) | \leq \epsilon/4.
$$
Here, let $\tilde{H} = \{ f \mid \forall {\bm x} \in \mathcal{X}, \ | p_{\rm upper} ({\bm x}) -   p_{\rm upper} ([ {\bm x} ]) | \leq \epsilon /4 \}$.
Then, noting that $p_{\rm upper} ({\bm x} ) - p_{\rm upper} ({\bm x}^\prime )  \leq 1$, we obtain
\begin{align*}
&\mathbb{E} [ p_{\rm upper} ({\bm x}^\ast ) -  p_{\rm upper} ([{\bm x}^\ast] ) ] \\
&=
\mathbb{E} [ \1  [ f \in \tilde{H}   ]    (p_{\rm upper} ({\bm x}^\ast ) -  p_{\rm upper} ([{\bm x}^\ast] )   )] +
\mathbb{E} [ \1  [ f \notin \tilde{H}   ]    (p_{\rm upper} ({\bm x}^\ast ) -  p_{\rm upper} ([{\bm x}^\ast] )   )]  \\
&\leq \mathbb{E} [ \1  [ f \in \tilde{H}   ]    |p_{\rm upper} ({\bm x}^\ast ) -  p_{\rm upper} ([{\bm x}^\ast] )   |] +
\mathbb{E} [ \1  [ f \notin \tilde{H}   ]   ]  \\
&\leq \mathbb{E} [ \1  [ f \in \tilde{H}   ]   \epsilon/4]  +
\mathbb{E} [ \1  [ f \notin \tilde{H}   ]   ]  \leq \epsilon/4 + \epsilon/4 =\epsilon/2.
\end{align*}
Therefore, we get the desired inequality.
\end{proof}

\begin{lemma}\label{lem:2eta_infinite}
Let $\epsilon >0$ and $2 \eta = \min \{  \frac{ \epsilon \sigma_{0,min}}{ 8  },   \frac{ \epsilon^3 \sigma_{0,min} }{ 256 |\tilde{\mathcal{X}}|  }  \}$. Then, the following holds for any $t \geq 1$:
$$
\mathbb{E} [   \tilde{p}_{2 \eta} ( [{\bm x}^\ast] )  -\epsilon/2  ]  \leq 0.
$$
\end{lemma}
\begin{proof}
Proof is the same as that of Lemma \ref{lem:expectation0}.
\end{proof}

\begin{lemma}\label{lem:xast_closest_bound}
Let $m \geq 2$, $\eta >0$, $\beta_t = t^2 |\tilde{\mathcal{X}}| $ and $U_t =\mu^{(p)}_{t-1;\eta} ({\bm x}) + \beta^{1/m}_t \gamma^{2/m}_{t-1;\eta} ({\bm x})$. Then, the following holds for any $T \geq 1$:
\begin{align*}
\sum _{t=1}^T \mathbb{E} [   p_{t-1;\eta} ( [{\bm x}^\ast]  ) - U_t ( [{\bm x}^\ast] )       ] &\leq \pi^2/6, \\
\sum _{t=1}^T \mathbb{E} [   U_t ( {\bm x}_t ) -  p_{t-1;\eta} ( {\bm x}_t  )        ] &\leq \frac{ 4m |\tilde{\mathcal{X}}|^{1/m}  } {  (2 \pi)^{1/(2m)} \log (1+\sigma^{-2})  }   T^{2/m} \eta^{-(2+1/m) } \kappa_T.
\end{align*}
\end{lemma}
\begin{proof}
Proof is the same as that of Lemma \ref{lem:second_term_bound} and Lemma \ref{lem:first_term_bound}.
\end{proof}
Finally, we prove Theorem \ref{thm:BPT-TS-infinite}.
\begin{proof}
From Lemma \ref{lem:4bunkatu}, \ref{lem:upper_dis}, \ref{lem:2eta_infinite} and \ref{lem:xast_closest_bound},
noting that $| \tilde{\mathcal{X}}  | = \zeta^d$ we get Theorem \ref{thm:BPT-TS-infinite}.
\end{proof}

\subsection{Convergence of BPT-LSE when $\mathcal{X}$ is infinite set}
In this subsection, we give the theorem about the accuracy and convergence for BPT-LSE.
First, for  BPT-LSE, we redefine $\mathcal{H}_t $, $\mathcal{L}_t $ as
\begin{align*}
\mathcal{H}_t = \{ {\bm x} \in \mathcal{X} \mid l_{t;\eta} ([{\bm x}]) > \alpha -\epsilon/2 \}, \
\mathcal{L}_t = \{ {\bm x} \in \mathcal{X} \mid u_{t;\eta} ([{\bm x}]) < \alpha +\epsilon/2 \}. 
\end{align*}
Then, the following theorem holds:
\begin{theorem}\label{thm:BPT-LSE-infinite}
Let $\delta \in (0,1)$, $\alpha \in (0,1)$, $\epsilon >0$, $m \geq 2$, $\zeta = \lceil d (\epsilon /4)^{-1} b r \sqrt{  \log (4 d a /\delta) }  \rceil  $,  $\beta _t = 2 (1+ \zeta^d ) \pi^2 t^2 /(3 \delta ) $ and    $2 \eta = \min \{ \frac{ \epsilon \sigma_{0,min}}{8}, \frac{ \epsilon^2 \delta \sigma_{0,min}}{128 \zeta ^d }\}$.
Furthermore, let $C_3$ be defined as in Theorem \ref{thm:lse_convergence}.
Then, running BPT-LSE terminates after at most $T$ rounds, where $T$ is the smallest positive integer satisfying \eqref{eq:LSEbound}.
Moreover, with probability at least $1-\delta$, BPT-LSE returns $\epsilon$-accurate solution, i.e., the following inequality holds:
$$
Pr  \left \{
\sup_{ {\bm x} \in \mathcal{X} } e_\alpha ({\bm x} ) \leq \epsilon
\right \}  \geq 1-\delta.
$$
\end{theorem}
\begin{proof}
From the definition of ${\bm x}_t$, the following inequality holds for any $t \geq 1$ and ${\bm x}^\prime \in \tilde{\mathcal{X} } $:
\begin{align}
\text{STR}_t ({\bm x}^\prime )  \leq \text{STR}_t ({\bm x}_t) . \label{eq:str_ine}
\end{align}
 In addition, by using the same argument as the proof of Lemma \ref{lem:str_bound},
the following holds for some $t^\prime \leq t$:
\begin{align}
\text{STR}_{t^\prime} ({\bm x}_{t^\prime} )  \leq \frac{  C_3 \eta^{-(2+1/m)} \beta^{1/m}_t \kappa_t      }{t}. \label{eq:str_bu}
\end{align}
Therefore, if a positive integer $T$ satisfies  \eqref{eq:LSEbound}, from \eqref{eq:str_ine} and
\eqref{eq:str_bu} we have
$$
\forall {\bm x}^\prime  \in \tilde{\mathcal{X} } , \ \text{STR}_{\tilde{t} } ({\bm x}^\prime ) < \epsilon /2,
$$
where $\tilde{t} \leq T$.
Hence, from the classification rule, each point ${\bm x} \in \mathcal{X}$ is classified into $\mathcal{H} _{\tilde{t}} $ or
 $\mathcal{L} _{\tilde{t} }$ at time $\tilde{t}$.

On the other hand, from Lemma \ref{lem:stochastic}, with probability at least $1-\delta/2$,
the following holds for any ${\bm x} \in \mathcal{X}$ and $t \geq 1$:
\begin{align}
-\epsilon/4 + l_{t;\eta} ([{\bm x}]) \leq p_{\rm upper} ({\bm x}) \leq \epsilon/4 + \tilde{p}_{2 \eta} ([{\bm x}]) +  u_{t;\eta} ([{\bm x}]) . \nonumber
\end{align}
Moreover, from Lemma \ref{lem:prior_pro}, with probability at least $1-\delta/2$, it holds that
$$
\forall {\bm x} \in \mathcal{X}, \ \tilde{p}_{2 \eta} ([{\bm x}]) < \epsilon/4.
$$
By combining these, with probability at least $1-\delta$, the following holds inequality holds for any ${\bm x} \in \mathcal{X}$ and $t \geq 1$:
\begin{align}
-\epsilon/4 + l_{t;\eta} ([{\bm x}]) \leq p_{\rm upper} ({\bm x}) \leq \epsilon/2 +  u_{t;\eta} ([{\bm x}]) . \label{eq:strltut}
\end{align}
Thus, when ${\bm x} \in \mathcal{H}_t$, $l_{t;\eta} ([{\bm x}])$ satisfies $l_{t;\eta} ([{\bm x}]) >\alpha -\epsilon/2$.
 By substituting this inequality into \eqref{eq:strltut}, we have
$$
\alpha - 3 \epsilon /4 \leq p_{\rm upper} ({\bm x}).
$$
Similarly, when ${\bm x} \in \mathcal{L}_t$, we get
$$
p_{\rm upper} ({\bm x}) \leq \epsilon.
$$
This means that $e_\alpha ({\bm x} ) \leq \epsilon$.
Therefore, with probability at least $1-\delta$, each point ${\bm x} \in \mathcal{X}$ satisfies $e_\alpha ({\bm x} ) \leq \epsilon$
 when BPT-LSE terminates.
\end{proof}

\subsection{Details for condition {\sf (C1)}}
In this subsection, we consider a sufficient condition of  {\sf (C1)} because all theorems in this section are derived based on it.
However, it is difficult to derive a sufficient condition 
 because $p_{\rm upper} ({\bm x} )$ does not follow GP.
For this reason, instead of  {\sf (C1)}, we derive a sufficient condition for the following inequality:
\begin{align}
Pr ( \forall {\bm x} \in \mathcal{X}, \ | p_{\rm upper} ({\bm x} ) - p_{\rm upper} ([{\bm x} ]) | \leq \epsilon ) \geq 1-\delta.
\label{eq:de_ine}
\end{align}
Note that \eqref{eq:de_ine}, derived from {\sf (C1)}, is the basis of proofs of all theorems in this section.
In order to derive  \eqref{eq:de_ine}, we assume the following condition about stochastic Lipschitz continuity for $f({\bm x},{\bm w})$:
\begin{description}
\item [(C2)] There exists positive constants $a^\prime$ and $b^\prime$ such that
$$
Pr \left \{ \sup_{ {\bm\theta}\equiv ({\bm x},{\bm w}) \in \mathcal{X} \times \Omega } | \partial f ({\bm \theta}) / \partial \theta_j | >L \right  \}  \leq a^\prime  e^{-(L/{b^\prime} )^2}, \ j=1,\ldots,d+k.
$$
\end{description}
This condition is also used in \cite{srinivas2009gaussian} to  derive theoretical guarantee of   (original) GP-UCB
under the case of infinite sample space.
It is known that   {\sf (C2)} holds under mild conditions such as  compactness of the sample space and  smoothness of the kernel function \cite{srinivas2009gaussian,ghosal2006posterior}. 
Hereafter, assume that $\mathcal{X} \times \Omega \subset [0,r]^{d+k} $ is compact and convex, which is the same assumption of   \cite{srinivas2009gaussian}.
In addition, we allow that a prior mean function $\mu ( {\bm x},{\bm w} ) $ of GP is non-zero.
For simplicity, we set $\mu ( {\bm x},{\bm w} ) = \mu <h$.
Then, the following lemma holds:
\begin{lemma}\label{lem:sc1}
Let $\delta \in (0,1)$, $\epsilon >0$, $h- \mu > \eta >0$ and
$\zeta = \lceil d \eta^{-1} b^\prime r \sqrt{  \log (2 d a^\prime /\delta) }  \rceil  $.
 Assume that {\sf (C2)} holds.
Then, with probability at least $1-\delta$, the following holds for any ${\bm x} \in \mathcal{X}$:
\begin{align}
| p_{\rm upper} ({\bm x}) - p_{\rm upper} ([{\bm x}])   | \leq \phi  ( h-\eta-\mu )
\frac{ \eta  }{\sigma_{0,min}} + \sqrt{ \frac{ 4 \zeta^d \eta \phi  ( h-\eta-\mu )   }{ \delta \sigma_{0,min} }   }. \label{eq:phi_bound}
\end{align}
Moreover, if  $ \phi (h-\eta - \mu )$ satisfies
\begin{align}
\phi (h-\eta - \mu ) \leq \min \left \{
\frac{ \sigma_{0,min} \epsilon }{2 \eta} ,
\frac{ \delta \sigma_{0,min} \epsilon^2 }{ 16 \zeta^d \eta    }
\right \}, \label{eq:phi_condition}
\end{align}
then, the inequality \eqref{eq:de_ine} holds.
\end{lemma}
\begin{proof}
From the condition {\sf (C2)}, with probability at least $1- d a^\prime  e^{-(L/{b^\prime})^2}$, the following inequality holds:
$$
\forall {\bm x} \in \mathcal{X}, \forall {\bm w} \in \Omega, \ | f({\bm x},{\bm w} )-  f([{\bm x}],{\bm w} )| \leq L \| {\bm x} -[{\bm x}]   \|_1 .
$$
Thus, since $\| {\bm x} -[{\bm x}]   \|_1 \leq r d / \zeta $, by choosing $L= b ^\prime \sqrt{ \log (2 d a^\prime/ \delta) }$
we have
$$
\forall {\bm x} \in \mathcal{X}, \forall {\bm w} \in \Omega, \ | f({\bm x},{\bm w} )-  f([{\bm x}],{\bm w} )| \leq b ^\prime \sqrt{ \log (2 d a^\prime/ \delta) } r d / \zeta.
$$
Note that this inequality holds with probability at least $1- \delta/2$.
Hence, from the definition of $\zeta$, we get
\begin{align*}
\forall {\bm x} \in \mathcal{X}, \forall {\bm w} \in \Omega, \ | f({\bm x},{\bm w} )-  f([{\bm x}],{\bm w} )| \leq \eta . 
\end{align*}
Here, with probability at least $1- \delta/2$, for any ${\bm x} \in \mathcal{X}$ it holds that
\begin{align*}
p_{\rm upper} ([{\bm x}]) &\leq \int _\Omega  \1 [ f([{\bm x}],{\bm w}) > h+\eta ] p({\bm w}) {\rm d} {\bm w} +
\int _\Omega  \1 [ h+\eta \geq f([{\bm x}],{\bm w}) > h ] p({\bm w}) {\rm d} {\bm w} \\
&\leq \int _\Omega  \1 [ f({\bm x},{\bm w}) + \eta > h+\eta ] p({\bm w}) {\rm d} {\bm w} +
\int _\Omega  \1 [ h+\eta \geq f([{\bm x}],{\bm w}) > h ] p({\bm w}) {\rm d} {\bm w} \\
&=p_{\rm upper} ({\bm x}) +
\int _\Omega  \1 [ h+\eta \geq f([{\bm x}],{\bm w}) > h ] p({\bm w}) {\rm d} {\bm w} \\
&\equiv   p_{\rm upper} ({\bm x}) +p_\eta ([{\bm x}])    .
\end{align*}
Similarly, we obtain
\begin{align*}
p_{\rm upper} ([{\bm x}]) &\geq \int _\Omega  \1 [ f([{\bm x}],{\bm w}) > h-\eta ] p({\bm w}) {\rm d} {\bm w} -
\int _\Omega  \1 [ h \geq f([{\bm x}],{\bm w}) > h -\eta ] p({\bm w}) {\rm d} {\bm w} \\
&\geq \int _\Omega  \1 [ f({\bm x},{\bm w}) - \eta > h-\eta ] p({\bm w}) {\rm d} {\bm w} -
\int _\Omega  \1 [ h \geq f([{\bm x}],{\bm w}) > h -\eta] p({\bm w}) {\rm d} {\bm w} \\
&=p_{\rm upper} ({\bm x}) -
\int _\Omega  \1 [ h \geq f([{\bm x}],{\bm w}) > h - \eta ] p({\bm w}) {\rm d} {\bm w} \\
&\equiv p_{\rm upper} ({\bm x}) -  p_{-\eta} ([{\bm x}]) .
\end{align*}
Therefore, by combining these we get
\begin{align}
|  p_{\rm upper} ({\bm x}) -p_{\rm upper} ([{\bm x}]) | \leq \max \{   p_\eta ([{\bm x}]), p_{-\eta} ([{\bm x}] ) \} . \label{eq:max_bound}
\end{align}
In addition, by using the same argument as the proof of Lemma \ref{lem:prior_pro},
with probability at least $1- \delta/2$ the following holds for any ${\bm x} \in \mathcal{X}$:
\begin{align*}
p_\eta ([ {\bm x}] ) &< \mathbb{E} [p_\eta ([ {\bm x}] )  ]   + \sqrt{\frac{ 4 | \tilde{\mathcal{X} }|  \mathbb{E} [p_\eta ([ {\bm x}] )  ]   }{ \delta  } }, \\
p_{-\eta} ([ {\bm x}] ) &< \mathbb{E} [p_{-\eta} ([ {\bm x}] )  ]   + \sqrt{\frac{ 4 | \tilde{\mathcal{X} }|  \mathbb{E} [p_{-\eta} ([ {\bm x}] )  ]   }{ \delta  } }.
\end{align*}
Moreover, $\mathbb{E} [p_\eta ([ {\bm x}] )  ]$ can be expressed as
$$
\mathbb{E} [p_\eta ([ {\bm x}] )  ] = \int _\Omega \left \{ \Phi \left (\frac{h+\eta - \mu }{\sigma_0 ([ {\bm x}],{\bm w})} \right ) - \Phi \left (\frac{h-\mu}{\sigma_0 ( [{\bm x}],{\bm w})} \right ) \right \} p({\bm w}) \text{d} {\bm w}.
$$
Furthermore, from Taylor's expansion, for any $0 < a< b$, it holds that
$$
\Phi (b) = \Phi (a) + \phi (c) (b-a) \leq  \Phi (a) + \phi (a) (b-a),
$$
where $c \in (a,b)$ and the last inequality is given by using $\phi (0) \geq \phi (a) \geq (c)$ and $b-a \geq 0$.
Hence, we have
\begin{align*}
\mathbb{E} [p_\eta ([ {\bm x}] )  ] &=    \int _\Omega \phi \left ( \frac{h-\mu}{\sigma_0 ([{\bm x}],{\bm w})} \right )
\frac{ \eta  }{\sigma_0 ([{\bm x}],{\bm w})}  p({\bm w} ) {\rm d} {\bm w} \\
& \leq \int _\Omega \phi \left ( \frac{h-\mu}{1} \right )
\frac{ \eta  }{\sigma_{0,min}}  p({\bm w} ) {\rm d} {\bm w} =\phi  ( h-\mu )
\frac{ \eta  }{\sigma_{0,min}} .
\end{align*}
Therefore, we get
\begin{align}
p_\eta ([ {\bm x}] ) < \phi  ( h-\mu )
\frac{ \eta  }{\sigma_{0,min}} + \sqrt{ \frac{ 4 |\tilde{\mathcal{X} } | \eta \phi  ( h-\mu )   }{ \delta \sigma_{0,min} }   }. \label{eq:p_eta_bound}
\end{align}
Similarly, we have
\begin{align}
p_{-\eta} ([ {\bm x}] ) < \phi  ( h-\eta-\mu )
\frac{ \eta  }{\sigma_{0,min}} + \sqrt{ \frac{ 4 |\tilde{\mathcal{X} } | \eta \phi  ( h-\eta-\mu )   }{ \delta \sigma_{0,min} }   }. \label{eq:p_minuseta_bound}
\end{align}
Here, it holds that $\phi  ( h-\eta-\mu )   \geq \phi  ( h-\mu )$ because $0 \leq h- \eta - \mu   \leq h- \mu$.
In addition, from the definition of $\tilde{\mathcal{X}}$, $ |\tilde{\mathcal{X} } | $ is $\zeta^d$.
Thus, by substituting   \eqref{eq:p_eta_bound} and \eqref{eq:p_minuseta_bound} into \eqref{eq:max_bound},
with probability at least $1-\delta$, the inequality \eqref{eq:phi_bound} holds for any ${\bm x} \in \mathcal{X}$.
Finally, if \eqref{eq:phi_condition} holds, then we get  \eqref{eq:de_ine}.
\end{proof}
Note that \eqref{eq:phi_condition} is the strong condition to derive \eqref{eq:de_ine}. Thus, it is important to derive a mild condition. 
It is also important to derive a sufficient condition  for the condition  {\sf (C1)} to be satisfied.
These will be considered in our future work.

\section{Details of Numerical Experiments}
\label{app:detail-exeperiments}
In this appendix, we describe the details of the experimental results summarized in \S5.
For the sake of readability, some results already described in \S5 are also replicated here.

\subsection{Artificial Data Experiments}\label{sec:art_data}

In this subsection, we present experiments on artificial data.

\subsubsection{Optimization Experiments}\label{sec:exp_bo_app}
Here, we tested the performances of the proposed methods in the optimization setting.
We evaluated the performances of the algorithms up to step $t$ by
$p_{\text{upper}}(\bm{x}^*) - p_{\text{upper}}(\hat{\bm{x}}_t)$,
where
$\hat{\bm{x}}_t$
is the estimated maximizer reported by the algorithm at step $t$.
In BPT-UCB and BPT-TS,
$\hat{\bm{x}}_t$
is defined as
${\rm argmax}_{t^\prime =1, \ldots, t}\mu_{t}^{(p)}(\bm{x}_{t^\prime})$.
For comparison,
we considered the following existing methods.
Although these existing methods can be directly applied to the problem setup considered in this paper,
they were originally designed to optimize different objective functions.

\begin{description}

\item[GP-UCB \cite{srinivas2009gaussian}]
    First, we considered the GP-UCB method by assuming that $\bm{w}$ is fixed to its mean. Specifically,
    in this method, we chose $\bm{x}_t$ and $\hat{\bm{x}}_t$
    as
    \begin{align*}
        \bm{x}_t &= \argmax_{\bm{x} \in \mathcal{X}} \text{ucb}_t^f(\bm{x},
        \mathbb{E}[\bm{w}]), \\
        \hat{\bm{x}_t} &= \argmax_{t^\prime =1, \ldots, t}\text{lcb}^f(\bm{x}_{t^\prime}).
    \end{align*}
    where $\text{lcb}_t^f(\bm{x}, \bm{w})$ and $\text{ucb}_t^f(\bm{x}, \bm{w})$ represent the LCB and UCB of $f(\bm{x}, \bm{w})$ at step $t$. To compute
    $\text{lcb}_t^f$ and $\text{ucb}_t^f$, we used $\beta_t^{1/2} = 2$.
 \item[StableOpt \cite{bogunovic2018adversarially}]
    This method was designed to find the worst-case maximizer
    within a user specified domain $\Delta \subset \Omega$.
    Given $\Delta$, $(\bm{x}_t, \bm{w}_t)$ were
    chosen as:
    \begin{align*}
        \bm{x}_t = \argmax_{\bm{x} \in \mathcal{X}}\min_{\bm{w} \in \Delta}~\text{ucb}_t^{f}(\bm{x}, \bm{w}),~
        \bm{w}_t = \argmin_{\bm{w} \in \Delta}~\text{lcb}_t^{f}(\bm{x}_t, \bm{w}).
    \end{align*}
    and we defined $\hat{\bm{x}}_t$ as ${\rm argmax}_{t^{\prime}=1, \ldots, t}\min_{\bm{w} \in \Delta} \text{lcb}_t^{f}(\bm{x}_{t^{\prime}}, \bm{w})$.
    We set $\beta_t^{1/2} = 2$ to compute $\text{ucb}_t^{f}$ and $\text{lcb}_t^{f}$.
    In this method, there is no canonical way to choose $\Delta$.
    In our experiments, we defined $\Delta$ as
    $50\%$ credible interval of $\bm{w}$. We confirmed that this choice worked well in all the settings of our experiment.

 \item[BQO-EI \cite{nguyen2020distributionally}]
    This method was designed to find the maximum of expected function
    $g(\bm{x}) \coloneqq \int_{\Omega} f(\bm{x}, \bm{w})p(\bm{w})\text{d}\bm{w}$.
    To this end,
    $\bm{x}_t$ was chosen by EI of $g(\bm{x})$
    and $\bm{w}_t$ was chosen as $\argmax_{\bm{w} \in \Omega}\sigma_{t-1}(\bm{x}_t, \bm{w})$.
    We defined $\hat{\bm{x}}_t$ as ${\rm argmax}_{t^\prime=1, \ldots, t}\int_{\Omega} \mu_t(\bm{x}_{t^\prime})p(\bm{w})\text{d}\bm{w}$.
    Note that this definition of $\hat{\bm{x}}_t$ represents the point that has the highest posterior mean of $g$
    among $\{\bm{x}_{t^\prime}\}_{t^\prime=1}^t$.
    \item[BQO-UCB]
	      In this method, $\bm{x}_t$ was chosen by UCB of $g$, and $\bm{w}_t$ was chosen in the same way as BQO-EI.
	       We set $\beta_t^{1/2} = 2$ to compute the UCB of $g$.
	       Additionally, $\hat{\bm{x}}_t$ was chosen in the same way as BQO-EI.

 \item[BQO-TS \cite{nguyen2020distributionally}]
	    In this method,
	    $\bm{x}_t$ was chosen
	    by the posterior probability
	    such that
	    $g(\bm{x})$
	    is maximized
	    and
	    $\bm{w}_t$ was chosen
	    in the same way as BQO-EI.
	    Additionally,
	    $\hat{\bm{x}}_t$ was chosen in the same way as BQO-EI.

\end{description}

We also tested the performances of Random Sampling (RS),
in which samples
$(\bm{x}_t, \bm{w}_t)$ were
uniformly at random,
and
$\hat{\bm{x}}_t$ was defined
as ${\rm argmax}_{t^\prime=1, \ldots, t}\mu_{t}^{(p)}(\bm{x}_{t^\prime})$.

Furthermore,
we also considered the adapted versions of these existing methods,
in which
$(\bm{x}_t, \bm{w}_t)$
were chosen as above,
while
their
$\hat{\bm{x}_t}$
was selected
in the same way as the proposed method, i.e.,
as
${\rm argmax}_{t^\prime=1, \ldots, t}\mu_{t}^{(p)}(\bm{x}_{t^\prime})$.
We denote these adapted versions of the extended methods with prefix of {\bf Pmax} (e.g., the adapted version of {\bf GP-UCB} is referred to as {\bf Pmax-GP-UCB}).
Furthermore,
in BPT-UCB,
since the theoretically recommended values of
$\beta_t$
and
$m$
are well known to be overly conservative,
we used
$\beta_t=2, m=2$ and chose $\eta = 0$ for simplicity.

\paragraph{GP Test Function}\label{sec:bo_gp}
First,
we tested the performances on a test function generated by two-dimensional GP.
We used Gaussian kernel
$k((\bm{x}, \bm{w}), (\bm{x}^{\prime}, \bm{w}^{\prime})) = \sigma_{\text{ker}}^2 \exp \left( (\|\bm{x} - \bm{x}^{\prime}\|^2 + \|\bm{w} - \bm{w}^{\prime}\|^2)/2l^2 \right)$
with
$l=0.5,~\sigma_{\text{ker}} = 1$,
and defined
$\mathcal{X}$
and
$\Omega$
as the $50$ girds points evenly allocated in $[-1, 1]$.
Furthermore,
we defined
$p(w) = \phi(w)/Z$, $Z = \sum_{w \in \Omega} \phi(w)$,
where
$\phi(w)$
is the density function of the standard Normal distribution.
We set $h=0$ and $\sigma=0.001$.

The experimental results are shown in
Fig.~\ref{fig:bo_gp_result}.
The proposed methods have better performances than existing methods.
It is reasonable since the proposed methods are developed to optimize the target task, while existing methods are developed to optimize different robustness measures. Furthermore, adaptive versions of existing methods did not work well
because their selected points $(\bm{x}_t, \bm{w}_t)$ are inefficient for optimizing $p_{\rm {upper}}$.

\begin{figure}[t]
 \centering
 \includegraphics[width=0.5\linewidth]{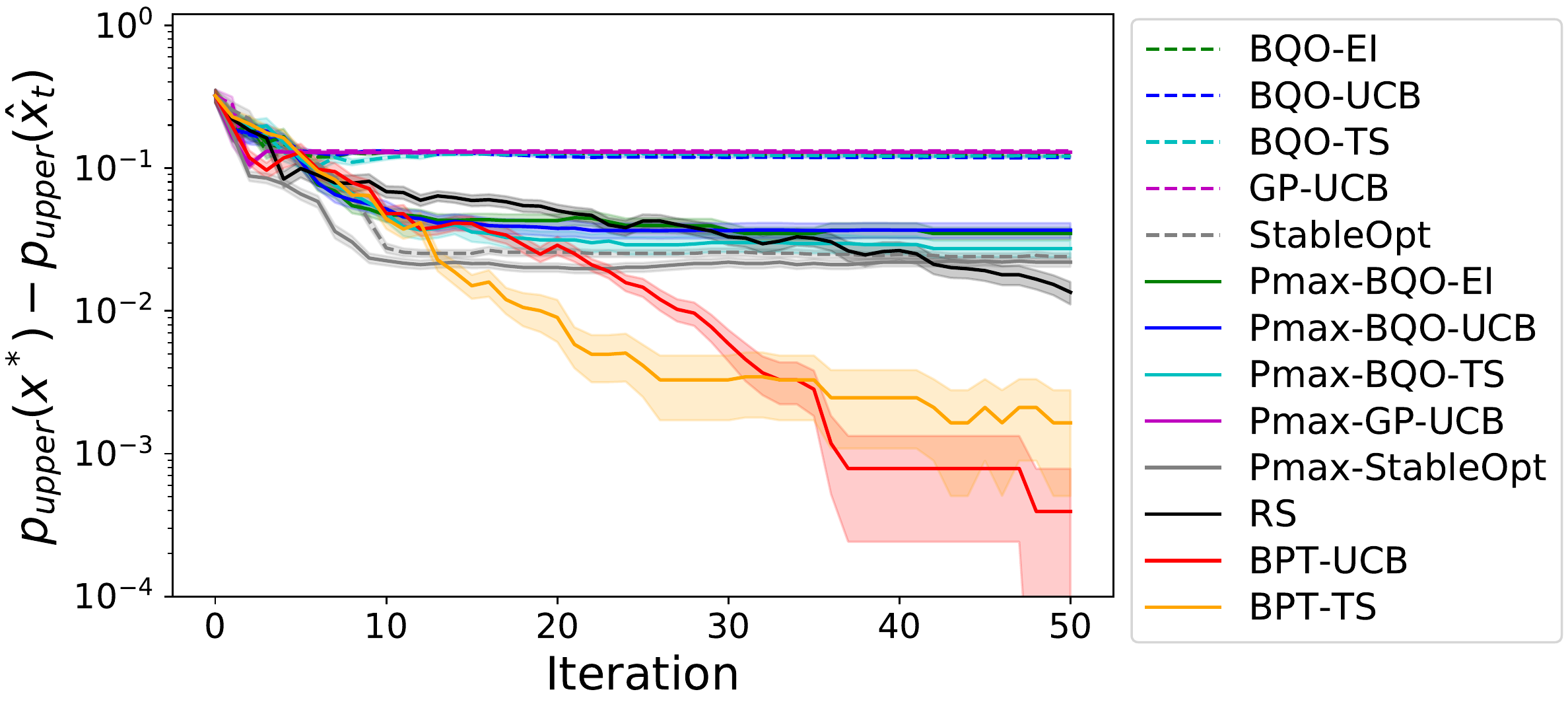}
 \caption{
 The experimental results in optimization setting with a test function generated from GP.
 This plot shows the average performance over 50 trials.
 The shaded area represents the confidence interval of the average performance ($\pm 2\times$[standard error]).
 }
 \label{fig:bo_gp_result}
\end{figure}

\paragraph{Benchmark Functions for Optimization}\label{sec:bo_bench}
We also tested the performances with two benchmark functions called 2d-Rosenbrock function and McCormick function that are often used as benchmark in optimization setting.
First, we rescaled the domain of these functions to $[-1, 1]^2$ and considered $50$ evenly allocated grids points in each dimension.
We set $\mathcal{X}, \Omega$ as the first and the second dimension of the domain, respectively.
Furthermore,
we defined
$p(w) = \text{Gam}(w + 1 \mid 2, 0.5)/Z$, $Z = \sum_{w \in \Omega} \text{Gam}(w + 1 \mid 2, 0.5)$,
where
$\text{Gam}(w \mid a, b)$
is the density of Gamma distribution with parameters $a$ and $b$.
For modeling $f$,
we used Gaussian kernel with
$l=0.5,~\sigma_{\text{ker}} = 150$
in Rosenbrock function,
and with
$l=1,~\sigma_{\text{ker}}=4$
in McCormick function.
We chose
$h=-1000$
in Rosenbrock function,
and
$h=-5$
in McCormick,
and
$\sigma = 0.01$
in both functions.
The experimental results are shown in Fig.~\ref{fig:bo_benchmark_result_app}.
The proposed methods also have better performances than existing methods.

\begin{figure}[t]
    \begin{center}
        \begin{tabular}{cc}
         \includegraphics[width=0.500\linewidth]{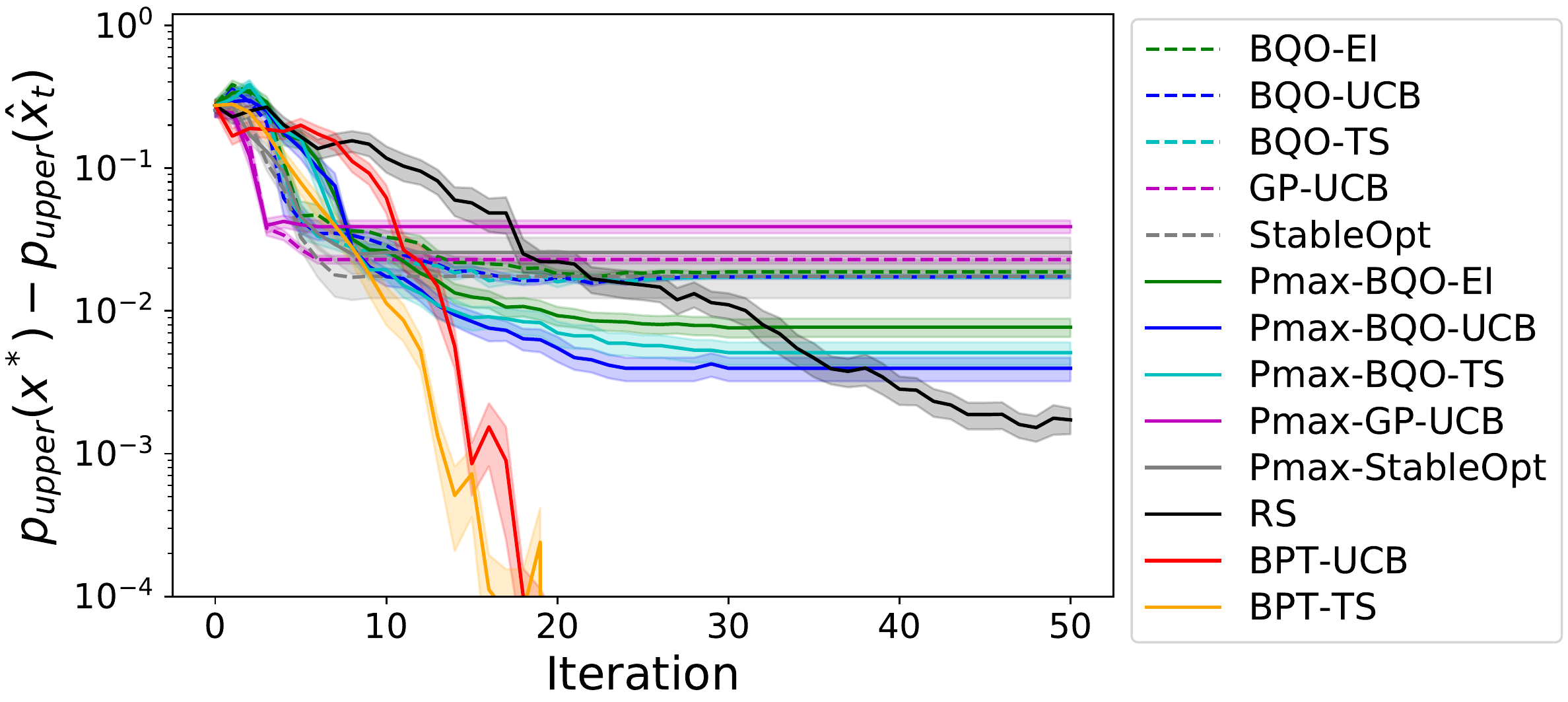} &
         \includegraphics[width=0.500\linewidth]{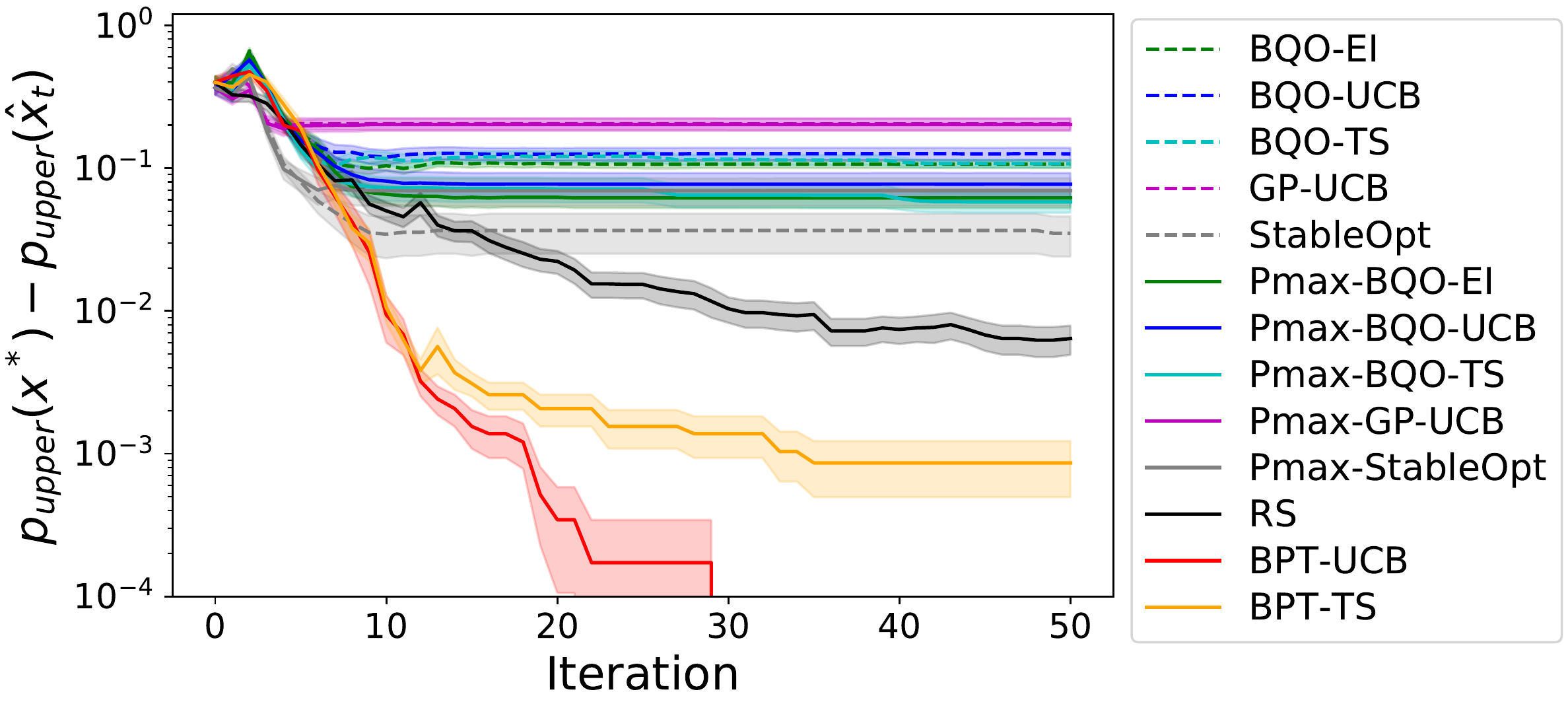} \\
         (a) 2D-Rosenbrock function &
         (b) McCormick function
        \end{tabular}
   \end{center}
 \caption{
 The experimental results in optimization setting with two benchmark functions.
 The left and the right plots represent the results of Rosenbrock function and McCormick function, respectively.
 These plots show the average performances over $50$ trials.
 }
 \label{fig:bo_benchmark_result_app}
\end{figure}

\subsubsection{Level Set Estimation Experiments}\label{sec:exp_lse_app}
We also tested the performance of the proposed method in the LSE setting.
In the LSE setting, we chose the following methods for comparison:
\begin{description}
\item[LSE\cite{bryan2006active}]
        We considered the standard LSE by assuming that $\bm{w}$ is fixed to its mean. Namely,
        we chose $\bm{x}_t$ as
        \begin{align*}
            \bm{x}_t = \argmax_{\bm{x} \in \mathcal{X}}
            \min \left\{ \tilde{u}_t^{f}(\bm{x}) - h,
            h - \tilde{l}_t^{f}(\bm{x}) \right\}
        \end{align*}
        where $[\tilde{l}_t^{f}(\bm{x}), \tilde{u}_t^{f}(\bm{x})]$ denotes the
        credible interval of $f(\bm{x}, \mathbb{E}[\bm{w}])$ at step $t$.
        In this method, at step $t$, estimated superlevel set $\hat{\mathcal{H}}_t$ and sublevel set $\hat{\mathcal{L}}_t$
        are respectively defined as
        \begin{align*}
         \hat{\mathcal{H}}_t = \{\bm{x} \in \mathcal{X} \mid \tilde{l}_t^{(f)}(\bm{x}) > h\},~
         \hat{\mathcal{L}}_t = \{\bm{x} \in \mathcal{X} \mid \tilde{u}_t^{(f)}(\bm{x}) < h\}.
        \end{align*}
        To compute $\tilde{l}_t^{f}$ and $\tilde{u}_t^{f}$, we used $\beta_t^{1/2} = 2$.
\item[StableLSE]
    We considered the LSE version of StableOpt to classify the worst-case function within a domain $\Delta \subset \Omega$.
    We initially constructed the credible interval
    $Q_t^{\text{worst}}$ as
    $Q_t^{\text{worst}}(\bm{x}) = [\min_{\bm{w} \in \Delta} \text{lcb}^{f}(\bm{x}, \bm{w}),~
    \min_{\bm{w} \in \Delta} \text{ucb}^{f}(\bm{x}, \bm{w})]
    \coloneqq [l_t^{\text{wosrt}}(\bm{x}),~u_t^{\text{wosrt}}(\bm{x})]$.
    In the experiment,
    we chose
    $\beta_t^{1/2} = 2$ to compute $\text{lcb}^{f}$ and $\text{ucb}^{f}$.
    Then, $(\bm{x}_t, \bm{w}_t)$ were chosen as
    \begin{align*}
     \bm{x}_t = \argmax_{\bm{x} \in \mathcal{X}} \text{STR}_t^{\text{worst}}(\bm{x}),
     ~\bm{w}_t = \argmax_{\bm{w} \in \Delta} \sigma_{t-1}(\bm{x}_t, \bm{w})
    \end{align*}
    where
    $\text{STR}_t^{\text{worst}}(\bm{x}) = \min \left\{ u_t^{\text{worst}}(\bm{x}) - h, ~ h - l_t^{\text{worst}}(\bm{x})\right\}$.
    In this method,
    $\hat{\mathcal{H}}_t$
    and
    $\hat{\mathcal{L}}_t$
    are respectively defined as
    \begin{align*}
     \hat{\mathcal{H}}_t = \{\bm{x} \in \mathcal{X} \mid l_t^{\text{worst}}(\bm{x}) > h\},~
     \hat{\mathcal{L}}_t = \{\bm{x} \in \mathcal{X} \mid u_t^{\text{worst}}(\bm{x}) < h\}.
    \end{align*}
    Finally, we chose $\Delta$ as $50$\% credible interval of $\bm{w}$ in our experiment.
 \item[BQLSE]
	    This method was designed to classify the expected function
	    $g(\bm{x}) \coloneqq \int_{\Omega} f(\bm{x}, \bm{w})p(\bm{w})\text{d}\bm{w}$.
	    First,
	    the credible interval
	    $Q_t^{(g)}$
	    was constructed as
	    $Q_t(\bm{x}) = [\mu_{t-1}^{(g)}(\bm{x}) - \beta_t^{1/2}\sigma_{t-1}^{(g)}(\bm{x}),~ \mu_{t-1}^{(g)}(\bm{x}) + \beta_t^{1/2}\sigma_{t-1}^{(g)}(\bm{x})] \coloneqq [l_t^{(g)}(\bm{x}),~u_t^{(g)}(\bm{x})]$,
	    where
	    $\mu_{t-1}^{(g)}$
	    and
	    $\sigma_{t-1}^{(g)}$
	    are the posterior mean and variance of
	    $g$,
	    respectively.
	    In the experiment,
	    we chose
	    $\beta_t^{1/2} = 3$.
	    Then, $(\bm{x}_t, \bm{w}_t)$ were chosen as
	    \begin{align*}
	     \bm{x}_t = \argmax_{\bm{x} \in \mathcal{X}} \text{STR}_t^{(g)}(\bm{x}),~\bm{w}_t = \argmax_{\bm{w} \in \Omega} \sigma_{t-1}(\bm{x}_t, \bm{w})
        \end{align*}
	    where
	    $\text{STR}_t^{(g)}(\bm{x}) = \min \left\{ u_t^{(g)}(\bm{x}) - h, ~ h - l_t^{(g)}(\bm{x})\right\}$.
	    In this method,
	    $\hat{\mathcal{H}}_t$
	    and
	    $\hat{\mathcal{L}}_t$
	    are respectively defined as
	    \begin{align*}
	     \hat{\mathcal{H}}_t = \{\bm{x} \in \mathcal{X} \mid l_t^{(g)}(\bm{x}) > h\},~
	     \hat{\mathcal{L}}_t = \{\bm{x} \in \mathcal{X} \mid u_t^{(g)}(\bm{x}) < h\}.
        \end{align*}

\end{description}
We also tested the performances of Random Sampling (RS),
in which
$(\bm{x}_t, \bm{w}_t)$
were sampled
uniformly at random.
We defined
$\hat{\mathcal{H}}_t$ and $\hat{\mathcal{L}_t}$ as in (\ref{eq:straddle}).
Furthermore,
we also considered the adapted versions of the existing methods,
in which
$(\bm{x}_t, \bm{w}_t)$
were chosen as above,
while
their
$\hat{\mathcal{H}}_t$ and $\hat{\mathcal{L}}_t$
were selected
in the same way as the proposed method.
We denote these adapted versions of the extended methods with prefix of {\bf P} (e.g., the adapted version of {\bf LSE} is referred to as {\bf P-LSE}).

For the evaluation of the algorithm performance,
as in \cite{gotovos2013active},
we used F1-score,
which is computed by treating $\mathcal{H}$ and $\mathcal{L}$ as positively and negatively labeled instances, respectively.
Furthermore, we choose
$\beta_t = 1.5, k=2, \eta=0, \epsilon=0$
in BPT-LSE.

\paragraph{GP Test Function}
First,
we tested the performances on the test function sampled from GP in the same way as \ref{sec:bo_gp} except $\alpha = 0.8$.
The results are shown in Fig.~\ref{fig:lse_gp_result}.
F1-score of BPT-LSE converged to $1$. On the other hand, existing methods tend to be low F1-score because their objective  functions have different formulations. Furthermore, adaptive versions of existing methods did not work
well. This is because existing methods tend to finish the classification of their objective
in relatively early stage
hence their sample points $(\bm{x}_t, \bm{w}_t)$ were stacked before our classification
scheme worked well. These results indicate that properly designed classification scheme and
sample strategy of $(\bm{x}_t, \bm{w}_t)$ are important to classify $p_{\rm upper}(\bm{x})$.
\begin{figure}[t]
    \centering
    \includegraphics[width=0.49\linewidth]{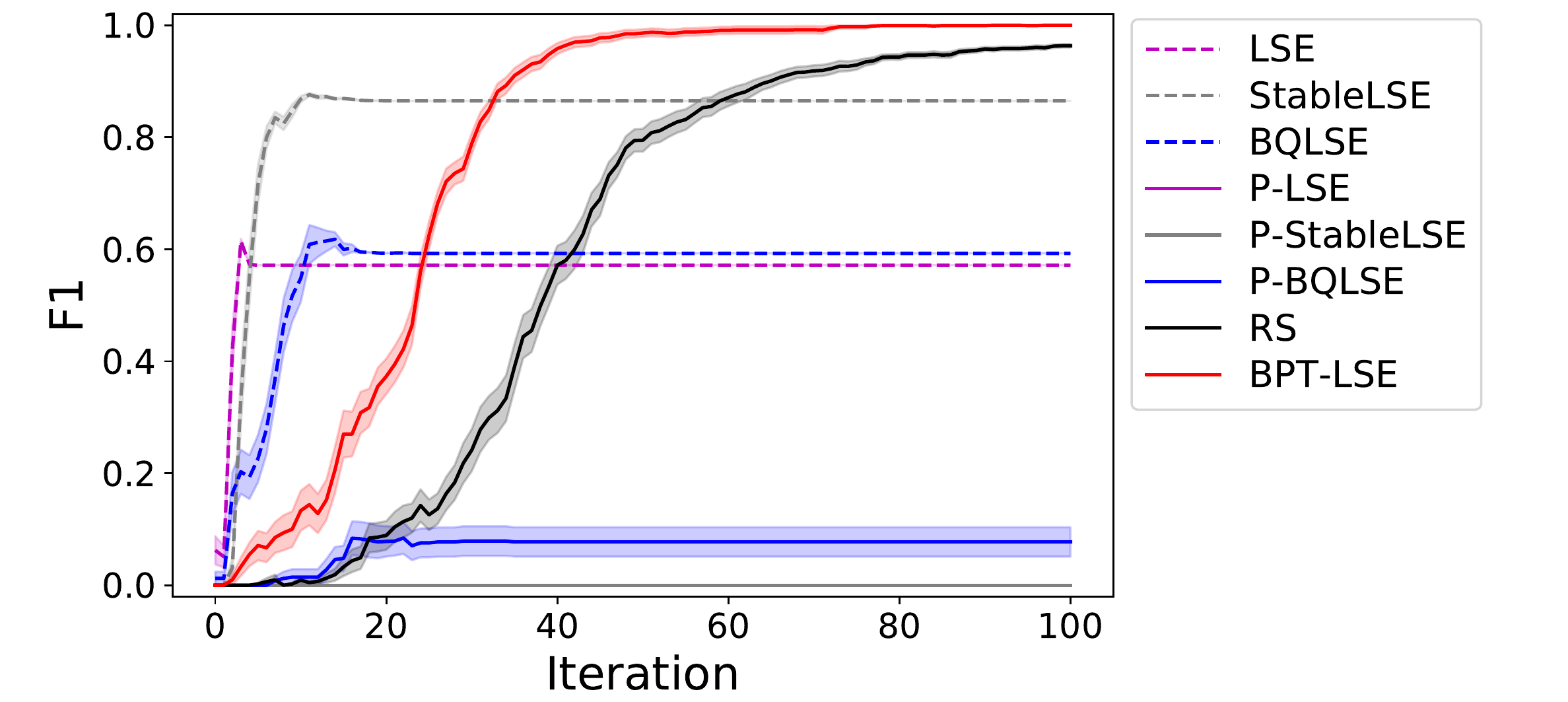}
 \caption{
 The experimental results in the LSE setting with the test function generated from GP.
 This plot shows the average performance over $50$ trials.
 }
 \label{fig:lse_gp_result}
\end{figure}

\paragraph{Benchmark Functions for Optimization}
We also tested the performances on two benchmark functions called Himmelblau function and Goldstein-Price function.
First,
we defined
$\mathcal{X}$
and
$\Omega$
in the same way as in
\ref{sec:bo_bench}.
Additionally,
in Goldstein-Price function,
we rescaled the function range by multiplying $10^{-5}$.
Furthermore,
we defined
$p(w)$
as in
\ref{sec:bo_gp}.
For modeling $f$,
we used Gaussian kernel with
$l=0.5,~\sigma_{\text{ker}} = 200$
in Himmelblau function,
and with
$l=0.4,~\sigma_{\text{ker}}=200$
in Goldstein-Price function.
Moreover,
we chose
$h=-150, \alpha=0.8$
in Himmelblau function,
and
$h=-1, \alpha=0.5$
in Goldstein-Price function,
and
$\sigma = 0.01$ in both functions.

The results are shown in Fig.~\ref{fig:lse_benchmark_result_app}.
Although some of the existing methods could increase the F1-scores in the early stage when the parameter settings were appropriate to the test functions. However, since the target robustness measures in the existing methods are inconsistent with the problem setup considered in this paper, the proposed methods eventually outperformed.

\begin{figure}[t]
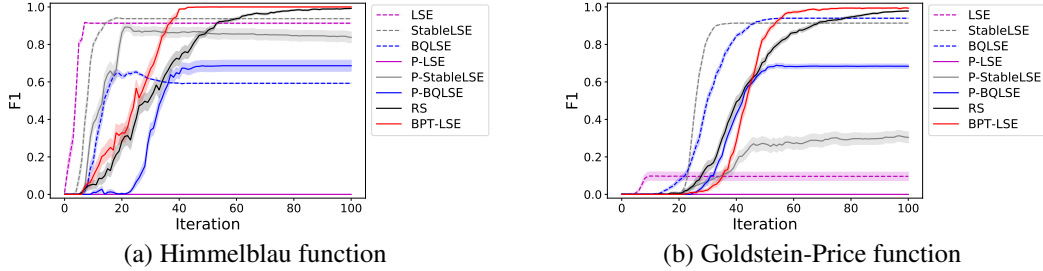

    \begin{center}
        \begin{tabular}{cc}
         \includegraphics[width=0.500\linewidth]{paper_lse_himmelblau_f1.pdf} &
         \includegraphics[width=0.500\linewidth]{paper_lse_goldstein_f1.pdf} \\
         (a) Himmelblau function &
         (b) Goldstein-Price function
        \end{tabular}
   \end{center}
 \caption{
 The experimental results in the LSE setting with two benchmark functions.
 The left and the right plots represent the result on Himmelblau function and Goldstein-Price function, respectively.
 These plots show the average performances over $50$ trials.
 }
 \label{fig:lse_benchmark_result_app}
\end{figure}

\subsubsection{Hyperparameter Sensitivity in BPT-UCB and BPT-LSE}
    In this subsection, we analyzed the effect of the choices $\beta_t$ and $m$ in BPT-UCB and BPT-LSE.
    The experiments were conducted in the same settings as \ref{sec:exp_bo_app} and \ref{sec:exp_lse_app} except
    $\beta_t$ and $m$.
    Fig.~\ref{fig:bo_effect_beta_app} and Fig.~\ref{fig:lse_effect_beta_app}
     show results of BPT-UCB and BPT-LSE with various $\beta_t$, respectively,
    while Fig.~\ref{fig:bo_effect_m_app} and Fig.~\ref{fig:lse_effect_m_app} show
    results of BPT-UCB and BPT-LSE with various $m$, respectively.
    From these results, we observe that BPT-LSE is especially sensitive to the choice of $m$. It is reasonable in LSE
    because $m$ affects not only sample points $(\bm{x}_t, \bm{w}_t)$ but also estimated sets
    $\hat{\mathcal{H}}$ and $\hat{\mathcal{L}}$. For example, if $m=2$ is sufficient to archive
    high precision, larger $m$ makes our classification scheme unnecessarily
    conservative and it leads to low F1-score.

    Although our experiments show that BPT-LSE is sensitive to the choice of $m$,
    since all our experiments show that $m=2$ is the best choice, we recommend $m=2$ as
    the choice of BPT-LSE in practice.


\begin{figure}[t]
    \begin{center}
        \begin{tabular}{ccc}
         \includegraphics[width=0.33\linewidth]{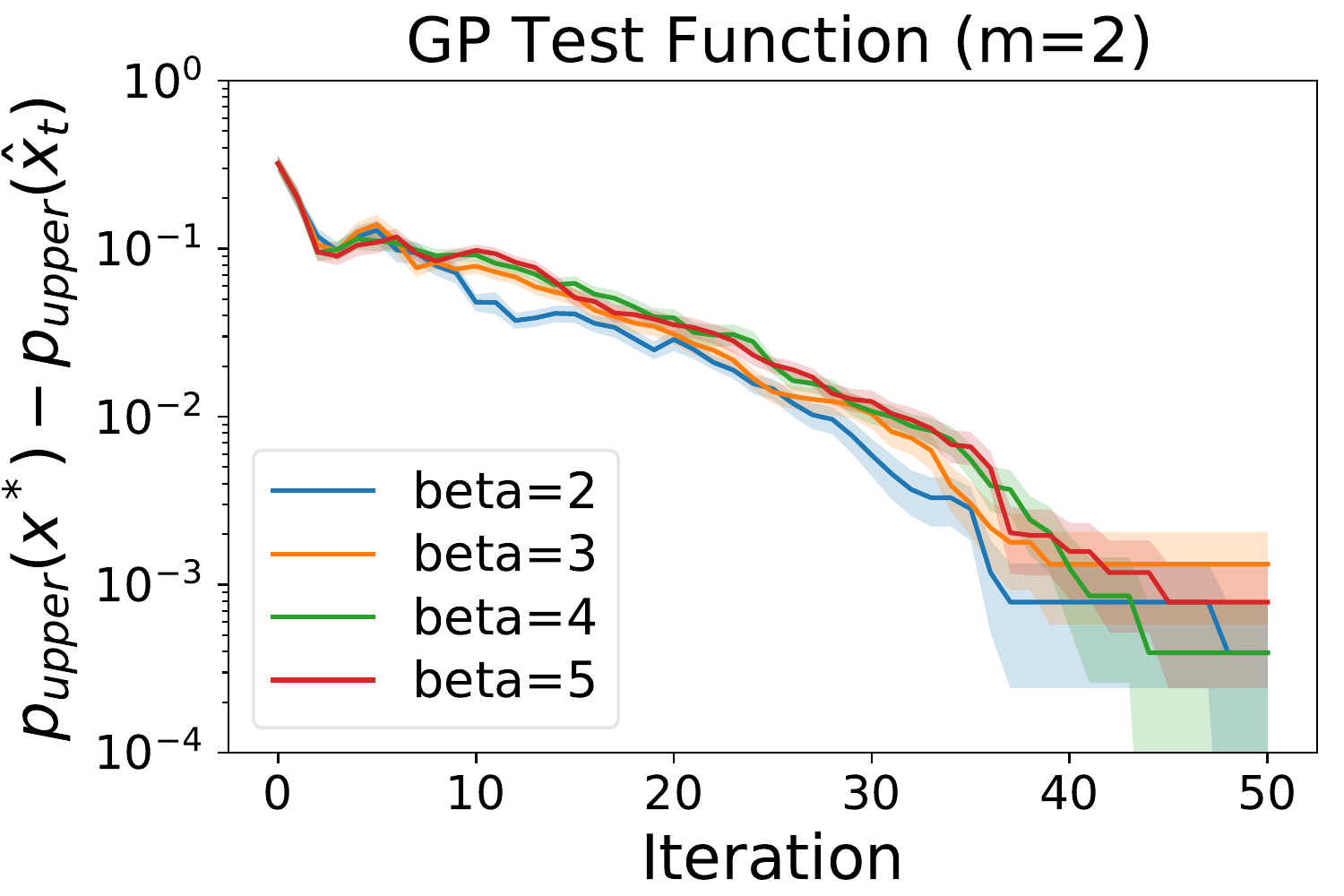} &
         \includegraphics[width=0.33\linewidth]{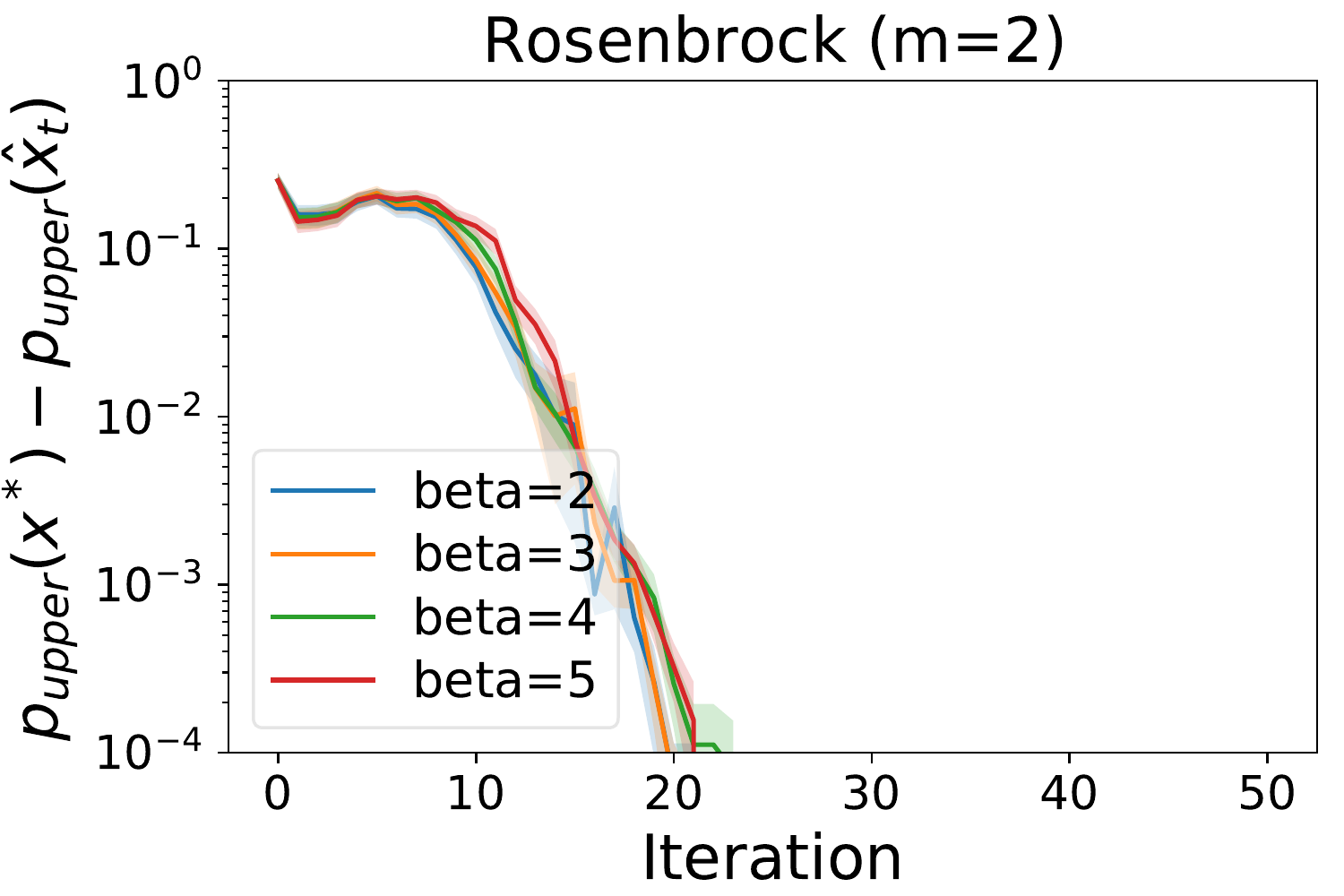} &
         \includegraphics[width=0.33\linewidth]{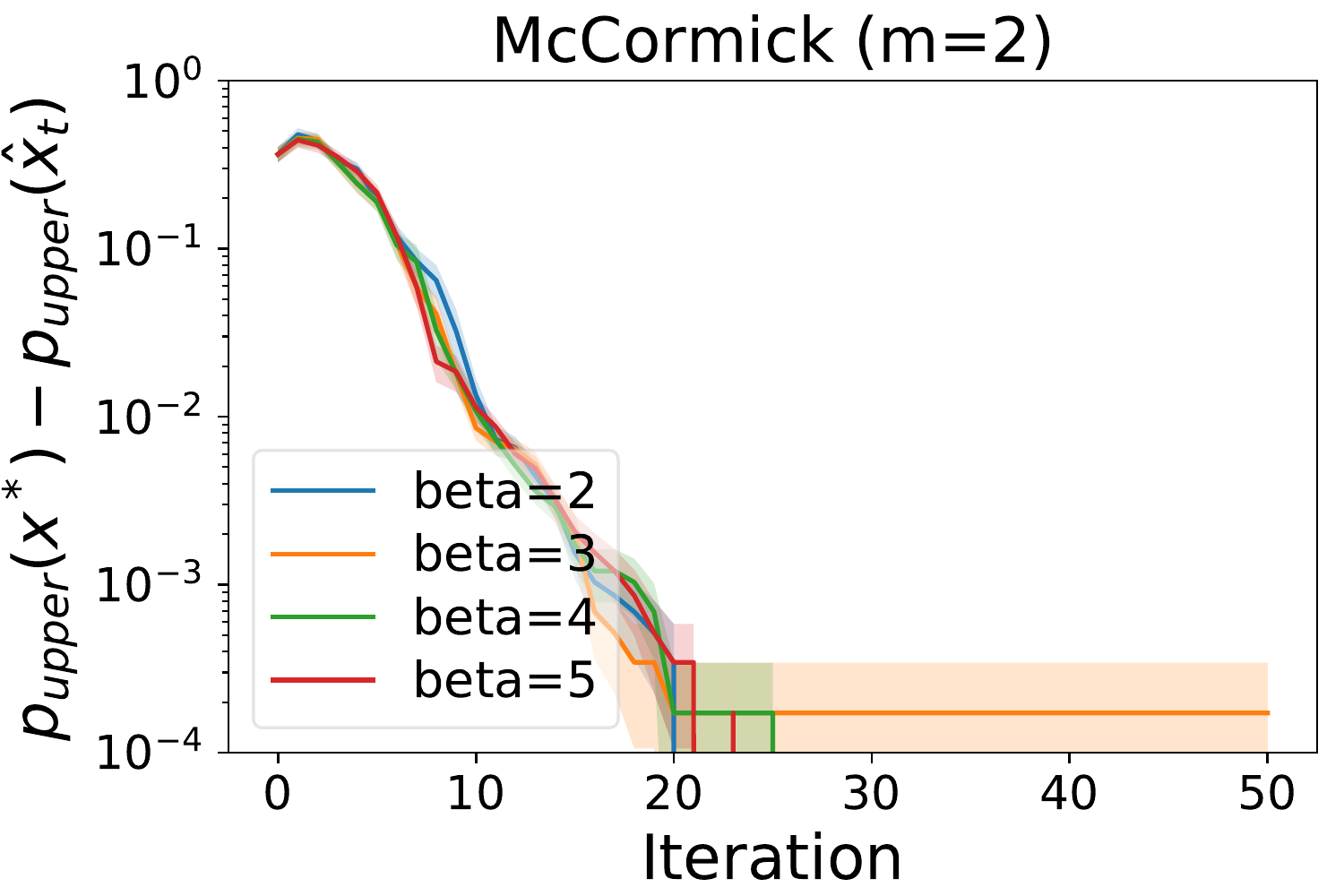} \\
         \includegraphics[width=0.33\linewidth]{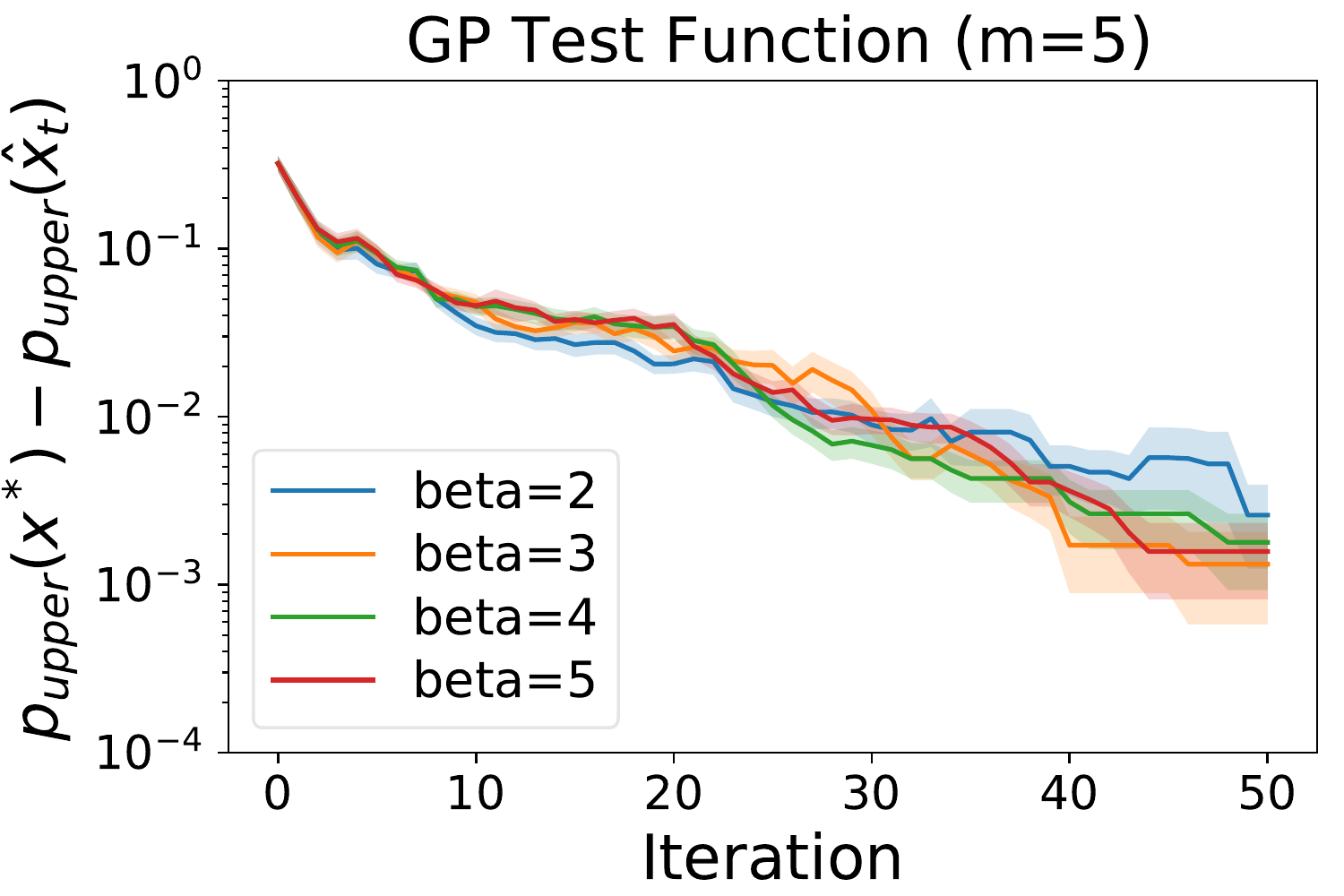} &
         \includegraphics[width=0.33\linewidth]{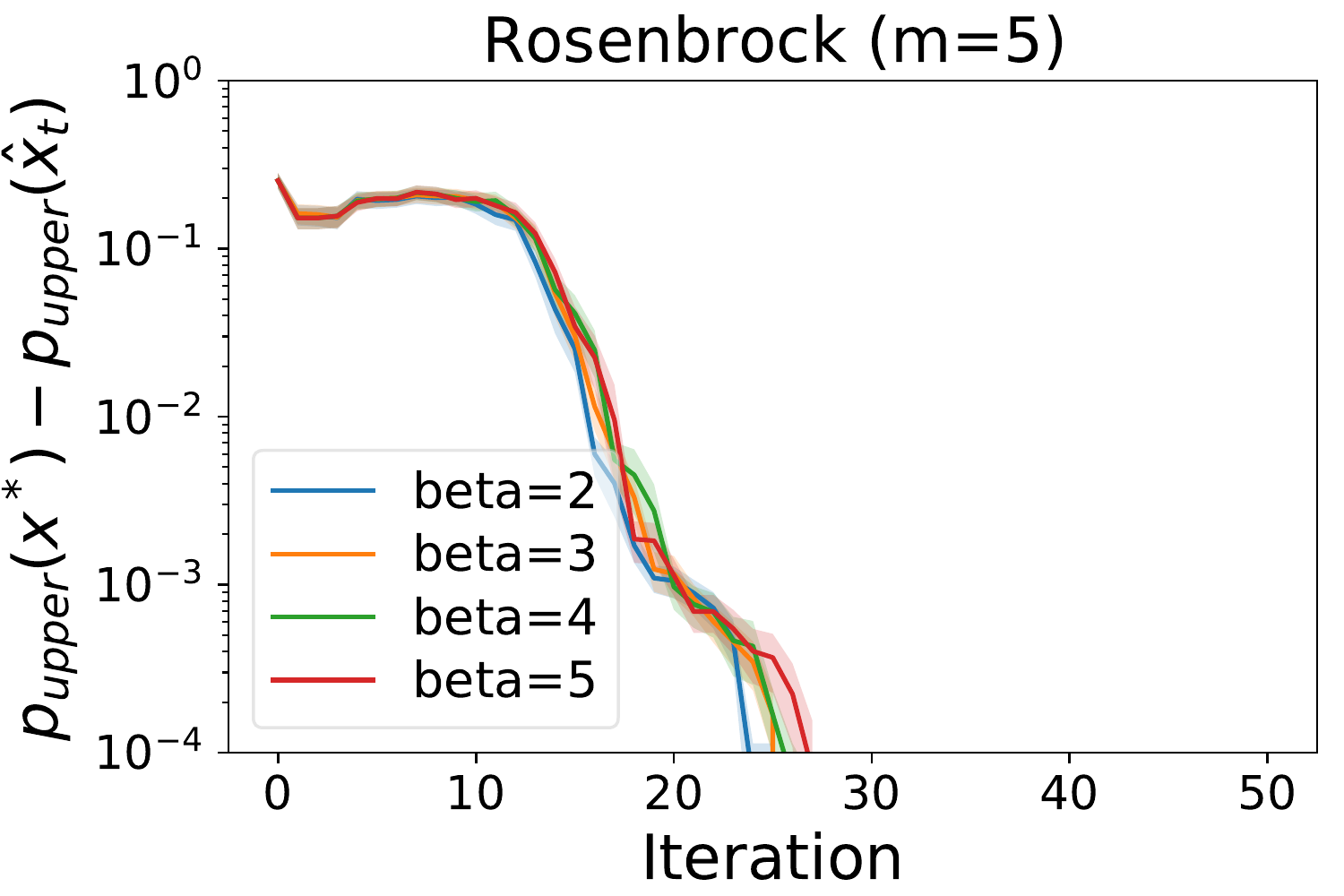} &
         \includegraphics[width=0.33\linewidth]{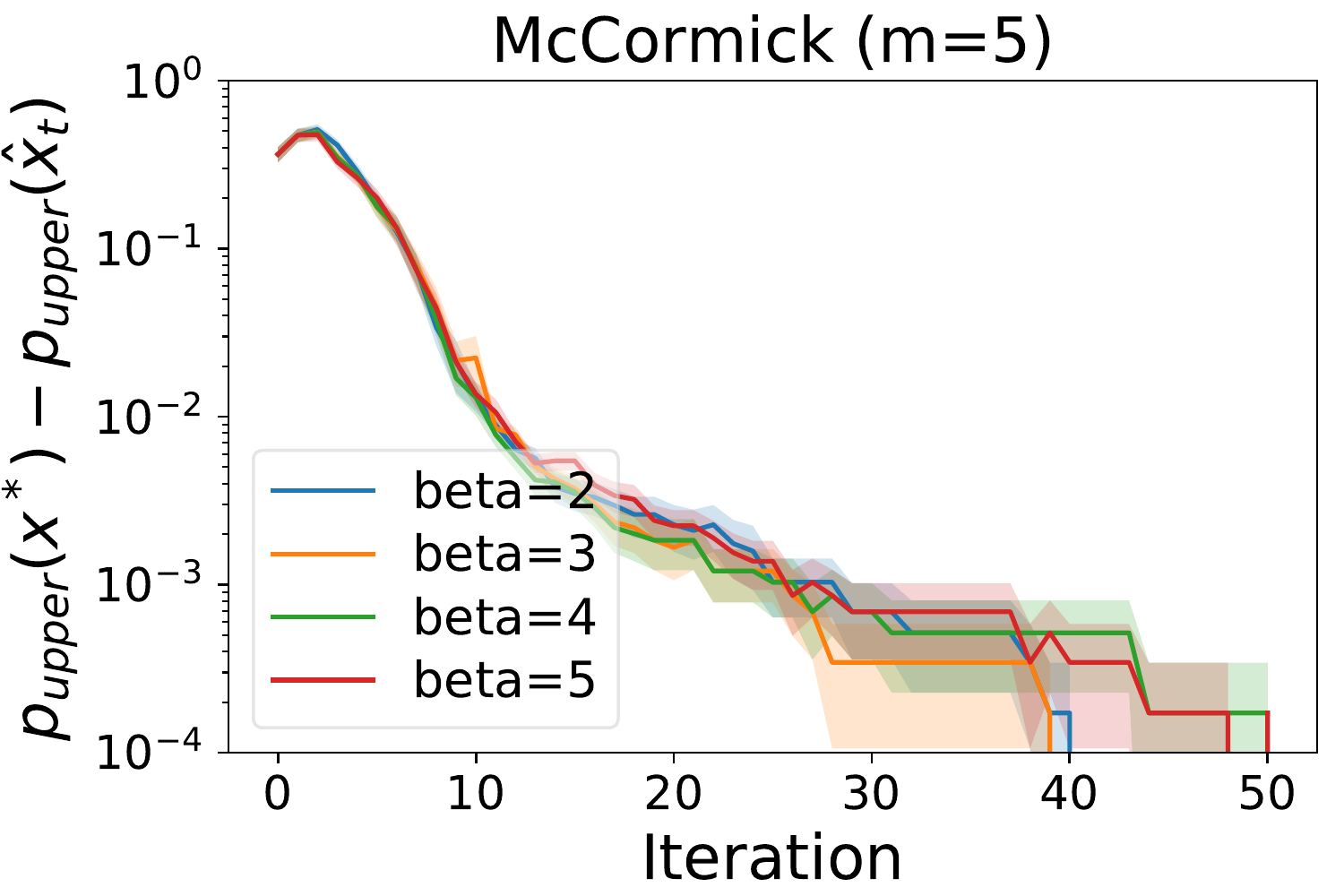} \\
        \end{tabular}
   \end{center}
 \caption{
 The experimental results of BPT-UCB with various $\beta_t$.
 The left, middle and the right plots represent the result on a function generated from GP, 2D-Rosenbrock function and McCormick function, respectively. Additionally, top and bottom plots represent results of
 $m=2$ and $m=5$ respectively.
 These plots show the average performances over $50$ trials.
 }
 \label{fig:bo_effect_beta_app}
\end{figure}

\begin{figure}[t]
    \begin{center}
        \begin{tabular}{ccc}
         \includegraphics[width=0.33\linewidth]{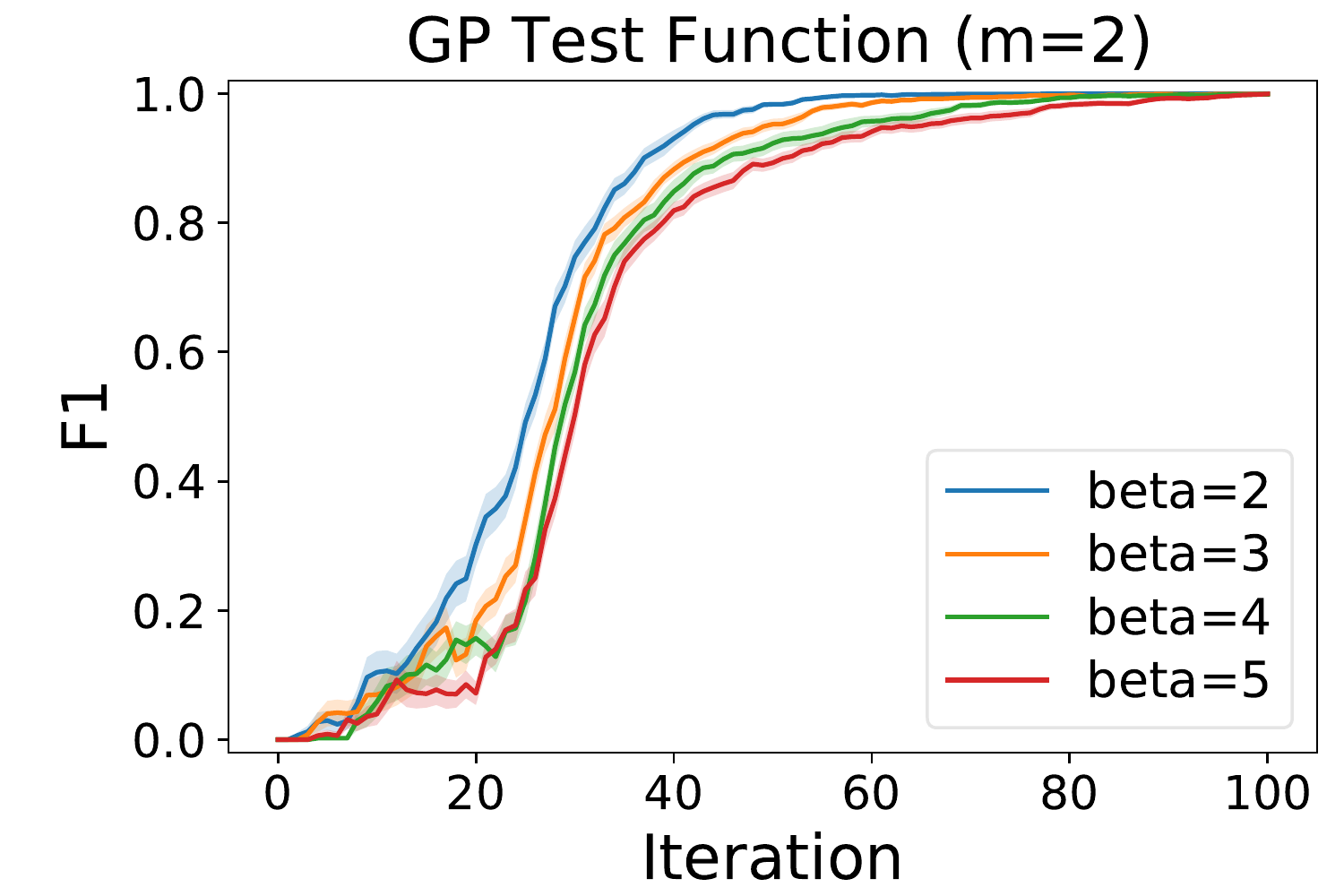} &
         \includegraphics[width=0.33\linewidth]{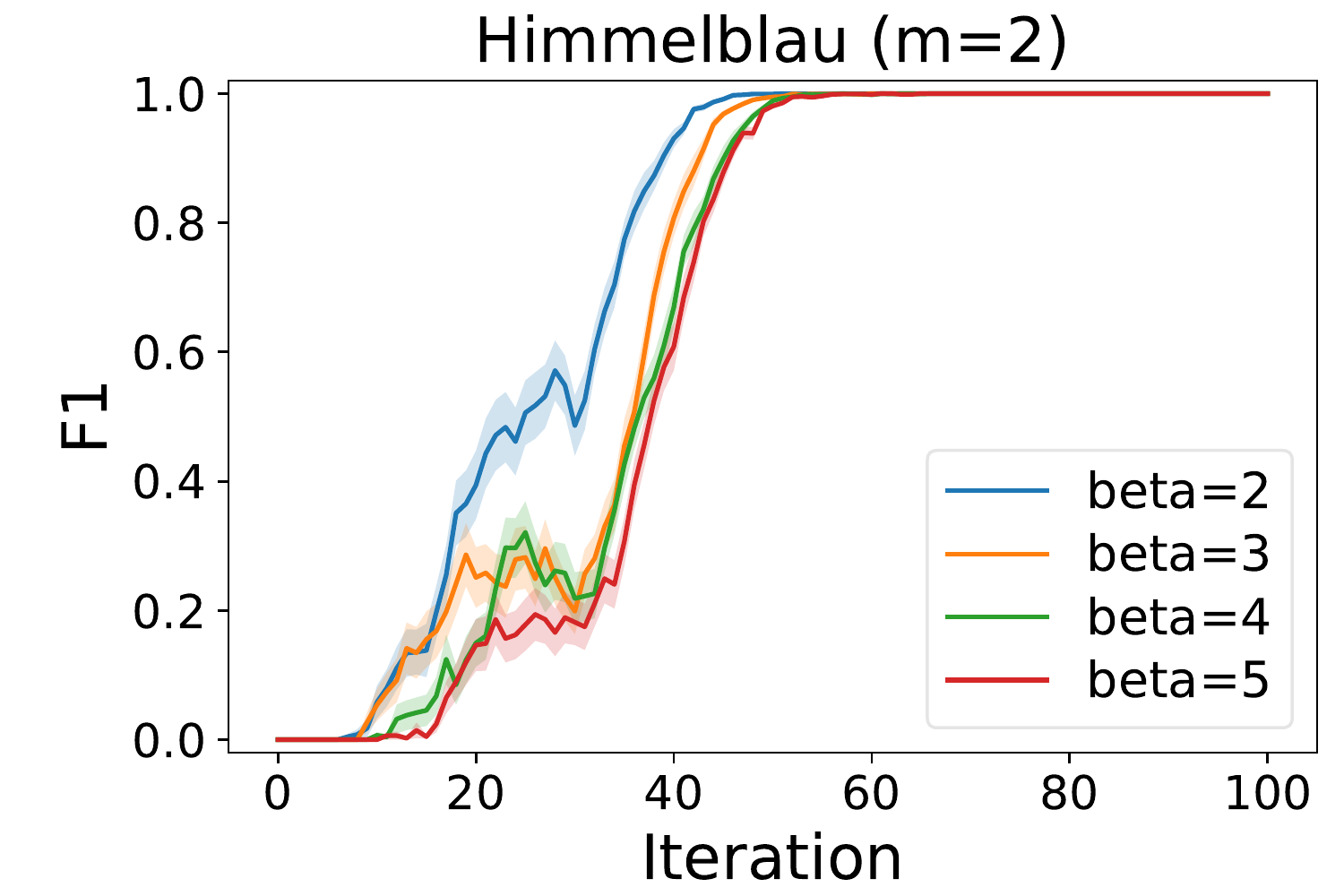} &
         \includegraphics[width=0.33\linewidth]{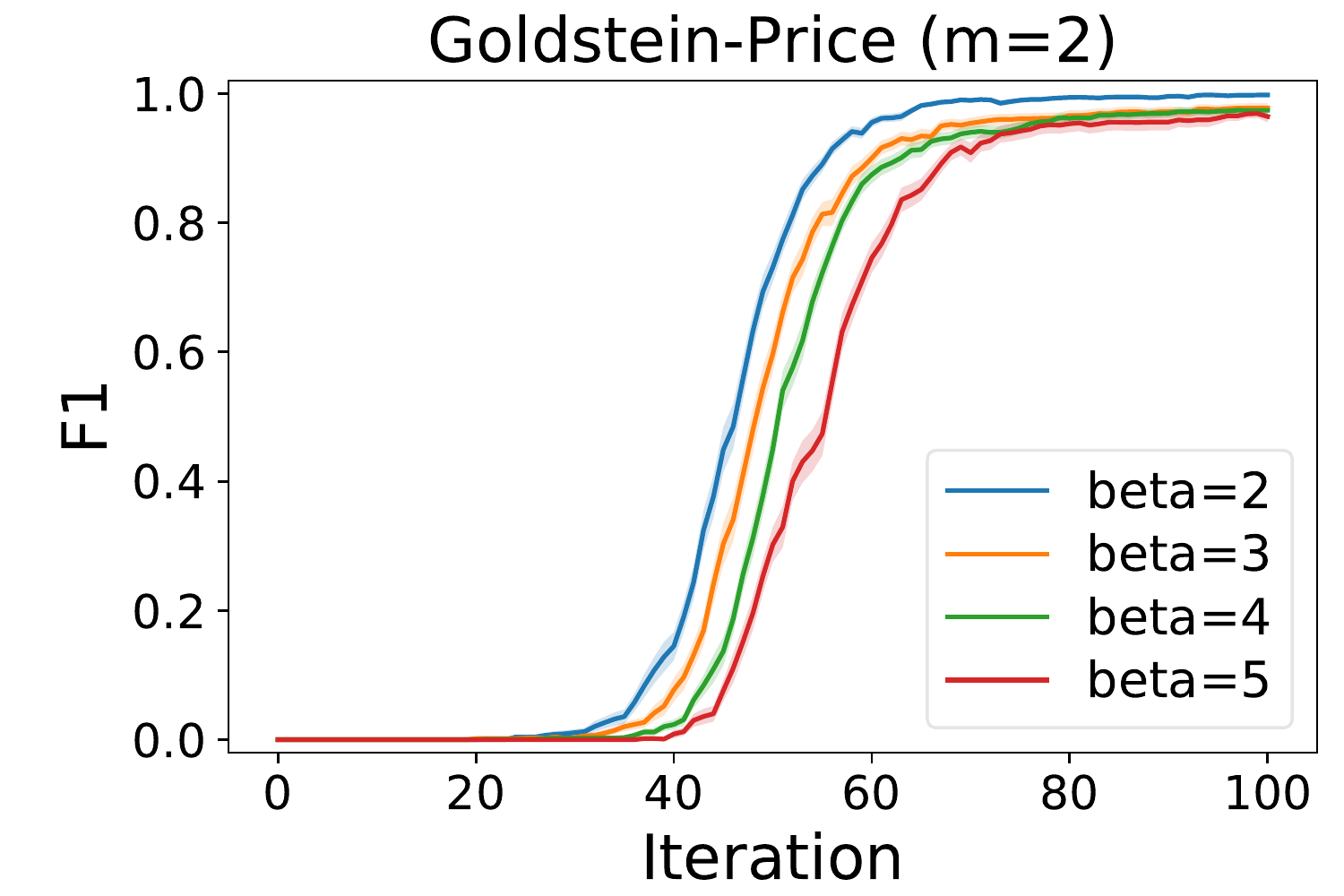} \\
         \includegraphics[width=0.33\linewidth]{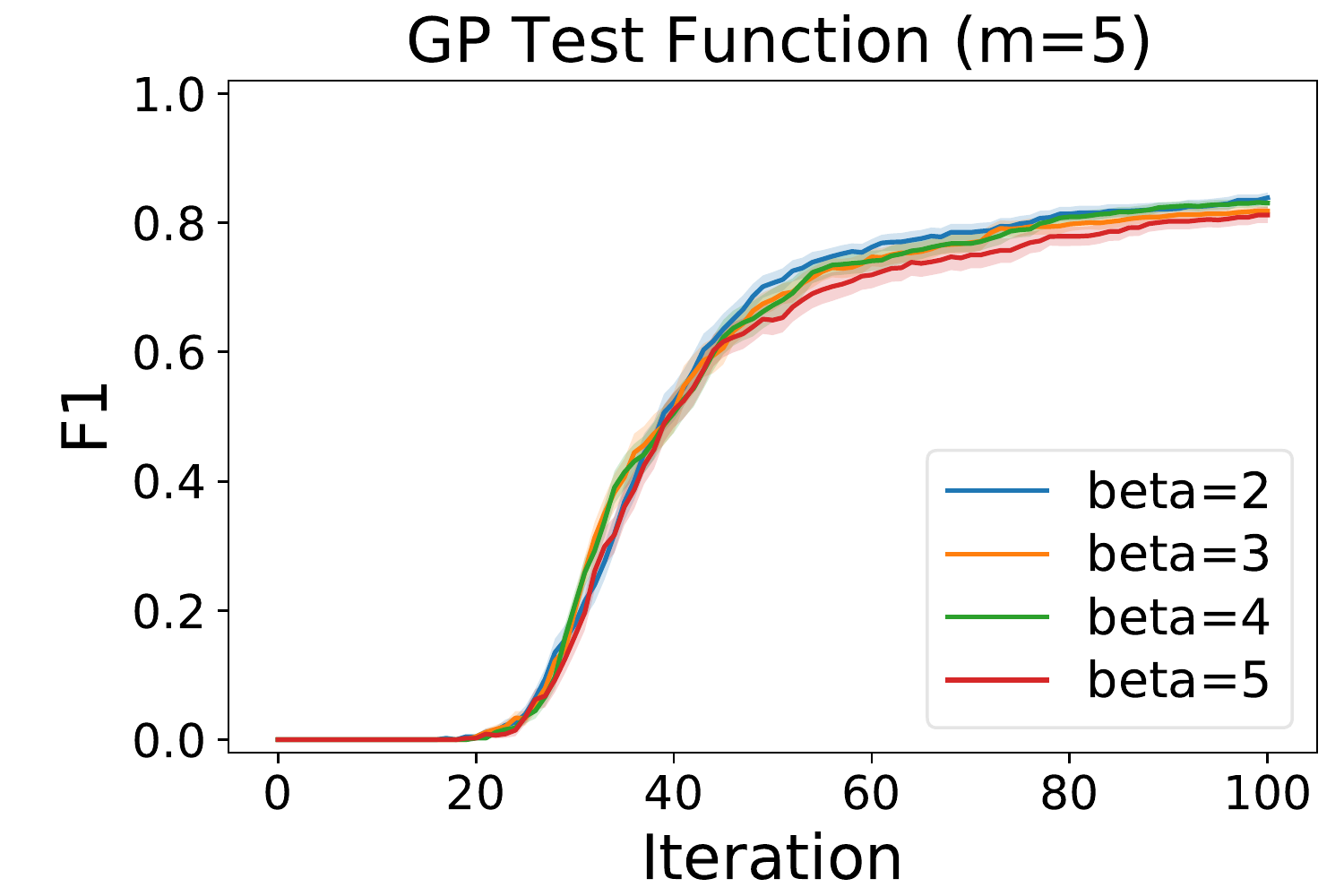} &
         \includegraphics[width=0.33\linewidth]{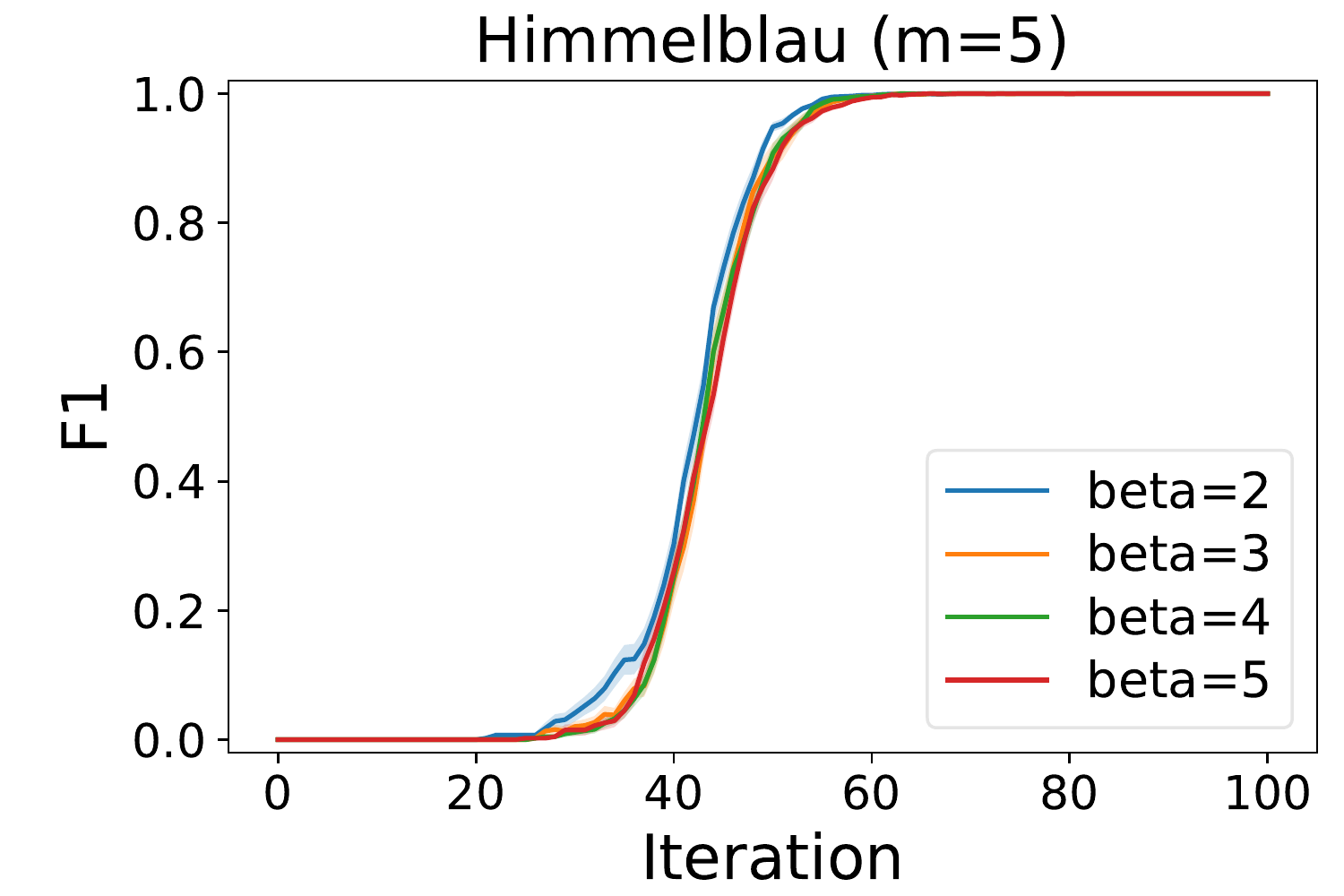} &
         \includegraphics[width=0.33\linewidth]{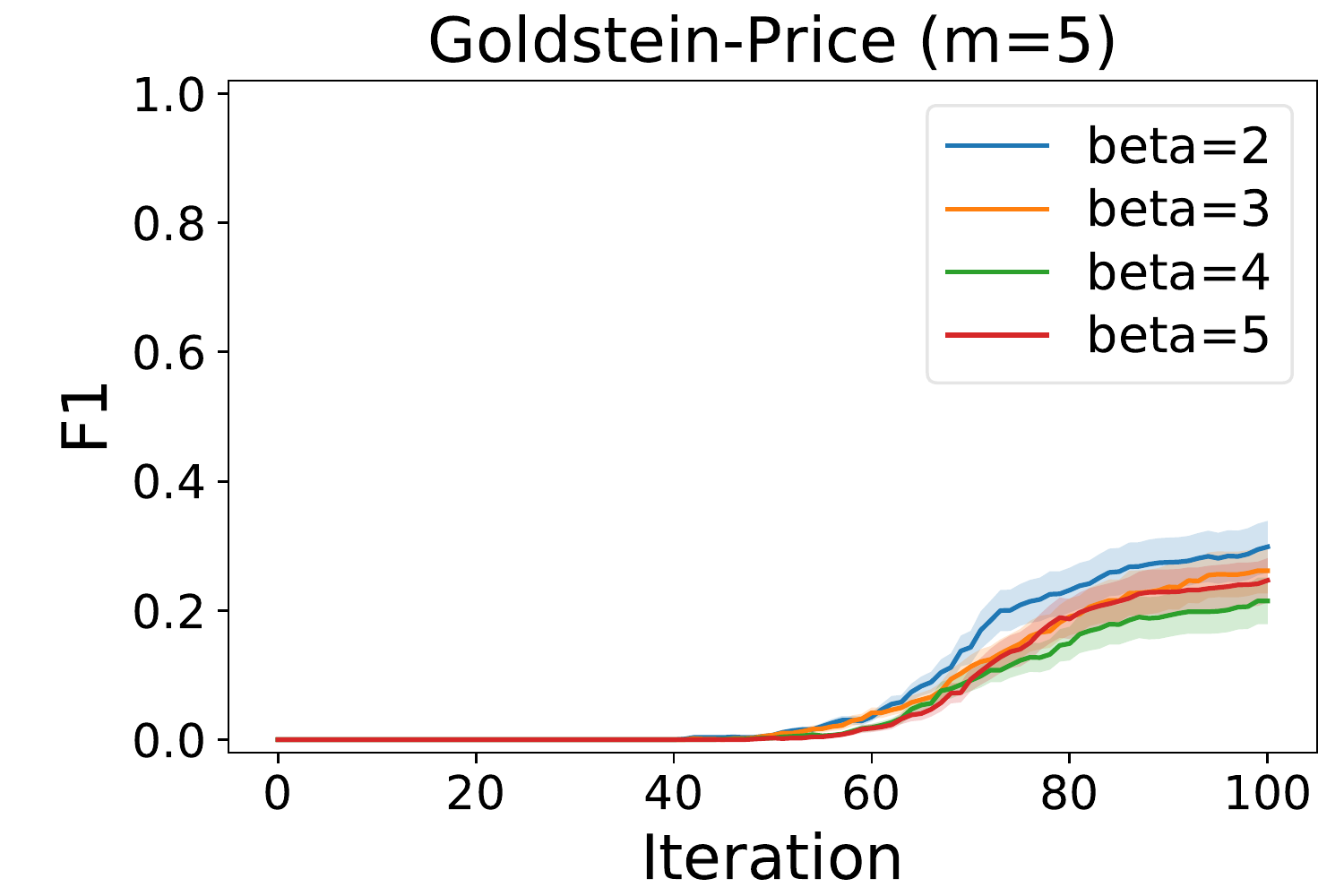} \\
        \end{tabular}
   \end{center}
 \caption{
 The experimental results of BPT-LSE with various $\beta_t$.
 The left, the middle and the right plots represent the result on a function generated from GP, Himmelblau function and Goldstein-Price function, respectively. Additionally, top and bottom plots represent results of
 $m=2$ and $m=5$ respectively.
 These plots show the average performances over $50$ trials.
 }
 \label{fig:lse_effect_beta_app}
\end{figure}

\begin{figure}[t]
    \begin{center}
        \begin{tabular}{ccc}
         \includegraphics[width=0.33\linewidth]{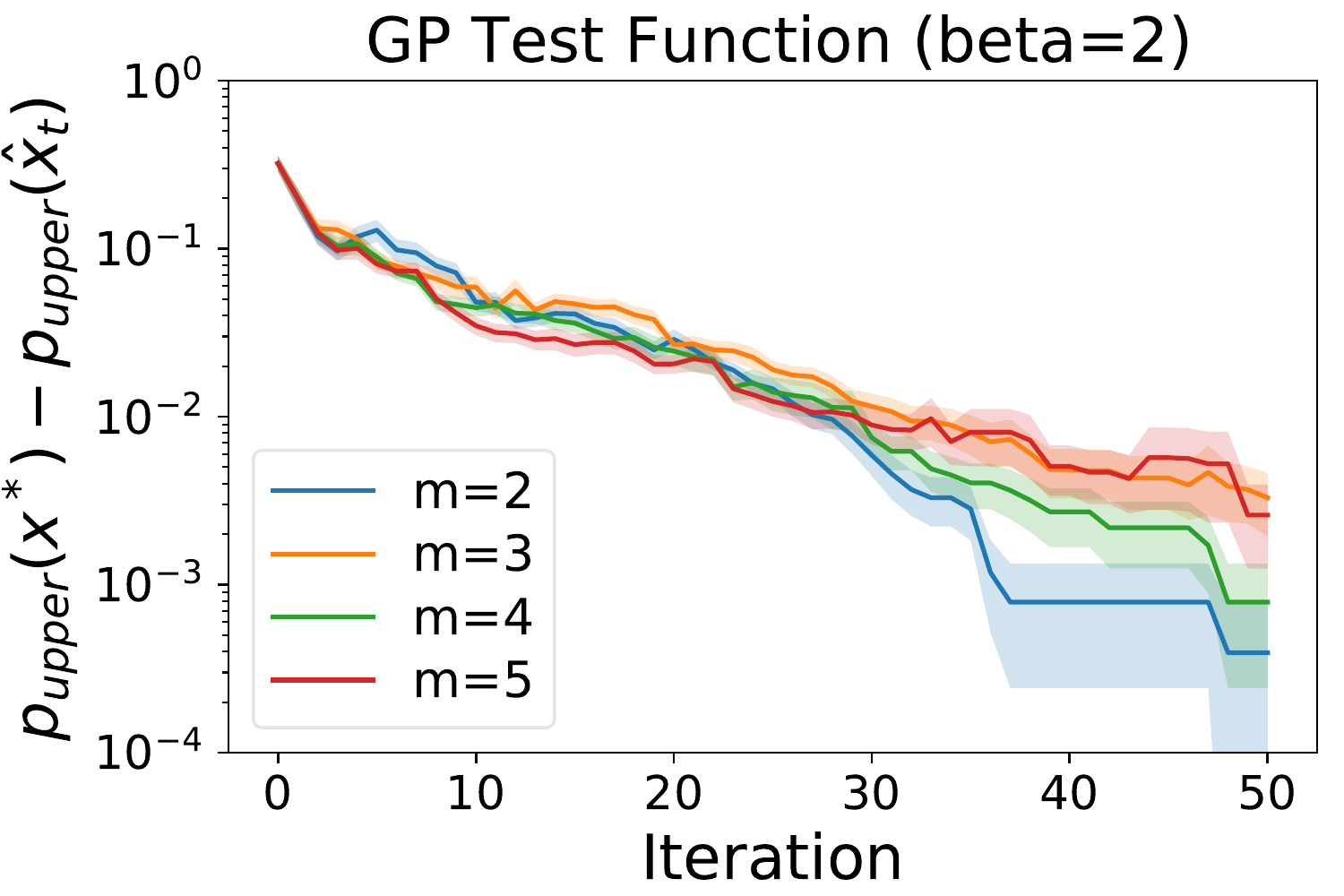} &
         \includegraphics[width=0.33\linewidth]{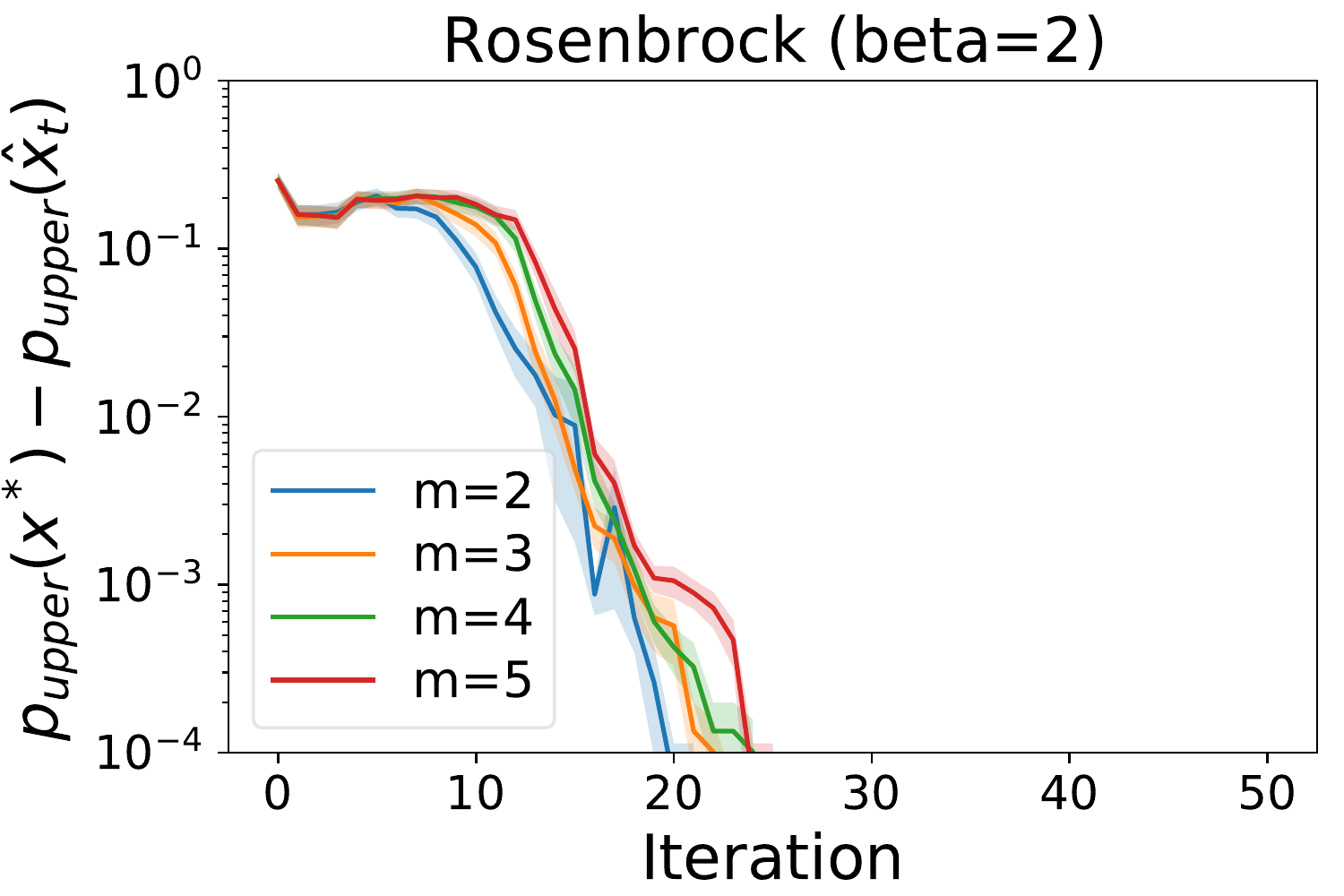} &
         \includegraphics[width=0.33\linewidth]{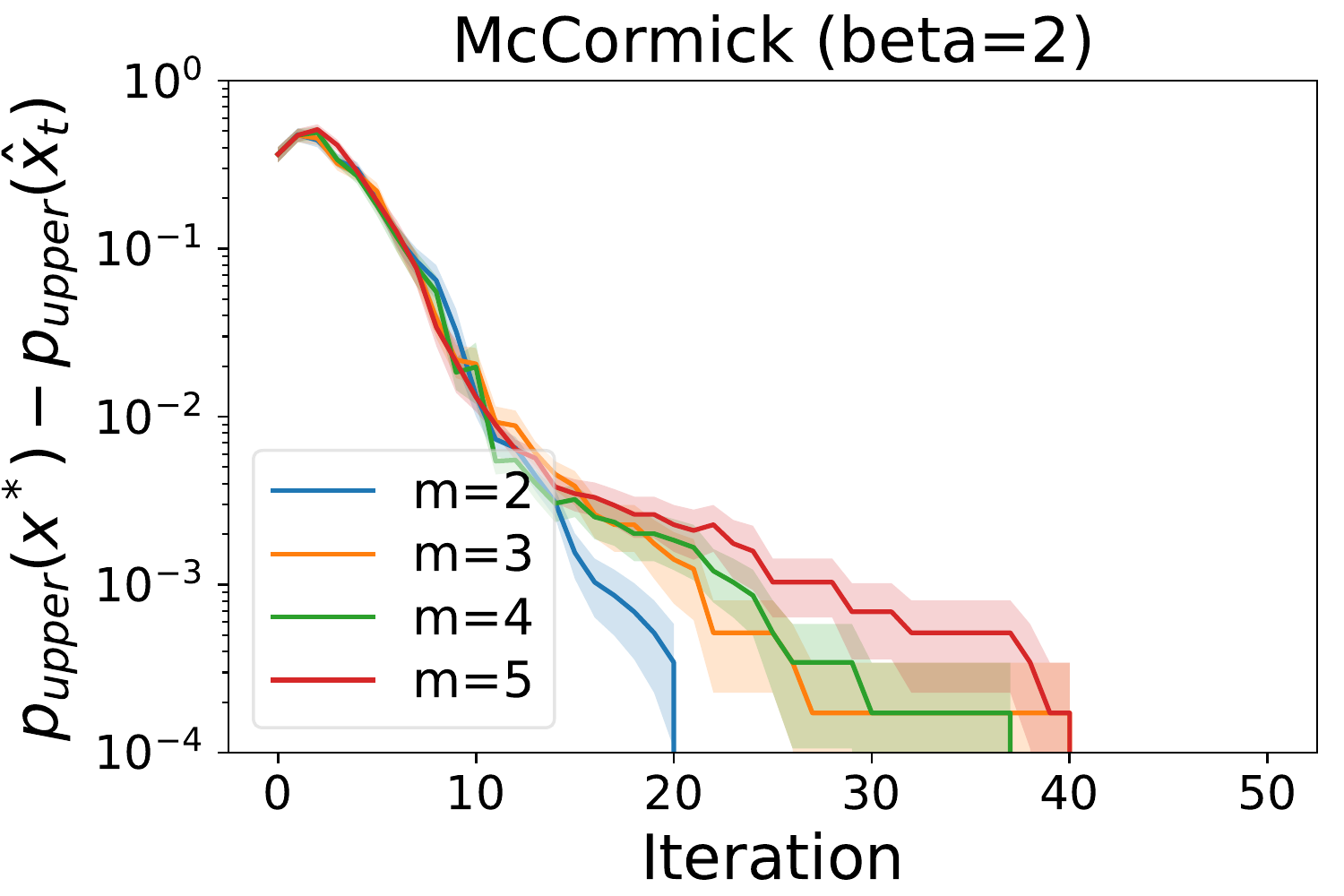} \\
         \includegraphics[width=0.33\linewidth]{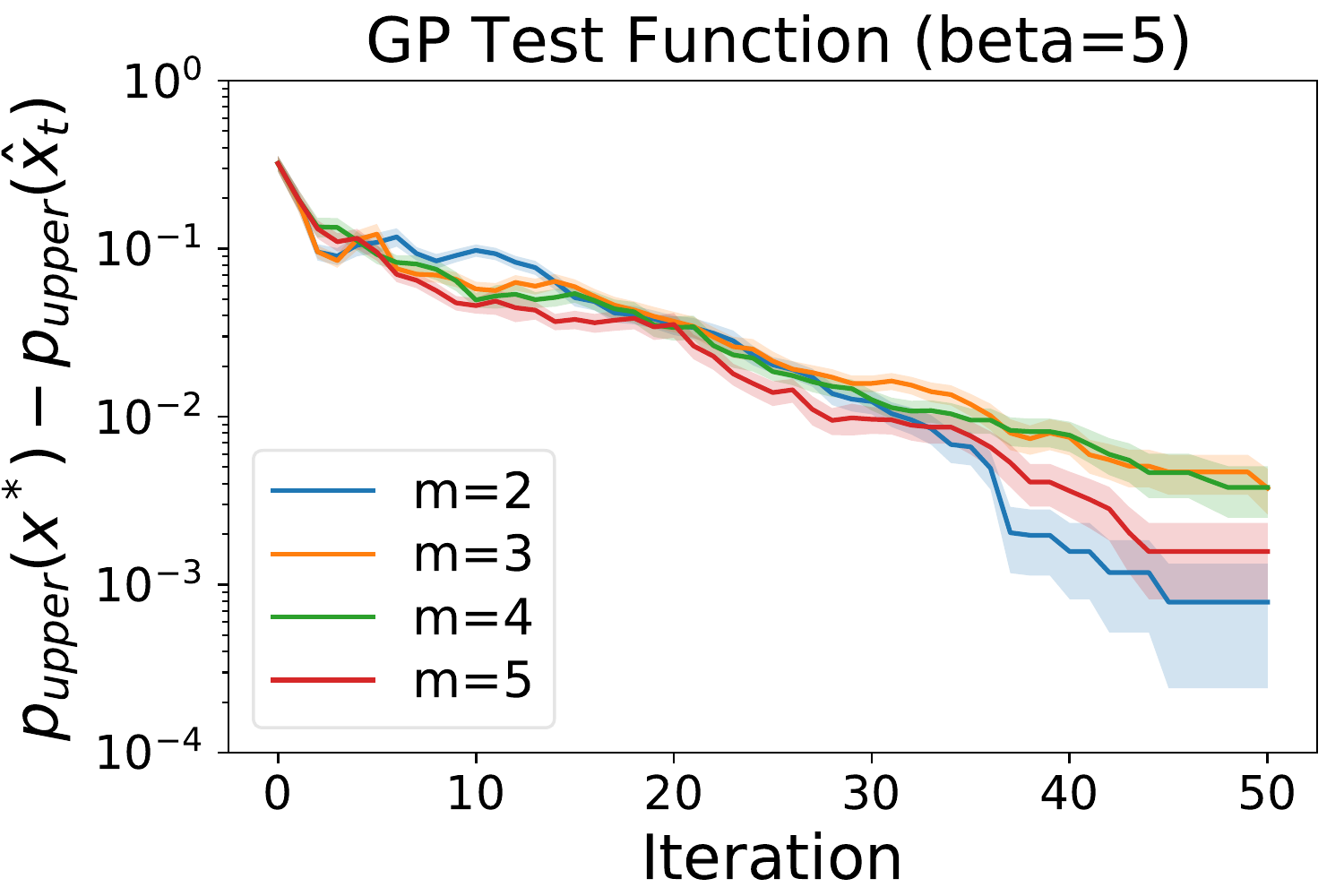} &
         \includegraphics[width=0.33\linewidth]{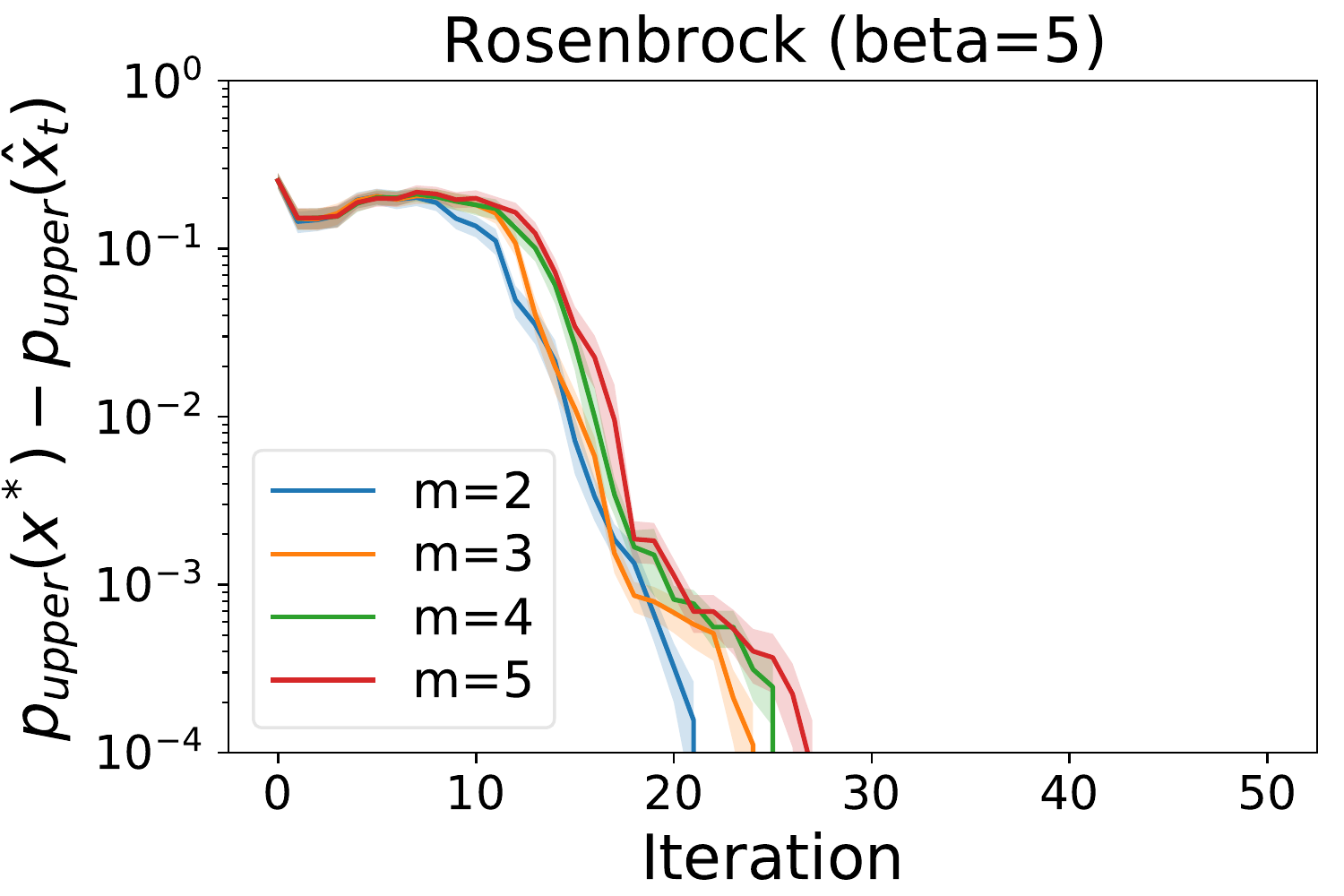} &
         \includegraphics[width=0.33\linewidth]{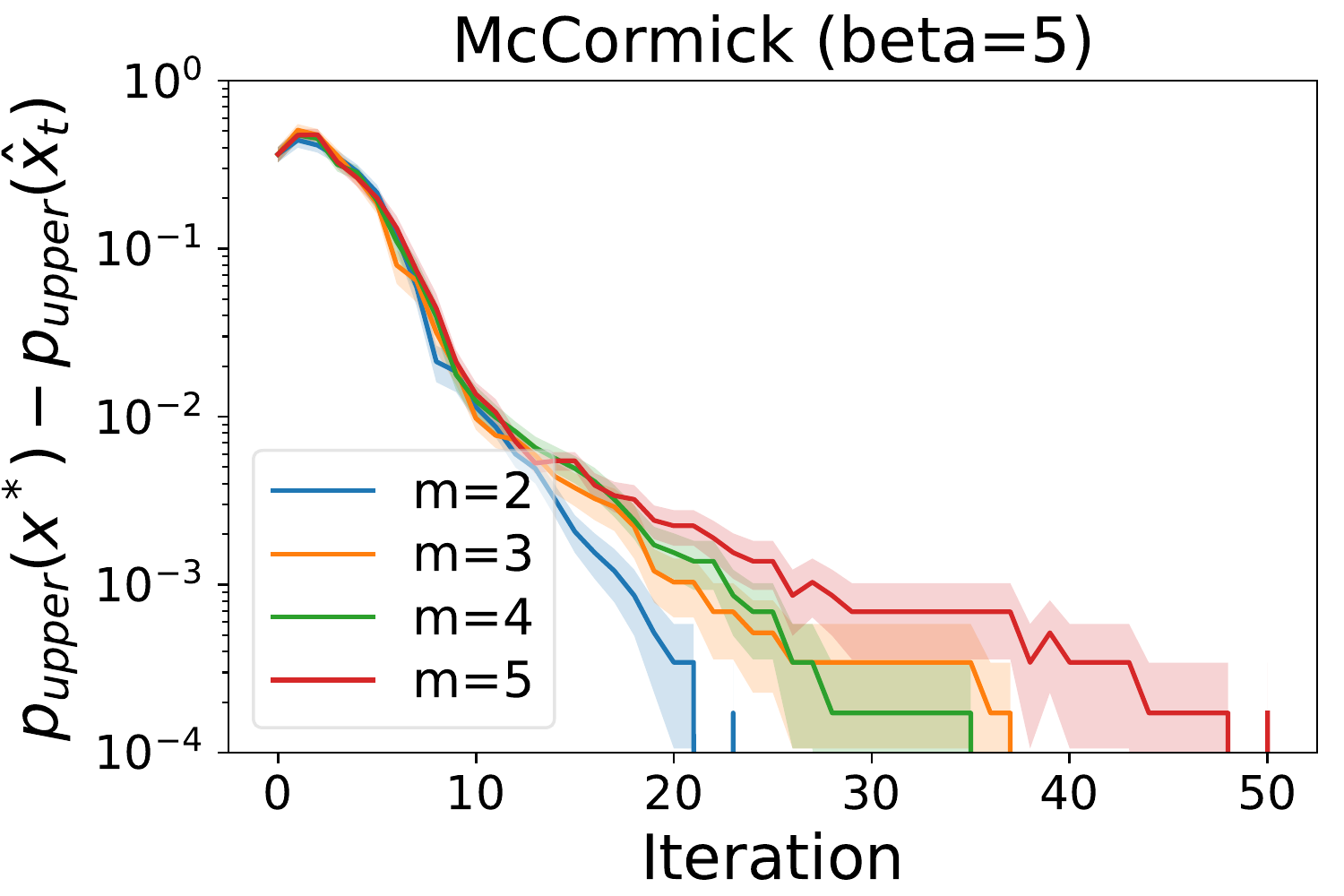} \\
        \end{tabular}
   \end{center}
 \caption{
 The experimental results of BPT-UCB with various $\beta_t$.
 The left, the middle and the right plots represent the result on a function generated from GP, 2D-Rosenbrock function and McCormick function, respectively. Additionally, the top and the bottom plots represent results of
 $\beta_t=2$ and $\beta_t=5$ respectively.
 These plots show the average performances over $50$ trials.
 }
 \label{fig:bo_effect_m_app}
\end{figure}

\begin{figure}[t]
    \begin{center}
        \begin{tabular}{ccc}
         \includegraphics[width=0.33\linewidth]{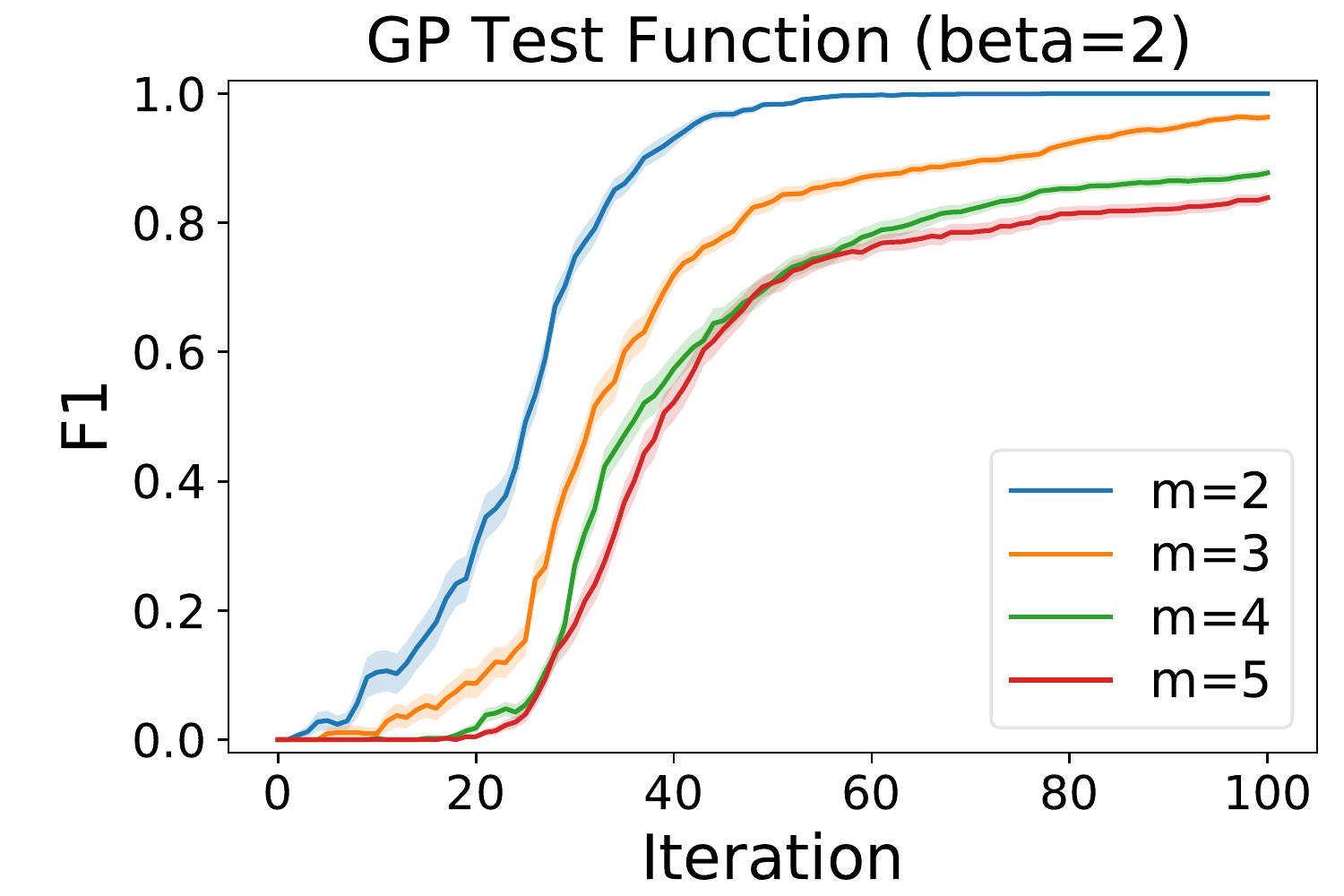} &
         \includegraphics[width=0.33\linewidth]{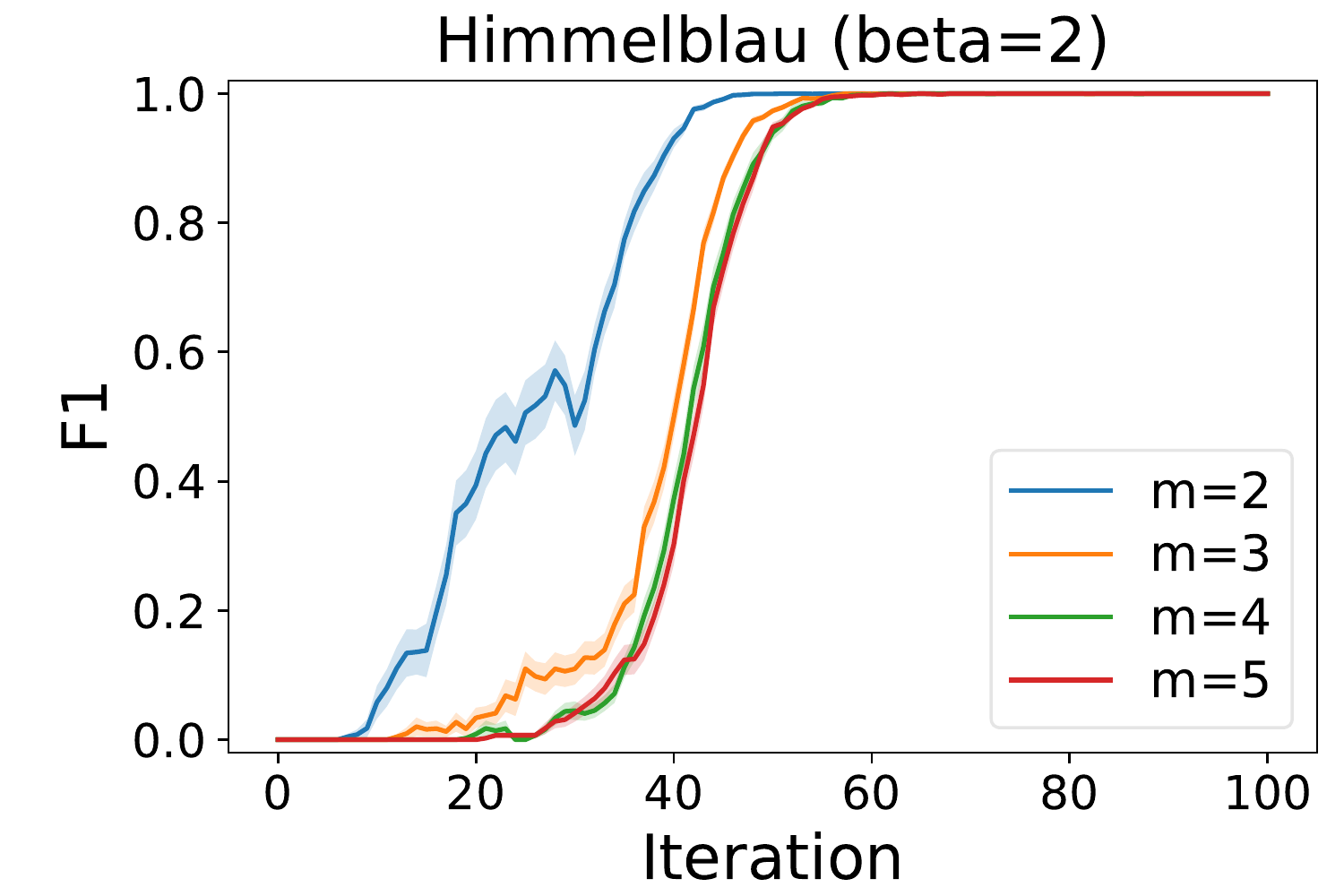} &
         \includegraphics[width=0.33\linewidth]{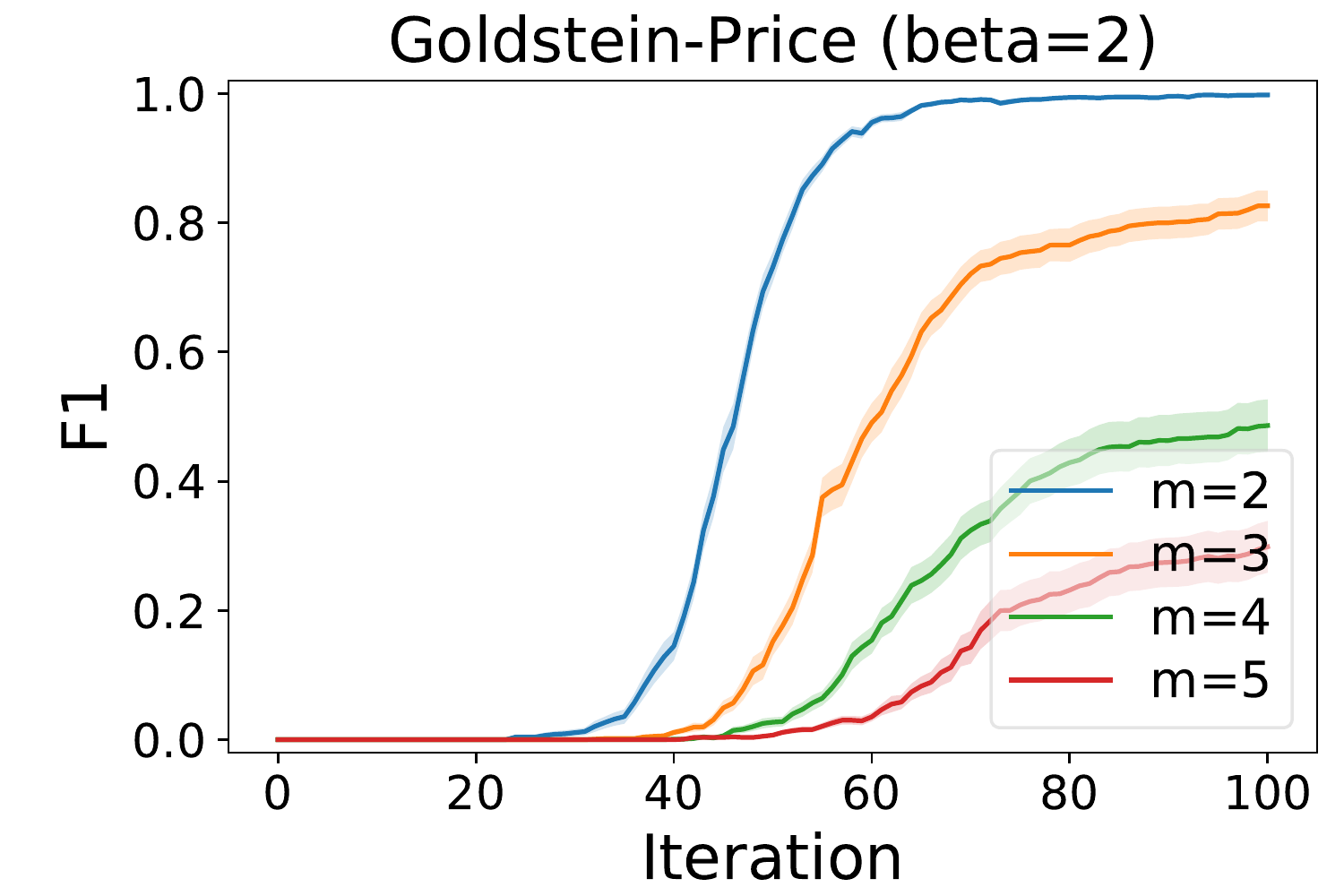} \\
         \includegraphics[width=0.33\linewidth]{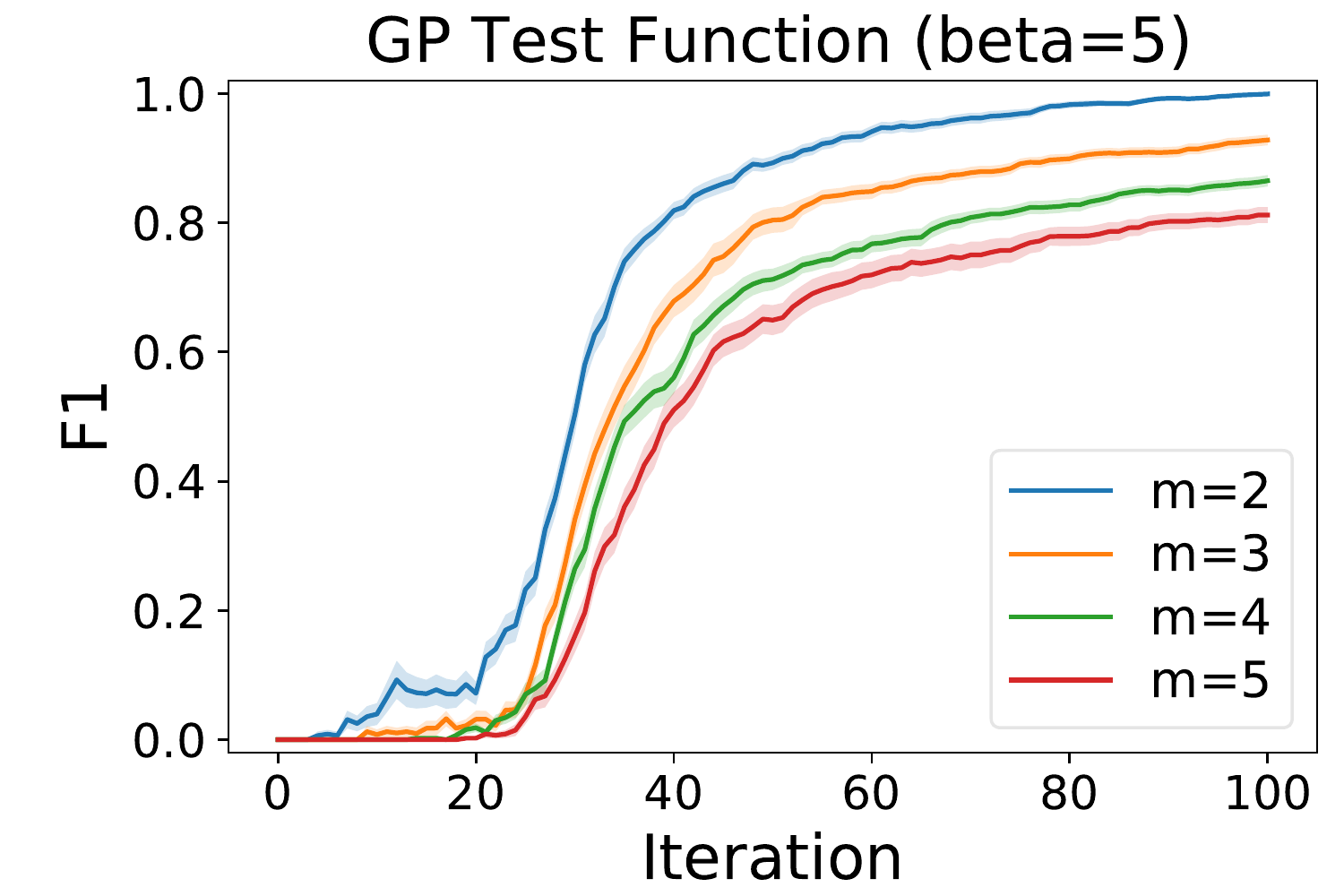} &
         \includegraphics[width=0.33\linewidth]{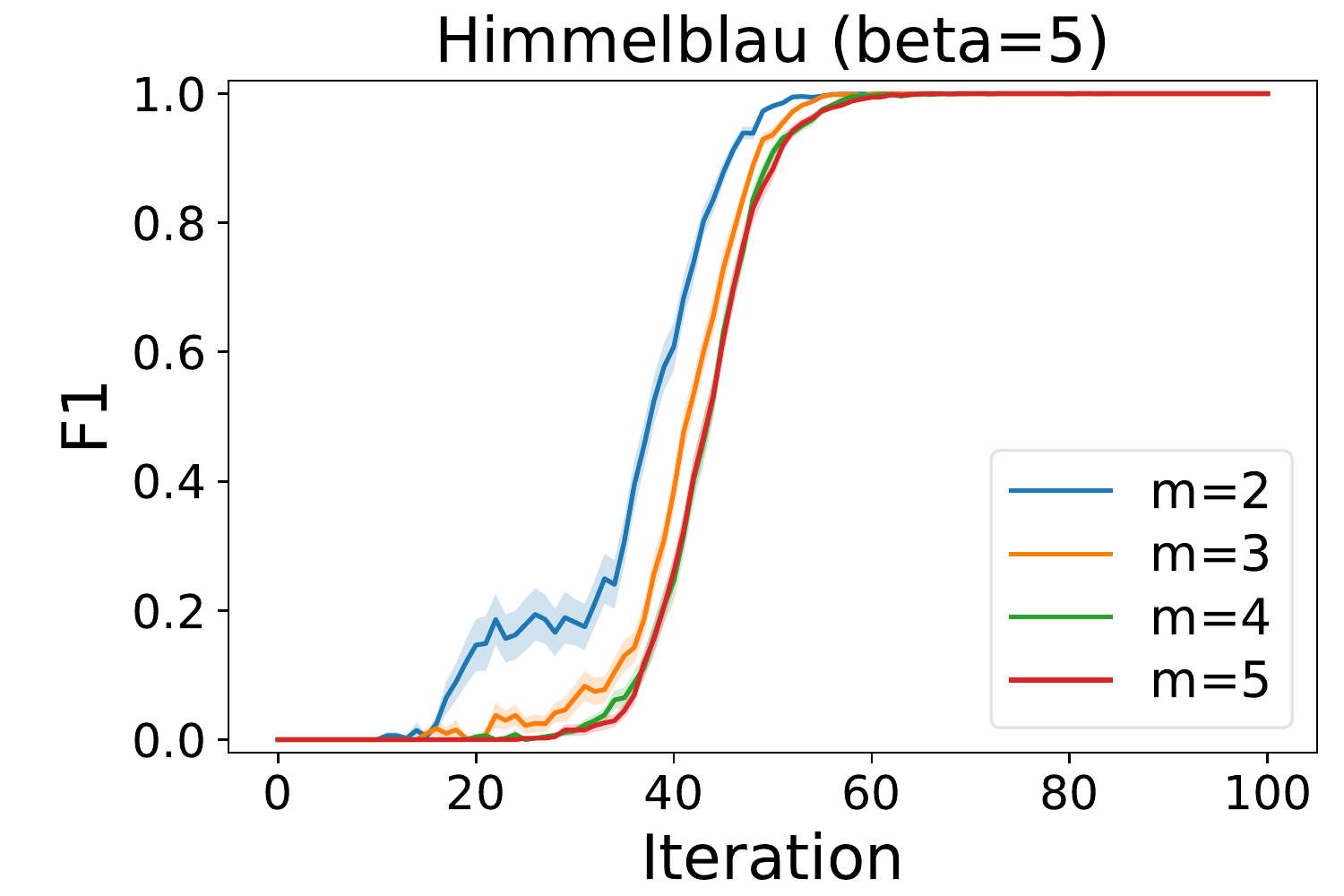} &
         \includegraphics[width=0.33\linewidth]{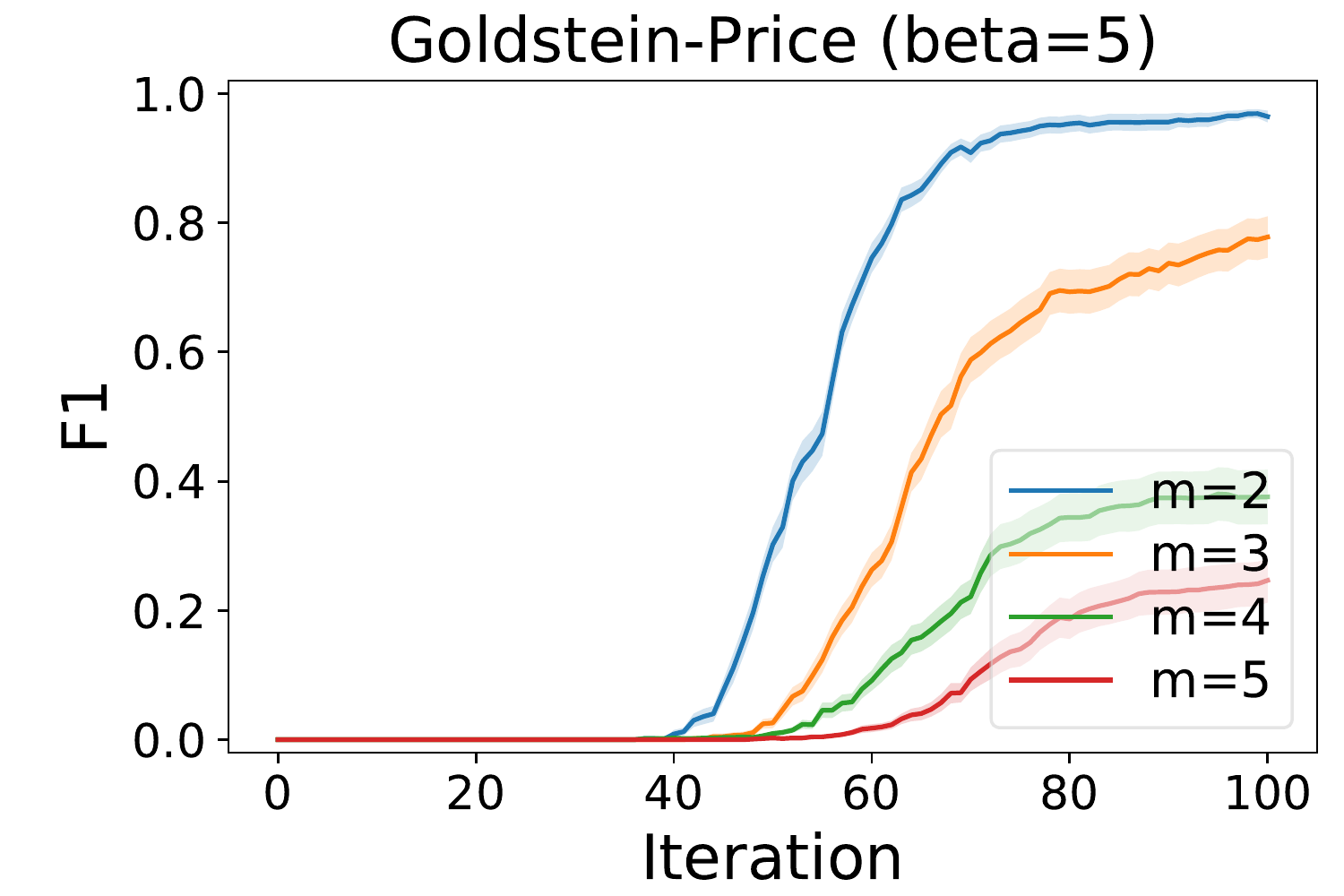} \\
        \end{tabular}
   \end{center}
 \caption{
 The experimental results of BPT-LSE with various $m$.
 The left, middle and the right plots represent the result on a function generated from GP, Himmelblau function and Goldstein-Price function, respectively. Additionally, the top and the bottom plots represent results of
 $\beta_t=2$ and $\beta_t=5$ respectively.
 These plots show the average performances over $50$ trials.
 }
 \label{fig:lse_effect_m_app}
\end{figure}

\subsection{Real Data Experiments}
We tested the performances of the proposed methods on two real examples for the optimization setting and one example in the LSE setting.

\subsubsection{Infection Control Problem}
We considered a decision making problem
on epidemic simulation model used in \cite{DBLP:conf/uai/GessnerGM19}
with a slight modification.
In this problem,
the goal is to decide the target infection rate to minimize the associated economic risk
with as small number of simulation runs as possible.
For instance, if we decide to make all the economic activity stop, the lowest infection rate would be archived but the economic risk would be extremely large.
On the other hand, if we do not take any action to control the infection, the infection rate stays high and the economic risk would be also non-negligibly high due to the spread of infection.
Hence,  we want to find the target infection rate that archives low risk on tolerance level $h$ with the highest probability (or with sufficiently high probability in LSE)
In our experiments, we used SIR model \cite{kermack1927contribution} as the epidemic simulation model.
This model simulates the transition of the number of infected people given two parameters called infection rate and recovery rate.
Here we regarded the infection rate as the design parameter $x$ and the recovery rate as the environmental parameter $w$ because the uncertainty of the latter is uncontrollable in reality.
We assumed shifted gamma prior: $c/w - 1 \sim \text{Gam}(a, b)$, where $c=0.5, a=5, b=4$ as in \cite{DBLP:conf/uai/GessnerGM19}, and then define $p(w)$.
We then rescaled the domain of $x$ and $w$ to $[-1, 1]$, and considered evenly allocated $50$ grid points in each dimension.
We assumed the following risk function as $f$:
\begin{align*}
 f(x, w) = n_{\text{infected}}(x, w) - 150x,
\end{align*}
where
$n_{\text{infected}}(x, w)$,
which is computed via SIR model simulation,
is the maximum number of infected people within a certain period.
Furthermore,
we set
$h=135$
and
$\alpha=0.9$,
and for GP modeling,
we used
Gaussian kernel with
$l=0.5,~\sigma_{\text{ker}}=250$,
and
$\sigma^2=0.025$.
Additionally,
we used the same settings as \S\ref{sec:art_data}
for other parameters.

The results are shown in Fig.~\ref{fig:sir_result_app}.
We confirm that the proposed methods worked well in both the optimization and the LSE settings.

\begin{figure}[t]
    \centering
     \begin{tabular}{cc}
         \includegraphics[width=0.5\linewidth]{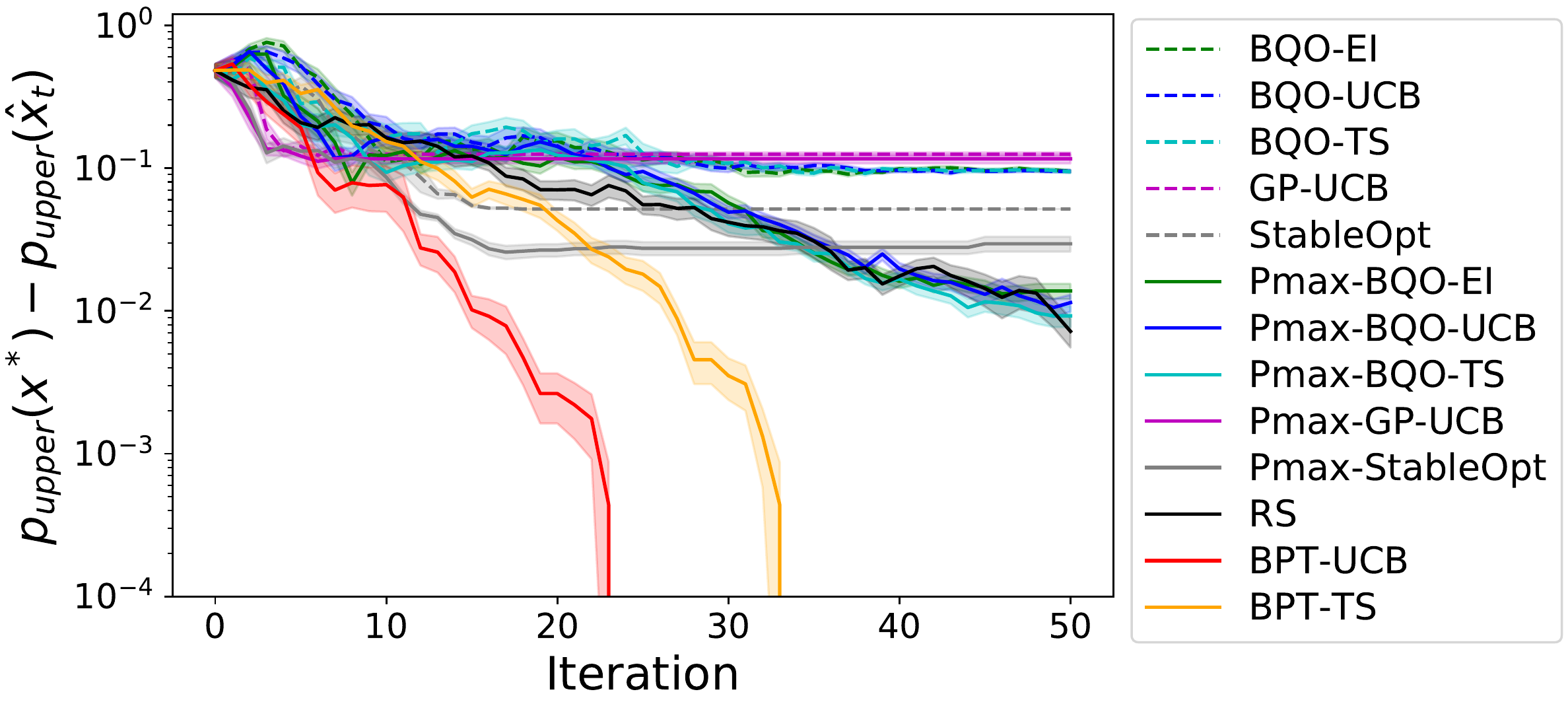} &
         \includegraphics[width=0.5\linewidth]{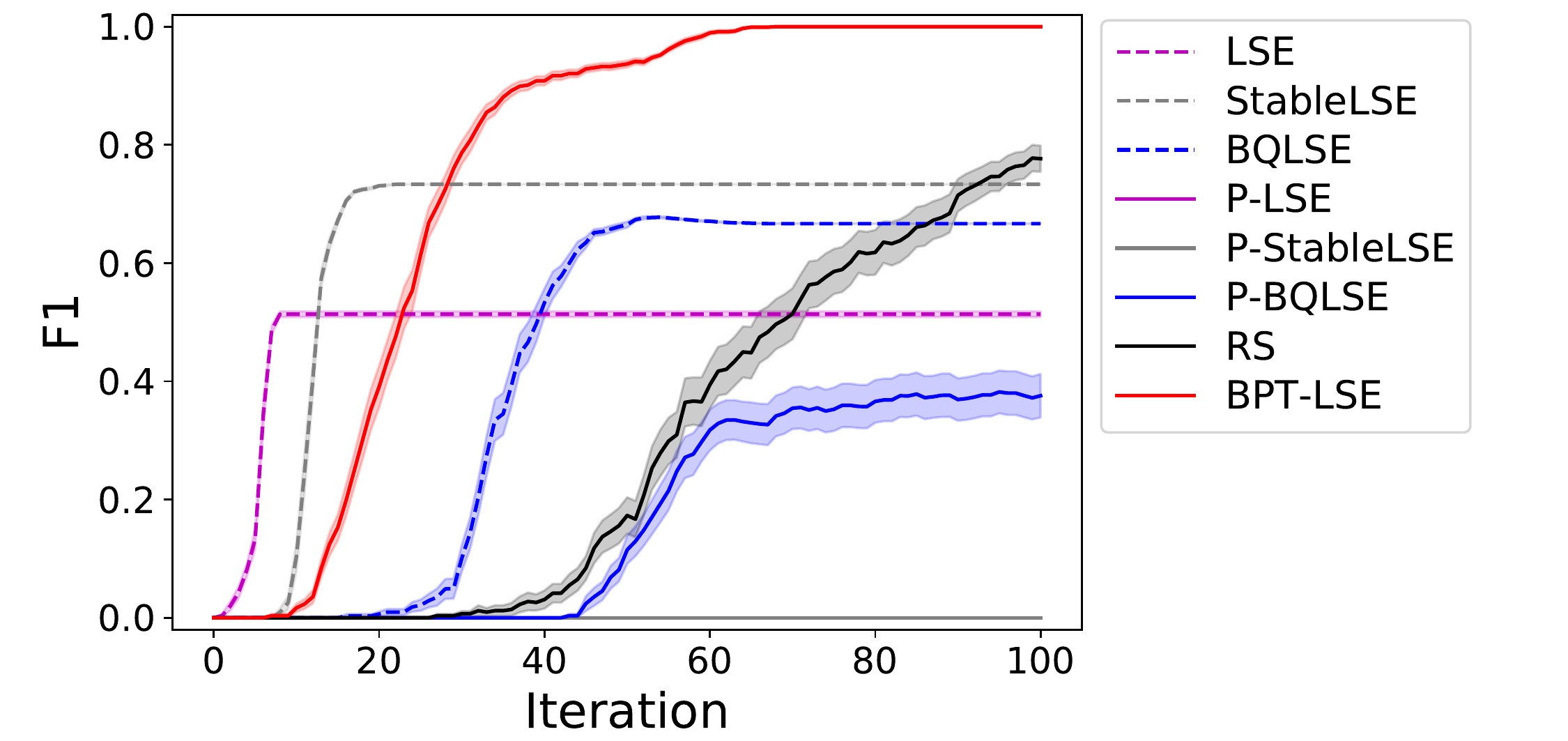} \\
      (a) Optimization setting &
      (b) LSE setting
     \end{tabular}

 \caption{
 The experimental results on the infection control problem.
 The left and the right plots represent the results in the optimization and the LSE settings, respectively.
 These plots show the average performances over $50$ trials.
 }
 \label{fig:sir_result_app}
\end{figure}

\subsection{Newsvendor Problem under Dynamic Consumer Substitution}
We applied the proposed methods
to Newsvendor Problem
under Dynamic Consumer Substitution \cite{mahajan2001stocking}.
This problem was also studied in \cite{toscano2018bayesian}.
The goal of this problem is to find the optimal initial inventory level to maximize profit, which is computed by a stochastic simulation, with as small number of simulation runs as possible.

In this problem, each product $j$ has the cost $c_j$ and $p_j$, and the initial inventory level is noted as $x_j$.
In a simulation, a sequence of $I$ customers indexed by $i$ arrives in order and decide whether they buy an in-stock product or not.
These decisions are made based on the utility $U_i^j$, which is assigned for the customer $i$ and the product $j$.
Utilities are modeled with the multi-nominal logit model,
where
$U_i^j = u^j + \xi_i^j$
and
$u_j$
are constant.
Here,
\{$\xi_i^j\}$
follows mutually independent Gumbel distributions,
whose distribution function is written as
$\Psi_i^j(z) \coloneqq P(\xi_i^j \leq z) = \exp(-e^{-(z/\mu + \eta)})$,
where $\eta$ is Euler's constant.
Furthermore,
let
$w_j$
be
$\sum_{i=1}^I \Upsilon^{-1}\left( \Psi_i^j(\xi_i^j) \right)$
where
$\Upsilon$
is the cumulative distribution function of Gamma distribution,
and
$w_j$
follows mutually independent Gamma distribution.
Additionally,
$\{\xi_i^j\}$
can be simulated given
$\{w_j\}$
(see more details at \S6.6 in \cite{toscano2018bayesian}).
In the end of the simulation,
the profit is computed as the sum of the prices of the products sold minus the cost of the initial inventory.
We defined the function
$f(\bm{x}, \bm{w})$
as the conditional expectation of the profit given initial inventory $\bm{x}$ and $\bm{w}$ described above.

In our experiment,
we considered two products whose costs are
$c_1=4, c_2=13$
and
the prices are
$p_1=10, p_2=23$,
respectively,
and chose
$I=50, u^1=1, u^2=1$.
Furthermore,
we set
$\mathcal{X}=[0, I]\times[0, I]$,
and
$\Omega=[w_1^{\text{st}}, w_1^{\text{ed}}]\times[w_2^{\text{st}}, w_2^{\text{ed}}]$,
where
$[w_j^{\text{st}}, w_j^{\text{ed}}]$
is the $99.9$\% confidence interval of
$w_j$.

In this experiment,
since
$\mathcal{X}$
and
$\Omega$
are continuous set,
we use Random Feature Map method \cite{rahimi2008random}
with
$1000$
random features to
approximate posterior sampling of GP in BQO-TS and BPT-TS.
Additionally,
we chose
$h=350$,
and for GP modeling,
we used Matern$5/2$ kernel
$k((\bm{x}, \bm{w}), (\bm{x}^{\prime}, \bm{w}^{\prime})) = \sigma_{\text{ker}}^2(1 + \sqrt{5}r+\frac{5}{3}r^2)\exp(-\sqrt{5}r)$,
$r = \sqrt{\sum_{j=1}^d(\bm{x}_j - \bm{x}^{\prime}_j)^2/l_j^{(x)2} + \sum_{j=1}^k(\bm{w}_j - \bm{w}^{\prime}_j)^2/l_j^{(w)2}}$,
and all the kernel hyper parameters were estimated within algorithms by maximizing marginal likelihood.

The experimental results in the optimization setting is in Fig~\ref{fig:newsvendor_result_app}.
The proposed methods archived better performance than existing methods.

\begin{figure}[t]
 \centering
 \includegraphics[width=0.6\linewidth]{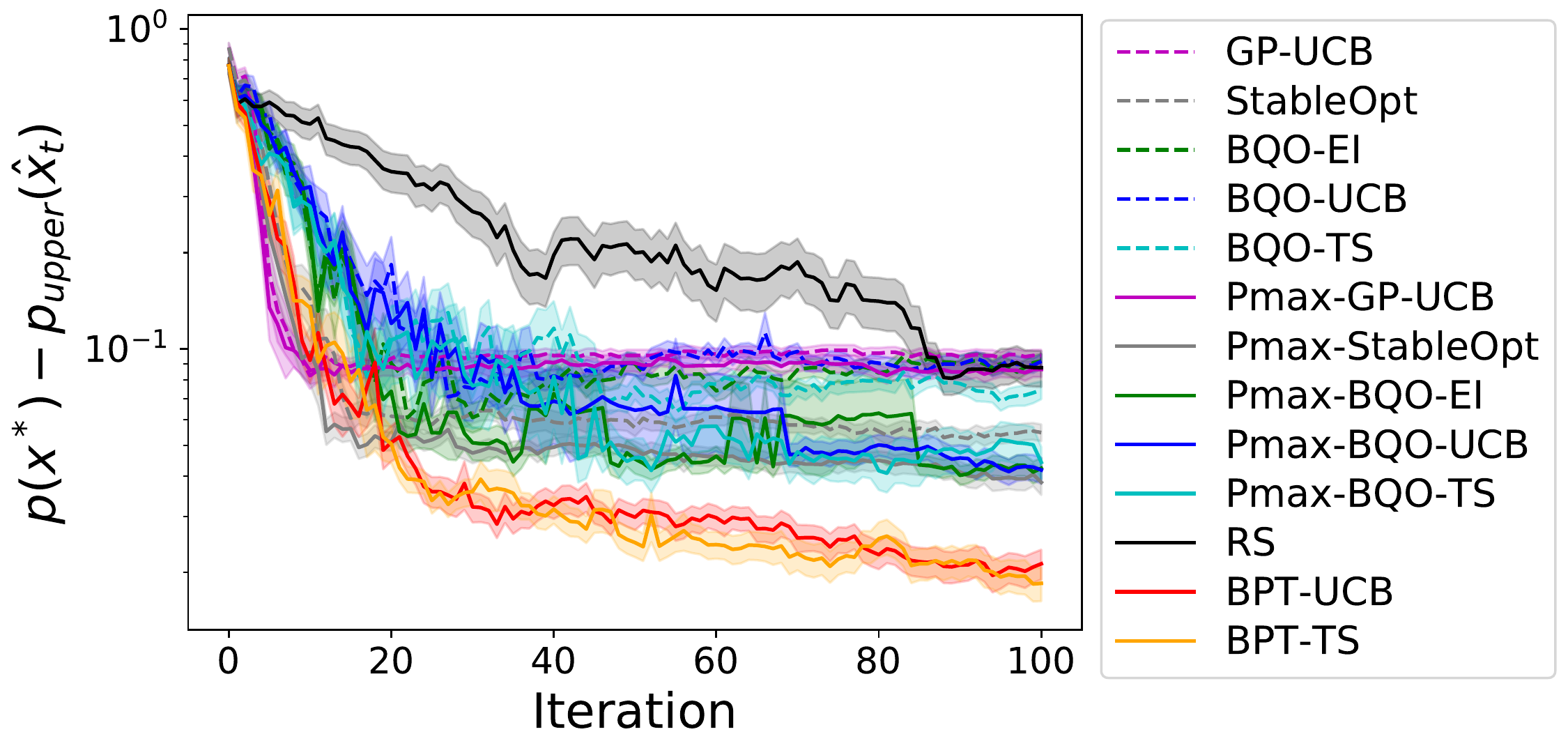}
 \caption{The experimental result in the optimization setting for Newsvendor Problem.
 This plot shows the average performance over $50$ trials.
 }
 \label{fig:newsvendor_result_app}
\end{figure}

\end{document}